%% file: book.tex
\documentclass{book}
\usepackage[utf8]{inputenc}

\include{header}

\usepackage{geometry}
\usepackage{booktabs}
\usepackage[tableposition=above]{caption}
\usepackage{pifont}

\title{Lecture Notes on Rough Paths and Applications to Machine Learning}
\author{Thomas Cass and Cristopher Salvi}
\date{Spring 2024}

\begin{document}

\maketitle


\tableofcontents

\addcontentsline{toc}{chapter}{Introduction}
\include{Introduction}
\include{chapter1}

\include{ex_chapter1}
\include{chapter2}
\include{ex_chapter2}

\include{chapter3}
\include{ex_chapter3}

\bibliographystyle{plain}
\bibliography{book}

\end{document}

%% file: header.tex
\usepackage{eurosym}
\usepackage{amsfonts}
\usepackage{amsmath}
\usepackage{amssymb}
\usepackage{graphicx}
\usepackage{color}
\usepackage[dvipsnames]{xcolor}
\usepackage{mathrsfs}
\usepackage[author={Oana}]{pdfcomment}
\usepackage{titling}
\usepackage{lipsum}
\usepackage{soul}
\usepackage{shuffle}
\usepackage{bm}
\usepackage{bbm}
\usepackage{comment}
\usepackage{enumerate}
\usepackage{framed}
\usepackage{pythonhighlight}

\usepackage{appendix}
\usepackage{hyperref}
\usepackage{amsthm}
\usepackage{cleveref}
\usepackage{amsthm}
\usepackage{mathtools}
\usepackage{stmaryrd} 

\newtheorem{theorem}{Theorem}[section]

\newtheorem{corollary}[theorem]{Corollary}

\newtheorem{definition}[theorem]{Definition}
\newtheorem{example}[theorem]{Example}
\newtheorem{exercise}[theorem]{Exercise}
\newtheorem{lemma}[theorem]{Lemma}

\newtheorem{proposition}[theorem]{Proposition}
\newtheorem{remark}[theorem]{Remark}

\newtheorem{notation}[theorem]{Notation}

\newcommand{\bE}{\mathbb{E}}
\newcommand{\N}{\mathbb{N}}
\newcommand{\R}{\mathbb{R}}

\newcommand{\wC}{\mathcal{C}}

\newcommand{\cH}{\mathcal{H}}
\newcommand{\cL}{\mathcal{L}}

\newcommand{\cX}{\mathcal{X}}

\newcommand{\floor}[1]{\lfloor #1 \rfloor}

\newcommand{\dd}{\,\mathrm{d}}

\DeclareMathOperator*{\argmin}{arg\,min}

\definecolor{ocean}{rgb}{0,0.1,0.6}

\newcommand{\notthis}[1]{}

\newcommand{\norm}[1]{\left\lVert#1\right\rVert}


%% file: Introduction.tex
\chapter*{Introduction}
These lecture notes introduce a mathematical concept, that of describing a path through its iterated integrals, and explain how the resulting theory can be used in data science and machine learning tasks involving streamed data. The study of iterated integrals has deep-seated roots in both mathematics and physics, with the early mathematical work stretching back at least to a series of papers by K-T. Chen, see e.g. \cite{chen1954iterated, chen1957integration, chen1961formal}. Chen's contribution spanned various domains, with a central emphasis being on homotopy theory and the use of path space integration to illuminate the interplay between topology and analysis\footnote{An overview of Chen's contributions can be found in the introduction to his collected works \cite{chen2001collected}}. Without providing an exhaustive historical narrative, related contemporaneous developments also appeared in physics, in the form of the Dyson exponential \cite{dyson1949radiation}, and the subsequent work of W. Magnus \cite{magnus1954exponential}.  Later insights revealed how similar principles can be applied in the control theory of nonlinear systems, catalysed by the work of Fliess \cite{fliess1983algebraic}. Consequently, what we term the ``signature'' often appears as the eponymous Chen-Fliess series in certain parts of the literature. 

By building on these foundations, B. Hambly and T. Lyons showed in a landmark 2010 paper \cite{hambly2010uniqueness} how, for two continuous paths of finite 1-variation, the agreement of signatures coincides with the notion that the pair of paths are tree-like equivalent. The collection of iterated integrals thus determines a path up to tree-like equivalence analogously to the way in which integrable functions on the circle are determined by their Fourier coefficients, up to Lebesgue-null sets.  

A later, no less influential development was the theory of rough paths, which was developed by T. Lyons as a branch of stochastic analysis in the 1990s \cite{lyons1998differential}, \cite{lyons2007differential}. A core achievement of this theory is to show how a path-by-path and robust notion of solution can be ascribed to differential equations of the form:  
\begin{equation}
dy_{t}=V(y_{t}) dx_{t}, \text{ started at $y_{0}$}.
\end{equation}
This construction holds even in cases where the input path $x$ may be highly irregular as a function of time. The crux is to enhance $x$ with iterated integral type terms, up to certain order, to form a rough path. The Hambly and Lyons theorem allows an extension to the setting of (geometric) rough paths \cite{boedihardjo2016signature}.  A breakthrough in merging these ideas with modern data science tools came in 2013, when B. Graham \cite{graham2013sparse} combined the signature features of pen strokes with sparse neural network methods to win part of an international competition on Chinese handwriting recognition.  

The purpose of these notes is thus to give an account of the extraordinarily rapid progress of the application of this mathematics to data science since these initial developments. By adopting simplifications, opting for a streamlined treatment over generality, we hope that this text can be used by researchers and graduate students seeking a quick introduction to the subject from a cross section of disciplines including mathematics, data science, computer science statistics and engineering. 

Our presentation proceeds in the order of abstraction-first, applications-second.  Accordingly, our first section treats the fundamental mathematical underpinnings of the theory, focusing on the core properties of the signature transform.  We freely make compromises to maintain the tempo. One such concession is to develop the theory specifically for continuous, finite-length paths, and not the more delicate cases of (geometric) rough paths or discontinuous paths. In the same vein, we keep algebraic notions to a minimum. A pivotal theme throughout the section is how the iterated integral perspective on signatures can be enriched by treating the ensemble as the solution to a particular controlled differential equation. This observation often facilitates simpler and more direct proofs. It also highlights another key theme of the subject: the idea of controlled differential equations as paradigmatic examples of functions on path space. The coordinate iterated integrals emerge as a universal feature set, or ``basis", for functions on path space, and serve as a foundation for functional approximation. To capture this idea using mathematics, we need to introduce unparametrised path space and understand some elementary topology on this space. The section culminates in a practical demonstration illustrating how this concept can be applied to execute nonlinear regression on path space using signature features.

Signature methods offer powerful techniques for handling streamed data, although they pose challenge themselves due to the inherently high dimensionality of the feature space. In statistical learning, kernel methods have emerged as effective tools for managing such difficulties for vector-valued data in an inner product space. Kernel methods, including signature kernel methods, can be applied in tasks where it is sufficient to know only the inner products of observations and not the exact values of the features themselves. Our second section develops the core concepts behind signature kernel methods building upon the foundations established in the first section. The exposition deliberately mirrors that of classical kernel methods for vector-valued data. We establish the main theoretical properties of these kernels -- universality and characteristicness -- and illustrate their practical implications for concrete applications, such as distribution regression for streamed data. Computational considerations predominate when assessing the effectiveness of a kernel; a kernel only becomes useful if there is an efficient means for evaluating it at a pair of observations. For a certain class of signature kernels, this can be achieved through solving a partial differential equation known as the signature kernel PDE. This approach is particularly effective for sufficiently smooth input paths, allowing for efficient computation of inner products. 

The concluding section provides an overview of the theory of rough differential equations, which naturally extends the concept of a controlled differential equation. In deep learning, neural ordinary differential equation (Neural ODE) models have recently emerged, offering a continuous-time approach for modelling the latent state of residual neural networks. Neural controlled differential equations (Neural CDEs) are continuous-time analogues of recurrent neural networks and provide a memory-efficient way to model functions of incoming data streams. When the driving signal is Brownian, the resulting neural stochastic differential equations (Neural SDEs) can be used as powerful generative models for time series. These models can be studied in a unified way by using the overarching framework of neural rough differential equation (Neural RDEs). We provide an overview of the fundamental constituents of rough path theory, placing emphasis on how they are used in practice through numerical schemes, such as the log-ODE method, for the training of Neural CDEs, SDEs and RDEs.

These notes have emerged from a series of lecture courses delivered at Imperial College London by the authors during the periods of 2020 to 2021, initially by TC and then, from 2022 to 2024, by CS.  We extend our gratitude to successive cohorts of students in the MSc in Mathematics and Finance programme, as well as the students in the EPSRC CDT in the Mathematics of Random Systems, whose insightful comments and feedback have enriched these sessions and hardened up the presentation of these notes. Special acknowledgments are owed to Nicola Muca Cirone, Francesco Piatti, Will Turner, and Nikita Zozoulenko for their invaluable feedback and discussions on earlier versions of these notes. TC is grateful to the Institute of Advanced Study where he spent the academic year 2023-24, with the generous support of a Erik Ellentuck Fellowship, and where a first draft of these notes were finalised.

Thomas Cass, Princeton  \\
Cristopher Salvi, London  

April 2024 

%% file: chapter1.tex
\chapter{The signature transform}\label{sec:signatures}

\section{Controlled differential equations} 

At the heart of rough path theory \cite{lyons1998differential} is the challenge of modelling the response over an interval of time generated by the interaction of a driving signal with a nonlinear control system.
The mathematical theory of controlled differential equations (CDEs) provides a rich setting in which complex interactions of this nature can be described and analysed. The essential data needed to understand this framework consist of a control path
$x:\left[  a,b\right]  \rightarrow V$, a map $f:W\rightarrow L\left(
V,W\right),$ and an initial state $y_{a} \in W$, where $[a,b] \subset\mathbb{R}_+$ is a closed  interval, $V,W$ are two finite-dimensional Banach spaces, and $L(V,W)$ is the space of linear maps from $V$ to $W$. The response $y$, itself a path
$y:\left[  a,b\right]  \rightarrow W$, will be uniquely determined whenever
there exists a unique solution to the following integral equation defined in terms of these data
\begin{equation}
y_{t}=y_{a}+\int_{a}^{t}f\left(  y_{s}\right)  dx_{s} \quad \text{ for all }t\text{
in }\left[  a,b\right]  .\label{integral}%
\end{equation}
We will refer to (\ref{integral}) as a \emph{controlled differential equation} (CDE),  usually adopting the more concise notation expressed in terms of differentials
\begin{equation}
dy_{t}=f\left(  y_{t}\right)  dx_{t} \quad \text{ started at }y_{a}%
\label{differential}%
\end{equation}
This notation makes it apparent that the role of the
function $f$ is to model how, at every point along the output trajectory,
 changes in $y$ are brought about by changes in
$x$. Before we develop these points, we recall the  definition of a \emph{vector field}. 

\begin{definition}[Vector field]
Let $k\in 
\mathbb{N}
\cup\left\{  0\right\}$. A $k$-times continuously differentiable
function $h:W\rightarrow W$ is called a $C^{k}$-vector field on
$W.$\ Denote the space of $C^{k}$ and infinitely differentiable vector fields by $\mathcal{T}^{k}\left(
W\right)$ and
$\mathcal{T}\left(  W\right)$ respectively.
\end{definition}

If we assume that $f$ in equation (\ref{differential}) is $C^k$, $f$ can be equivalently defined as a
linear map $v\mapsto f_{v}$ from $V$ into the (real) vector space $\mathcal{T}
^{k}\left(W\right)$ of $C^{k}$-vector fields on $W$, where $f_{v}$ and $f$ are related by
\[
f_{v}\left(  w\right)  =f(w)\left(  v\right)  \text{ for all }v\in V,w\in W.
\]

With this point of view, one can use the alternative notation
for the CDE (\ref{differential})
\[
dy_{t}=f_{dx_{t}}\left(  y_{t}\right)  \text{ started at }y_{a}.
\]
In particular, if the driving path $x$ is differentiable, then this CDE becomes $\dot
{y}_{t}=f_{\dot x_t}\left(y_{t}\right)$ so that $y$ is the integral curve of the time-dependent vector field
\[
F_{x}\left(t,y\right)  =f_{\dot x_t}\left(  y\right)  .
\]
The CDE terminology originates from optimal control theory, where typically some desired feature of the output $y$ is realised by appropriately selecting some \textit{control} $\dot x$.

\paragraph{Geometric view} A vector field $h$ in $\mathcal{T}\left(  W\right)$, defined below as a smooth function from $W$ to itself, can equivalently be
defined as a first-order differential operator acting on functions in
$C^{\infty}(W).$ From this point of view $h:C^{\infty}(W)\rightarrow
C^{\infty}(W)$ is the linear map 
\[
h\phi\left(  w\right)  =D\phi\left(  w\right)  \left(  h\left(  w\right)
\right)  =\left.  \frac{\partial}{\partial\eta}\right\vert _{0}\phi\left(
w+\eta h\left(  w\right)  \right)  .
\]

A choice of basis $\left\{  b_{i}\right\}  _{i=1}^{n}$ for $W$  gives rise
to a global coordinate system$\ w=\left(  w^{1},...,w^{n}\right)  $ on $W\,,$
which then allows us to write $h\phi$ in terms of partial derivatives
\begin{align}
h\phi\left(  w\right)  =\sum_{i=1}^{n}h^{i}\left(  w\right)  \partial_{i}%
\phi\left(  w\right) &:=\sum_{i=1}^{n}h^{i}\left(  w\right)  \left.
\frac{\partial}{\partial\eta}\right\vert _{0}\phi\left(  w+\eta b_{i}\right),\label{hphi} \\
&\text{ with }h^{i}\left(  w\right)  =\left[  h\left(  w\right)  \right]
^{i}. \nonumber
\end{align}
It can be convenient to write as $h\phi=h^{i}\partial_{i}\phi$,
implicitly summing over repeated indices.

Viewing vector fields as differential operators allows us to consider their composition as well as the important operation of taking their Lie brackets.

\begin{definition}
Suppose that $g$ and $h$ are two vector fields in $\mathcal{T}\left(
W\right)  ,$ then the composition $gh$ is the differential operator defined by
$\left(  gh\right)  \phi=g\left(  h\phi\right)  $ for all $\phi$ in
$C^{\infty}\left(  W\right)  .$ Furthermore, the Lie bracket of $f$ and $g$ is
the differential operator defined by $\left[  g,h\right]  \phi=\left(
gh-hg\right)  \phi$ for all $\phi$ in $C^{\infty}\left(  W\right)  .$
\end{definition}

The Lie bracket $\left[  g,h\right]  $ is in fact another vector field in the sense that it is a first-order differential operator, as can
be seen from the fact that the terms involving second-order derivatives
disappear:
\[
\left[  g,h\right]  \phi=g^{j}\partial_{j}\left(  h^{i}\partial_{i}\right)
-h^{j}\partial_{j}\left(  g^{i}\partial_{i}\right)  =\left(  g^{j}\partial
_{j}h^{i}-h^{j}\partial_{j}g^{i}\right)  \partial_{i}.
\]

As a consequence $\left[  \cdot,\cdot\right]  :\mathcal{T}\left(  W\right)
\times\mathcal{T}\left(  W\right)  \rightarrow\mathcal{T}\left(  W\right)  $
is a bilinear map.

One of the achievements of rough path theory is to give a framework for interpreting solutions to controlled differential equations when the driving signal $x$ is very far from differentiable. The focus of this chapter and the next will not be on this theory, but many of the underpinning ideas are nonetheless inspired by it. A common theme is the central role will be played by the collection of iterated integrals of the path $x$, and it is with an our account of their properties that we begin.

\subsection{Coordinate iterated integrals}
To understand why iterated integrals emerge from studying CDEs, we recall the classical method of Picard iteration for constructing solutions to CDEs. The basic idea is to inductively define a sequence of paths $\left(
y^{n}\right)  _{n=0}^{\infty}$, the Picard iterates, as follows
\begin{align*}
y^{(0)}   \equiv y_{a} \quad \text{ and } \quad y_{t}^{(n+1)} = y_{a}+\int_{a}^{t}f(y_{s}^{(n)})dx_{s}.
\end{align*}
Provided sufficient regularity of the data $f$ and $x$ is assumed, the sequence of Picard iterates can be shown to converge to a limit in a suitable sense. A final step is then to show that this limit solves (\ref{integral}), and to understand whether this solutions is unique.

We consider a simple special case of this scheme when
$f:W\rightarrow L\left(  V,W\right)  $ is a linear function; that is, when
$f$ belongs to the space $L\left(  W,L\left(  V,W)\right)  \right)  $. If we allow that $x$ is differentiable, a general Picard
iterate then has the form
\begin{equation}
y_{t}^{\left(  n\right)  }=\sum_{k=0}^{n}\int_{a<t_{1}<...<t_{k}<b}f^{\otimes
k}\left(  y_{a}\right)  \left(  \dot{x}_{t_{1}},\dot{x}_{t_{2}},...,\dot
{x}_{t_{k}}\right)  dt_{1}dt_{2}...dt_{k}\label{n-Pic}%
\end{equation}
where $f^{\left(  0\right)  }=$id$_{W},$ the identity function on $W\,$, and
$f^{\otimes k}\left(  y\right)  :V\times V\times....\times V\rightarrow W$ is
the $k$-fold multilinear map defined by
\begin{equation}
f^{\otimes k}\left(  y\right)  \left(  v_{1},..,v_{k}\right)  =f(f^{\otimes
\left(  k-1\right)  }\left(  y\right)  \left(  v_{1,...,}v_{k-1}\right)
)(v_{k}).\label{fk}%
\end{equation}
Multilinear maps can be better understood through the notion of tensor product that we now recall.

\begin{definition}\label{tensor_product}
Let $S$ and $T$ be two vector spaces. The tensor product of $S$ and $T$ is a
vector space $S\otimes T$ together with a bilinear map from $S\times T$ to
$S\otimes T\,,$ whose action is denoted by $\left(  s,t\right)  \mapsto
s\otimes t,$~which has the universal property that for any bilinear function
$\phi:S\times T\rightarrow U$ into another vector space $U$ there exists a
unique linear map $\hat{\phi}:S\otimes T\rightarrow U$ such that $\phi\left(
s,t\right)  =\hat{\phi}\left(  s\otimes t\right)  .$
\end{definition}

We will most often use this in the case $V=S=T$ when we will write
$V^{\otimes2}$ as short hand for  $V\otimes V.$ The definition is associative,
$V_{1}\otimes(V_{2}\otimes V_{3})$ and $\left(  V_{1}\otimes V_{2}\right)
\otimes V_{3}$ are isomorphic, and hence we can write $V^{\otimes k}$
unambiguously to denote the $k$-fold tensor product of $V.$ Under this
isomorphism we have that $v_{1}\otimes\left(  v_{2}\otimes v_{3}\right)
=\left(  v_{1}\otimes v_{2}\right)  \otimes v_{3}$ and thus $v_{1}\otimes
v_{2}\otimes v_{3}$ is also well defined. As a consequence of the universal
property, the multilinear map $f^{\otimes k}\left(  y\right)  $ in
(\ref{fk}) is such that%
\[
f^{\otimes k}\left(  y\right)  \left(  v_{1},...,v_{k}\right)
=\widehat{f^{\otimes k}\left(  y\right)  }\left(  v_{1}\otimes...\otimes
v_{k}\right)  ,
\]
for a unique linear map $\widehat{f^{\otimes k}\left(  y\right)  } : V^{\otimes k} \to W$. To ease
the notational burden we abuse notation and refer to this map also as
$f^{\otimes k}\left(  y\right)$. It will also be unwieldy to write $\otimes$
each time a tensor product of vectors occurs and so we adopt the more concise
notation $v_{1}...v_{k}$ instead of $v_{1}\otimes...\otimes v_{k}$. Using this
and the linearity of the integral we write the $n^{\text{th}}$ Picard iterate
(\ref{n-Pic}) as%
\begin{equation}\label{eqn:picard_iterate}
y_{t}^{\left(  n\right)  }=\sum_{k=0}^{n}f^{\otimes k}\left(  y_{a}\right)
\left(  \int_{a<t_{1}<...<t_{k}<b}\dot{x}_{t_{1}}\dot{x}_{t_{2}}...\dot
{x}_{t_{k}}dt_{1}dt_{2}...dt_{k}\right)  ,
\end{equation}
where
\begin{equation}
\int_{a<t_{1}<...<t_{k}<b}\dot{x}_{t_{1}}\dot{x}_{t_{2}}...\dot{x}_{t_{k}%
}dt_{1}dt_{2}...dt_{k}=\int_{a<t_{1}<...<t_{k}<b}dx_{t_{1}}dx_{t_{2}%
}...dx_{t_{k}}\in V^{\otimes k}\label{kii}%
\end{equation}
is the $k$-fold iterated integral of $x.$ The assumption that $x$ is differentiable is an unnecessarily strong one to define these iterated integrals. For a less restrictive notion, we will need the following way to quantify the regularity of a path.

\begin{definition}[$p$-variation]\label{def:p-variation}
    Let $p \geq 1$ be a real number. The $p$-variation of a path $x: [a,b] \to V$ be on any subinterval $[s,t] \subset [a,b]$ is the following quantity
    \begin{equation}\label{eqn:p_var}
        \norm{x}_{p,[s,t]} = \left( \sup_{\mathcal{D} \subset [s,t]} \sum_{t_i \in \mathcal{D}} \norm{x_{t_{i+1}} - x_{t_i}}^p \right)^{1/p},
    \end{equation}
    where $||\cdot||$ is any norm on $V$, and the supremum is taken over all partitions $\mathcal{D}$ of the interval $[s,t]$, i.e. over all increasing finite sequences such that $\mathcal{D}=\{s=k_0 < k_1 <... < k_r=t\}$. We say that $x$ has finite $p$-variation on $[s,t]$ if $\norm{x}_{p,[s,t]} < \infty$. If $x$ has finite $1$-variation we will say that $x$ has bounded variation (or has finite length).
\end{definition}

\begin{notation}
 We will denote by $C_{p}([a,b],V)$, to be the space of continuous paths from the interval $[a,b]$ to $V$ which have finite $p$-variation. If the intervals can be understood from the context, then we use the abridged notation $ \norm{x}_{p}$ and $C_{p}(V)$ respectively in place of $ \norm{x}_{p,[s,t]}$ and $C_{p}([a,b],V)$.
\end{notation}

Using the definition of integration theory due to Young (e.g. \cite[Chapter 6]{friz2010multidimensional}), the integrals in equation (\ref{kii}) will continue to exist whenever $x$ is in $C_{p}([a,b],V)$ and for any $1\leq p <2$. As these objects will appear
repeatedly, we introduce the notation
\[
S\left(  x\right)  ^{\left(  k\right)}=S\left(  x\right)  ^{\left(  k\right)   }_{[a,b]}=\int_{a<t_{1}<...<t_{k}<b}dx_{t_{1}%
}dx_{t_{2}}...dx_{t_{k}}\in V^{\otimes k}%
\]
to denote the $k$-fold iterated integral in (\ref{kii}). A central theme in the sequel will be the way in which the collection of \emph{all} such iterated integrals determines the effect of the path $x$ on general \emph{non-linear} CDEs.

\begin{definition}[Signature transform]\label{def:signature}
For any path $x\in C_p([a,b],V)$ with $p \in [1,2)$ and any subinterval $[s,t] \subset [a,b]$, the signature transform (ST) $S(x)_{[s,t]}$ of $x$ on $[s,t]$ is defined as the following collection of iterated integrals
\begin{equation}\label{eqn:signature}
S(x)_{[s,t]} = \left(1,S\left(  x\right)  ^{\left(  1\right)   }_{[s,t]},....S\left(  x\right)  ^{\left(  k\right)   }_{[a,b]},...\right) \in \prod_{i=0}^{\infty}V^{\otimes i}.
\end{equation}
When it is clear from the context, we will suppress the dependence on the interval and instead denote the ST of $x$ by $S(x)$. 
\end{definition}

\begin{remark}
    It is also worth recording notation for a version of the ST in which only finitely many iterated integrals are retained. Let $n$ be in $\mathbb{N}.$ The $n$-step truncated signature transform (TST) $S_{n}\left(  x\right)  _{\left[a,b\right]  }$ is defined by
    \begin{equation*}
        S_n\left(  x\right)  _{\left[  a,b\right]  }=\left(  1,S\left(  x\right)
        _{\left[  a,b\right]  }^{\left(  1\right)  },...,S\left(  x\right)  _{\left[
        a,b\right]  }^{\left(  n\right)  },0,0,..\right)  \in
        {\textstyle\prod\nolimits_{i=0}^{\infty}}
        V^{\otimes i}
    \end{equation*} 
    As with the ST, we omit reference to the interval $S_{n}\left(
    x\right)  $ when the interval can be inferred from the context.
\end{remark}

The iterated integrals characterized by the ST have been the object of numerous mathematical studies prior to rough paths. Initially presented in a geometric framework by Chen in \cite{chen1957integration} and later by Fliess and Kawski \cite{fliess1983algebraic, kawski1997noncommutative} in control theory, many researchers explored this sequence of iterated integrals of a path to derive pathwise Taylor series for solutions of differential equations. In fact, the ST is sometime referred to as the "Chen-Fliess series". This book delves into integrating the ST in machine learning to model functions on sequential data.

\subsection{Algebras of tensors}

As anticipated in the previous paragraph, terms of the ST act as non-commutative monomials consistituing the main building blocks to perform Taylor expansions for continuous functions on paths. To clarify this point, consider the space of continuous real-valued functions from the interval $\left[ 0,1\right]$. The classical Weierstrass theorem states that for any level of precision $\epsilon >0$, there exists a polynomial in one variable $"x"$ that uniformly approximates the function up to error $\epsilon$. More generally, let $U$ be a Banach space and consider the vector space of polynomials $\mathbb{R}[U]$ on $U$ 
\begin{equation}\label{eqn:monomials}
    \mathbb{R}[U] = \text{Span}\{\phi_1(\cdot)...\phi_k(\cdot) \mid \phi_i \in U^*, k \geq 0\}
\end{equation}
where $U^*$ denotes the dual space of $U$. Consider a compact subset $K \subset U$ such that the ring of polynomials $\mathbb{R}[K]$ restricted to $K$ separates points (i.e. for any $x\neq y \in K$ there exists a polynomial $P \in \mathbb{R}[K]$ such that $P(x)\neq P(y)$). Then, by the Stone-Weierstrass Theorem (e.g. \cite[Thm. 4.45]{folland1999real}), for any continuous function on $K$ and for any level of precision $\epsilon >0$, there exists a polynomial $P \in \mathbb{R}[K]$ that uniformly approximates the function up to error $\epsilon$. 

In this book we are interested in the setting where $U$ is the set of continuous paths of bounded $p$-variation for $p \in [1,2)$. Monomials will be given by terms of the ST. This motivates the need to study the algebraic structure of the space of tensors $\prod_{i=0}^{\infty}V^{\otimes i}$. Subsequently, the algebraic and anlytic properties of the ST will yield the fundamental approximation result for the ST in \Cref{thm: universal approximation} which will be derived as a simple consequence of the Stone-Weierstrass Theorem.
    
It will be important to have the structure of an algebra on the product space ${\textstyle\prod\nolimits_{i=0}^{\infty}}V^{\otimes i}$. To introduce this, we observe from the universal property of
the tensor product that the space $V^{\otimes n}\otimes V^{\otimes m}$ is
isomorphic to $V^{\otimes\left(  n+m\right)  }$ and therefore that $v_{m}v_{n}$ is a well-defined element of $V^{\otimes\left(  n+m\right)  }.$ This allows to define an associative product on ${\textstyle\prod\nolimits_{i=0}^{\infty}}
V^{\otimes i}$ which is compatible with the tensor products on the individual levels $V^{\otimes
i}$. By this we mean that if $\mathfrak{i}_{n}:V^{\otimes n}\hookrightarrow$ $
{\textstyle\prod\nolimits_{i=0}^{\infty}}
V^{\otimes i}$ denotes the canonical inclusion, then the tensor product on $
{\textstyle\prod\nolimits_{i=0}^{\infty}}
V^{\otimes i}$ satisfies
\[
\mathfrak{i}_{n+m}\left(  v_{n}v_{m}\right)  =\mathfrak{i}_{n}\left(
v_{n}\right)  \mathfrak{i}_{m}\left(  v_{m}\right)
\]
for every $n$ and $m.$ The following definition provides this.

\begin{definition}\label{tensor product}
Suppose that $v=\left(v_{0}%
,v_{1},...\right)  $ and $w=\left(  w_{0},w_{1},...\right)$ are two elements
of ${\textstyle\prod\nolimits_{i=0}^{\infty}}
V^{\otimes i}.$ Then we define the (extended) tensor product 
$v \cdot w=\left( z_{0},z_{1},...\right) $ as
\begin{equation}
z_{i}=\sum_{l=0}^{i}v_{l}w_{i-l}\in V^{\otimes i}, \quad \forall k \in \mathbb{N}
\cup\left\{  0\right\}\label{tp}
\end{equation}
\end{definition}

\begin{definition}[Extended tensor algebra]
The product space $
{\textstyle\prod\nolimits_{i=0}^{\infty}}
V^{\otimes i}$ equipped with the product $\cdot$ is an algebra. Sometimes it is
convenient to interpret elements of $
{\textstyle\prod\nolimits_{i=0}^{\infty}}
V^{\otimes i}$ as formal series of tensors; that is, to identify $v=\left(
v_{0},v_{1},...\right)  $ with $\sum_{i=0}^{\infty}v_{i}$. We denote by $T\left(  \left(  V\right)  \right)$ the algebra of formal tensor
series with the product (\ref{tp}).  
\end{definition}

\begin{remark}[Tensor algebra]
\label{series poly}
The more commonly encountered tensor algebra $T(V)=\bigoplus
_{i=0}^{\infty}V^{\otimes i}$ is a subalgebra of $T((V)).$ Note that that the
difference between $T\left(  V\right)  $ and $T\left(  \left(  V\right)
\right)  $ is that $T\left(  V\right)  $ consists only of tensor polynomials,
i.e. $\sum_{i=0}^{\infty}v_{i}$ for which $v_{i}=0$ for all but finitely many
$i$ in $
\mathbb{N}
\cup\left\{  0\right\}  .$ By identifying again $v=\left(  v_{0}
,v_{1},...\right)  $ and $\sum_{i=0}^{\infty}v_{i}$, we can equivalently
describe $T\left(  V\right)  $ as the subalgebra of ${\textstyle\prod\nolimits_{i=0}^{\infty}}V^{\otimes i}$ in which all but finitely many projections are zero.
\end{remark}

The tensor algebra $T(V)$ has the following universal property, which in fact can be used
as an alternative way to uniquely characterise it (up to isomorphism). 

\begin{definition}\label{def:freeness_tensor_algebra}
Let $A$ be an associative algebra
and let $F:V\rightarrow A$ be a linear map. Then there exists a unique algebra
homomorphism $\hat{F}:T\left(  V\right)  \rightarrow A$ which extends $F$ in
the sense that $F=\hat{F}\circ\mathfrak{i}_{1},$where $\mathfrak{i}%
_{1}:V\hookrightarrow T\left(  V\right)  $ is the canonical inclusion.
\end{definition}

\begin{example}
Let End$\left(  W\right)  =L\left(  W,W\right)  $ denote the algebra of
endomorphisms of $W$ with an (associative) product given by $A\ast B=B\circ A$
where $\circ$ denotes the usual composition of linear maps. In the case of our
earlier discussion of a differential equation with linear vector fields we
used $f$ in $L\left(  W,L\left(  V,W\right)  \right)  $. Any such $f$ is
determined by the map $F$ in $L\left(  V,\text{End}\left(  W\right)  \right)
$ through $F\left(  v\right)  \left(  w\right)  =f\left(  w\right)  \left(
v\right)  .$ By using the universal property of the tensor algebra, we can
extend $F$ to an algebra homomorphism $\hat{F}$ on $T\left(  V\right)  $ into
$($End$\left(  W\right)  ,\ast).$ We note that for any $k$ this map is
determined by
\begin{align*}
\hat{F}\left(  v_{1}...v_{k}\right)  \left(  w\right)    & =\left[  F\left(
v_{1}\right)  \ast F\left(  v_{2}\right)  \ast...\ast F\left(  v_{k}\right)
\right]  \left(  w\right)  \\
& =f\left(  ...f\left(  f\left(  w\right)  \left(  v_{1}\right)  \right)
\left(  v_{2}\right)  ...\right)  \left(  v_{k}\right)  \\
& =f^{\otimes k}\left(  w\right)  \left(  v_{1},..,v_{k}\right)  .
\end{align*}
With this formulation, we can further simplify the $n^{th}$ Picard iterate (\ref{eqn:picard_iterate})
by writing
\begin{equation}\label{eqn:n_th_Picard_iterate}
    y_{t}^{\left(  n\right)  }=\hat{F}\left(  S_{n}\left(  x\right)_{a,t}\right)  \left(  y_{a}\right).
\end{equation}
\end{example}

\paragraph{Inner products} Given inner products on $V$ and $W$, we can
endow $V\otimes W$ with a canonical inner product.

\begin{definition}
Let $V,W$ be vector spaces endowed with the inner-products $\left\langle
\cdot,\cdot\right\rangle _{V}$ and \thinspace$\left\langle \cdot
,\cdot\right\rangle _{W}$ respectively. Define the (Hilbert-Schmidt)
inner product on $V\otimes W$ by
\[
\left\langle v_1 \otimes w_1, v_2 \otimes w_2\right\rangle_{V\otimes W} = \left\langle v_1, v_2\right\rangle _{V}\left\langle w_1,w_2\right\rangle _{W} 
\]
for any $v_1,v_2 \in V$ and $w_1,w_2 \in W$.
\end{definition}

Thus, we can equip  $V^{\otimes k}$ with the norm determined by 
$$\norm{v}_{V^{\otimes k}} = \sqrt{\prod_{i=1}^k \langle v_i, v_i\rangle_{V}}, \text{ for } v=v_{1}...v_{k}.$$
In the sequel, we will consider Banach spaces $E \subset T((V))$ defined as follows 
\begin{equation*}
    E=\left\{  v\in T\left(  \left(  V\right)  \right)  :\norm{v} < \infty\right\},
\end{equation*}
in which the norm will typically be an $l_{q}$-norm, with $q$ in $[1,\infty)$, of the form 
$$\norm{v}=\left(\sum_{k=0}^{\infty}\left\vert \left\vert
v_{k}\right\vert \right\vert_{V^{\otimes k}}^{q}\right)^{1/q}.$$
\begin{example}\label{norm properties}
Two cases are of particular interest. When $q=2$ the resulting space is a Hilbert space with the inner product given by  
\[
\langle v, w \rangle = \sum_{k=0}^{\infty} \langle v_{k}, w_{k} \rangle_{V^{\otimes k}}.
\]
If $q=1$ then $(E,\norm{\cdot})$ is a Banach algebra, i.e. $\norm{ab} \leq \norm{a} \norm{b}$, as is easily verified.
\end{example}

\subsection{Functions on the range of the signature transform}

As noted in equation (\ref{eqn:monomials}) at the beginning of the previous section, to construct the vector space of polynomials $\mathbb{R}[U]$ over a Banach space $U$, one needs to consider monomials of the form $\phi \in U^*$, where $U^*$ denotes the dual space of $U$. 

In this section we outline how the dual space  $T(V)^*$ of the tensor algebra $T(V)$ can be constructed from the dual space $V^*$ of $V$. Dual elements in $T(V)^*$, when restricted to functionals on the range of the ST, will furnish the main building blocks (or monomials) of our Taylor expansion for functions on path space.

The (algebraic) dual space \thinspace$V^{\ast}$ of $V$ is the space of linear
functionals on $V,$ in other words $V^{\ast}=L\left(  V,%
\mathbb{R}
\right)  .$ Any basis $\left\{  e_{i}:i=1,..,d\right\}  $ of $V$ has a
corresponding dual basis of linear functionals $\left\{  e_{i}^{\ast
}:i=1,..,d\right\}  $ which are uniquely determined by
\[
e_{i}^{\ast}\left(  e_{j}\right)  =\left\{
\begin{array}
[c]{cc}%
1 & i=j\\
0 & i\neq j
\end{array}
\right.
\]
and by linearity. In particular, we see that $e_{i}^{\ast}\left(  v\right)
=v_{i}$ when $v=\sum_{i=1}^{d}v_{i}e_{i}.$ This leads naturally to a dual
basis for $\left(  V^{\ast}\right)  ^{\otimes k}$.

\begin{notation}\label{tensor basis}
Given a multi-index $I=\left(  i_{1},...,i_{k}\right)  $ of length $\left\vert I\right\vert
=k$ and a basis for $V$ as
above, we denote $e_{I}=e_{i_{1}}...e_{i_{k}}$. The set $\left\{ e_{I}:\left\vert I\right\vert =k\right\} $ is a basis for $%
V^{\otimes k}$ and so it also has a dual basis whose elements we denote by $%
e_{I}^{\ast }.$
\end{notation}
Any choice of basis for $V$ induces a vector space isomorphism between $%
\left( V^{\otimes k}\right) ^{\ast }$ and $\left( V^{\ast }\right) ^{\otimes
k}$ which is characterised by 
\[
e_{I}^{\ast }\mapsto e_{i_{1}}^{\ast }e_{i_{2}}^{\ast }...e_{i_{n}}^{\ast }.
\]%
This isomorphism is, in fact, independent of the original choice of basis
for $V$, as can be easily shown, and as such we will often implicitly
identify elements of $\left( V^{\otimes k}\right) ^{\ast }$ with the
corresponding element of  $\left( V^{\ast }\right) ^{\otimes k}$ under this
canonical isomorphism. A similar observation can be made for the dual space $%
T\left( V\right) ^{\ast }$, although more care is needed because $T(V)$ is
infinite dimensional; the follow lemma describes the situation.

\begin{lemma}
The vector spaces $T\left( \left( V^{\ast }\right) \right) $ and $T\left(
V\right) ^{\ast \ }$can be identified through the explicit isomorphism $\Phi
:T\left( \left( V^{\ast }\right) \right) \rightarrow T\left( V\right) ^{\ast
\ }$ 
\begin{equation}
\sum_{k\in \mathbb{N}_{0}}f_{k}=f\mapsto \Phi _{f}\in T\left( V\right)
^{\ast \ }  \label{isom}
\end{equation}
where, for any $v=\sum_{k\in A}v_{k}\in T\left( V\right) $ with $A\subset 
\mathbb{N}_{0}$ a finite subset, $\Phi _{f}\left( v\right) $ is defined to
equal $\sum_{k\in A}f_{k}\left( v_{k}\right) .$
\end{lemma}

\begin{remark}
Note that here $f_{k}$ in $\left( V^{\ast }\right) ^{\otimes k}$ is
identified with $\left( V^{\otimes k}\right) ^{\ast }$ in the way described
above.
\end{remark}

\begin{proof}
It is clear that $\Phi _{f}$ is a well-defined element of $T\left( V\right) .
$ The linear map $f\mapsto \Phi _{f}$  is injective since if $\Phi _{f}$ $%
=\Phi _{g}$ then also $f_{k}=\Phi _{f}\circ \mathfrak{i}_{k}=\Phi _{g}\circ 
\mathfrak{i}_{k}=g_{k}$ for all $k$ in $\mathbb{N}_{0}$, where $\mathfrak{i}%
_{k}:V^{\otimes k}\rightarrow T\left( V\right) $ is the canonical inclusion
map, and hence $f=g$. It is a bijection as $\phi $ in $T\left( V\right)
^{\ast \ }$ is determined by the collection of linear maps $\left\{ \phi
_{k}=\phi \circ \mathfrak{i}_{k}:k\in \mathbb{N}_{0}\right\} $ and $\phi
\left( v\right) =$ $\sum_{k\in \mathbb{N}_{0}}\phi _{k}\left( v\right) $,
the sum being finite for any $v$ in $T\left( V\right) .$ Each $\phi _{k}$ is
in $\left( V^{\otimes k}\right) ^{\ast }$ and therefore it can be identified
with an element of $\left( V^{\ast }\right) ^{\otimes k}$ canonically. By
letting $f=\sum_{k\in \mathbb{N}_{0}}\phi _{k}\in T\left( \left( V^{\ast
}\right) \right) $ we then see that $\phi =\Phi _{f}$.
\end{proof}

\begin{remark}
By considering the restriction of the linear map (\ref{isom}) to $T\left(
V^{\ast }\right) $ we can identify $T\left( V^{\ast }\right) $ with a
subspace of $T\left( \left( V\right) \right) ^{\ast }$. This is because $%
\Phi _{f}$ can be extended to $T\left( \left( V\right) \right) ^{\ast }$
whenever $f$ is a tensor polynomial.
\end{remark}

By regarding $e_{I}^{\ast }\in \left( V^{\otimes k}\right) ^{\ast }\cong
\left( V^{\ast }\right) ^{\otimes k}$ as an element of $T\left( V^{\ast
}\right) $ under the usual inclusion of $\left( V^{\ast }\right) ^{\otimes k}
$ into $T\left( V^{\ast }\right) ,$ the previous remark allows us to define
a collection of scalars       
\[
S\left( x\right) ^{I}=e_{I}^{\ast }\left( S\left( x\right) \right) \,
\]%
as $I$ ranges over all multi-indices $I$ of finite length. This offers
another view on the signature which is captured by the following definition.
\begin{definition}[Coordinate iterated integrals]
Let $\{e_{i}\}_{i=1,..,d}$ be a basis for $V$ with corresponding dual basis $\{e_{i}^{*}\}_{i=1,...d}$ and suppose that $1\leq p<2$. If $x$ is in $C_{p}(V)$ and $I=\left(  i_{1},...,i_{k}\right)$ is a multi-index of integers from $\{1,...,d\}$, then the coordinate iterated integral $S(x)^{I}$ is the real number given by $S(x)^{I}=e_{I}^{*}(S(x))$.
\end{definition}

\section{Analytical properties}

We will now apply the tools we have developed for studying tensors to understand some important analytic  properties of the ST. Firstly, the ST exhibits \emph{invariance under reparameterizations}; this essentially allows the ST to act as a filter that removes an infinite dimensional group of symmetries given by time reparameterizations. Secondly, the \emph{factorial decay} of the terms in the ST  implies that truncating the ST at a sufficiently high level retains the bulk of the critical information. This is especially pertinent for describing solutions to CDEs, as the omitted infinite number of ST coefficients diminish in importance factorially. Finally, the \emph{uniqueness} of ST for a certain class of paths, ensuring a one-to-one identifiability of a path with its ST.

\subsection{Sampling invariance}

An important obstacle that machine learning models have to face is the potential presence of symmetries in the data\footnote{In computer vision for example, a good model should be able to recognize an image even if the latter is rotated by a certain angle. The $3$D rotation group, often denoted $SO(3)$, is low dimensional ($3$), therefore it is relatively easy to add components to a model that build a rotation invariance, for example through \emph{data-augmentation}.}. When dealing with sequential data one is confronted with a (infinite dimensional) group of symmetries given by all \emph{reparametrizations} of a path, i.e. continuous, increasing surjective functions from the interval $[a,b]$ to itself. Practically speaking, the action of reparameterizing a path can be thought of as the action of sampling its observations at a different frequency. For example, consider the reparametrization $\phi: [0,1] \to [0,1]$ given by $\phi(t) = t^2$ and the path $\gamma : [0,1] \to \mathbb{R}^2$ defined by $\gamma_t = (\gamma^x_t, \gamma^y_t)$ where $\gamma^x_t=\cos(10t)$ and $\gamma^y_t=\sin(3t)$. As it is clearly depicted in \Cref{fig:reparam}, both channels ($\gamma^x,\gamma^y$) of $\gamma$ are individually affected by the reparametrization $\phi$, but the shape of the curve $\gamma$ is left unchanged. Building an invariance into a model is usually very hard. However, the ST has the useful property of being invariant to reparameterization, as stated in the following lemma. 

\begin{figure}[h]
    \centering
    \makebox[\textwidth]{\includegraphics[scale=0.4]{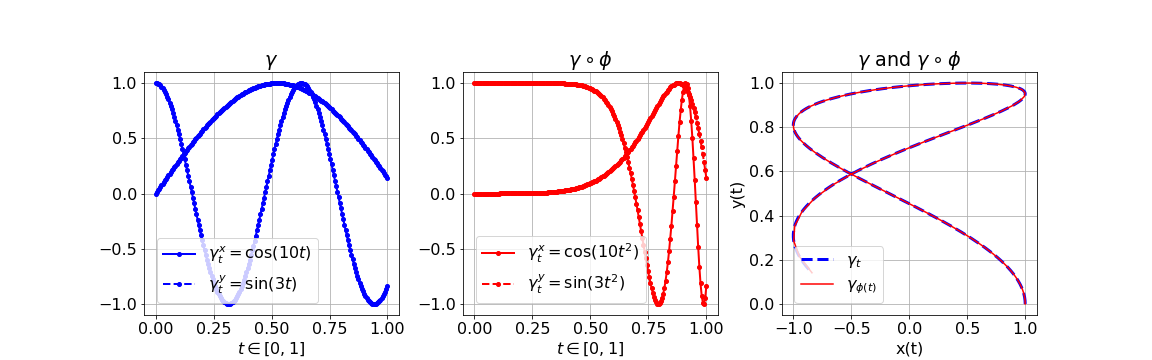}}
    \caption{{\small On the \textbf{left} are the individual channels $(\gamma^x, \gamma^y)$ of a $2$D paths $\gamma$. In the \textbf{middle} are the channels reparametrized under $\phi: t \mapsto t^2$. On the \textbf{right} are the path $\gamma$ and its reparametrized version $\gamma \circ \phi$. The two curves overlap, meaning that the reparametrization $\phi$ represents irrelevant information if one is interested in understanding the shape of $\gamma$.} Figure taken from \cite{salvi2021signature}.} 
    \label{fig:reparam}
\end{figure}

\begin{lemma}[Reparameterization invariance]\label{lemma:repearam_invariance} Suppose that $1\leq p <2$ and let $x$ be in $C_{p}([a,b],V)$ and let $[a,b]$ and assume that $\lambda :[c,d] \to [a,b]$ is a continuous non-decreasing surjection. Then $S(x)_{[a,b]} = S(x \circ \lambda)_{[c,d]}$.
\end{lemma}

\begin{proof}
We prove the result first under the assumption that $p=1$.
Let $x^{\lambda}:=x\circ \lambda$ denote the reparameterisation of $x$ by $\lambda$. We prove the statement by an induction on the level of the ST. The $1^{st}$-order iterated integral is  the increment, and so we have
\begin{equation*}
S(x^{\lambda})^{(1)}_{[c,t]} = x^{\lambda}_{t}-x^{\lambda}_{c}  = \int_{a}^{\lambda \left( t\right) }dx_{u} = S(x)^{(1)}_{[a,\lambda_t]} ,
\end{equation*}
which holds for all $t \in [c,d]$. Fixing a basis $\{v_{i}:i=1,...,d\}$ for $V$, we assume that for any $k \geq 0$ and any multi-index $(i_1,...,i_m)$ of integers from $\{1,...,d\}$ of length $m\leq k$, we have  equality of the coordinate iterated integrals defined with respect to this basis, i.e.
\begin{equation*}
S(x^{\lambda})_{[c,t]}^{(i_1,...,i_m)} = S(x)_{[a,\lambda_t]}^{(i_1,...,i_m)} \text{for all $t \in [c,d]$.}
\end{equation*}
We then consider an arbitrary multi-index $(i_1,...,i_k,i_{k+1})$ of length $k+1$. For any $t \in [c,d]$ we can use the induction hypothesis to write
\begin{eqnarray*}
S(x^{\lambda})_{[c,t]}^{(i_1,...,i_k,i_{k+1})} =\int_{c}^{t} S(x^{\lambda})_{[c,u]}^{(i_1,...,i_k)}d(x^{\lambda} _{u})^{(i_{k+1})}
 =\int_{c}^{t} S(x)_{[a,\lambda_u]}^{(i_1,...,i_k)}d(x^{\lambda}_{u})^{(i_{k+1})}.
\end{eqnarray*}
Let $\mu_{x}^{i}$ (resp. $\mu_{x^{\lambda}}^{i}$) be the unique (signed) Lebesgue-Stieltjes measure, defined on the Borel subsets of $[a,b]$ (resp. on $[c,d]$), which is associated to the function $x^{i}$ (resp. $x^{i}\circ \lambda$). It is easy to see that the image measure $\lambda_{*}\mu_{x^{\lambda}}^{i}:=\mu_{x^{\lambda}}^{i} \circ \lambda^{-1}$ equals $\mu_{x}^{i}$. Then, for any function $f:[a,b]\mapsto \mathbb{R}$ for which $f\circ \lambda$ is $\mu_{x}^{i}$-integrable, it holds that
\[
\int_{c}^{t}f(\lambda_u)\mu_{x^{ \lambda}}^{i}(du)=\int_{a}^{\lambda_t}f(r)(\lambda_{*}\mu_{x^{ \lambda}}^{i})(dr)=\int_{a}^{\lambda_t}f(r)\mu_{x}^{i}(dr);
\]
see, e.g., Theorem 3.6 of \cite{bogachev2007}.  By applying this result with the choices $f(r)=S(x)_{[a,r]}^{(i_1,...,i_k)} $ and $i=i_{k+1}$ we obtain 
\[
S(x^{\lambda})_{[c,t]}^{(i_1,...,i_k,i_{k+1})}=\int_{a}^{\lambda_t}S(x)_{[a,r]}^{(i_1,...,i_k)}\mu_{x}^{i}(dr)=S(x)_{[a,t]}^{(i_1,...,i_k,i_{k+1})}
\]
whereupon the induction step is complete.

The case $1<p<2$ can be obtained using a simple 
approximation argument by combining the fact that the closure of $C_{1}(V)$ in $C_{q}(V)$ contains $C_{p}(V)$ for any $p<q$. The argument is then concluded by considering such $q$ with $q<2$ and then using the joint continuity, in $q$-variation, of the Young integral map (e.g. \cite[Prop. 6.11]{friz2010multidimensional})
\[
C_{p}(V) \times C_{p}(V) \owns (f,g) \mapsto \int f dg.
\]
It is left to the reader to fill in the fine details.
\end{proof}

\begin{remark}
\label{Lipschitz reparam}It is useful sometimes to work with
re-parameterisations which may fail to be continuous or even surjective. In
these cases, it is possible for a version of the previous theorem to
still hold. For example, suppose that $x$ is a path in $C_{1}\left(
[a,b],V\right)  $ of length $L=\left\vert \left\vert x\right\vert \right\vert
_{1;\left[  a,b\right]  }$ then we can define a right-continuous function
$\tau:\left[  0,L\right]  \rightarrow\lbrack a,b]$
\begin{equation}
\tau\left(  s\right)  =\inf\left\{  t\geq a:\left\vert \left\vert x\right\vert
\right\vert _{1;[a,t]}>s\right\}  =\sup\left\{  t\geq a:\left\vert \left\vert
x\right\vert \right\vert _{1;[a,t]}\leq s\right\}  \label{reparam}%
\end{equation}
with the convention that $\inf\emptyset=b$. In spite of the lack of continuity
of $\tau,$ it can be seen easily that $x\circ\tau$ is itself still continuous
and indeed still has finite length. In fact more is true: under this choice of
reparameterisation, the path $x\circ\tau$ is Lipschitz continuous. The reader is invited to prove
this in Exercise \ref{ex:Lip}. The steps in the proof of the previous theorem can
be retraced in this case to show again that $S\left(  x^{\tau}\right)
=S\left(  x\right)  $.
\end{remark}

\subsection{Fast decay in the magnitude of coefficients}

As an application of this property we now prove two key results which expose
properties of the ST, and which we will use repeatedly. The
first, the so-called factorial decay estimate, quantifies the magnitude of the
terms $S\left(  x\right)  ^{\left(  k\right)  }$  with respect
to appropriate terms norms. We state it for paths of finite length, although
generalisations to less regular paths are possible, see for instance \cite[Sec. 9.1.1]{friz2010multidimensional}.

\begin{proposition}[Factorial decay]\label{prop:factorial_decay} Let $x$ be a path in $C_{1}\left([a,b],V\right)$. Then
for any $k$ in $\mathbb{N}$ we have 
\begin{equation}
\norm{S(x)^{(k)}}_{V^{\otimes k}} \leq \frac{\norm{x}_{1,[a,b]}^{k}}{k!},\label{fac decay}
\end{equation}
\end{proposition}

\begin{proof}
Set $L:=\norm{x}_{1,[a,b]}$ to be the $1$-variation of $x$, or its \emph{length}. Let $\tau$ be the reparamterisation defined by (\ref{reparam}) then $x^{\tau
}=x\circ\tau$ $:\left[  0,L\right]  \rightarrow V$ is Lipschitz continuous; see \Cref{ex:Lip}.
Using Remark \ref{Lipschitz reparam} we have that $S\left(x\right)
=S\left(x^{\tau}\right)  $ and, as the length of $x$ is also a
parameterisation-invariant quantity, it suffices to work with $x^{\tau}$ in
place  of $x$ when proving (\ref{fac decay}).  Any Lipschitz continuous
function is absolutely continuous and hence, by the Radon-Nikodym Theorem (e.g. \cite[Thm. 3.8]{folland1999real}),
there exists an integrable function $\left(  x^{\tau}\right)  ^{\prime}$ in
$L^{1}\left(  \left[  0,L\right]  ,V\right)  $ such that
\[
x_{u,v}^{\tau}=\int_{u}^{v}\left(  x_{r}^{\tau}\right)  ^{\prime}dr
\]
for all $\left[  u,v\right]  $ in $\left[  0,L\right]  .$ The derivative of
$x_{t}^{\tau}$ exists for almost every $t$ in $\left[  0,L\right]  $ and it
equals $\left(  x_{r}^{\tau}\right)  ^{\prime}$; the fact that $x^{\tau}$ has
Lipschitz norm bounded by one, cf. (\ref{Lip norm 1}), then shows that
$||\left(  x^{\tau}\right)  ^{\prime}||\leq1$ almost everywhere. Using these
observations, together with Fubini's Theorem (e.g. \cite[Thm. 2.37]{folland1999real}), it is easy to see that
\[
S\left(  x^{\tau}\right)  ^{\left(  k\right)  }=\int_{0<t_{1}<...<t_{k}%
<L}\left(  x_{t_{1}}^{\tau}\right)  ^{\prime}\left(  x_{t_{2}}^{\tau}\right)
^{\prime}...\left(  x_{t_{k}}^{\tau}\right)  ^{\prime}dt_{1}dt_{2}...dt_{k}\in
V^{\otimes k},
\]
and therefore
\begin{align*}
\norm{S\left(  x^{\tau}\right)  ^{\left(  k\right)
}}_{k}  & \leq\int_{0<t_{1}<...<t_{k}<L} \norm{\left(  x_{t_{1}}^{\tau}\right)  ^{\prime}\left(  x_{t_{2}}^{\tau
}\right)  ^{\prime}...\left(  x_{t_{k}}^{\tau}\right)  ^{\prime}}_{k}dt_{1}dt_{2}...dt_{k}\\
& \leq\int_{0<t_{1}<...<t_{k}<L}\norm{\left(  x_{t_{1}}^{\tau}\right)  ^{\prime
}} \cdot \norm{\left(  x_{t_{2}}^{\tau}\right)  ^{\prime}}\cdot \cdot \cdot \norm{\left(  x_{t_{k}}^{\tau
}\right)  ^{\prime}}dt_{1}dt_{2}...dt_{k}\\
& \leq\frac{L^{k}}{k!}.
\end{align*}
\end{proof}

\subsection{Uniqueness}

The next result we show is that for any two paths $x$ and $y$ in
$C_{p}\left(  V\right)$ which share an identical, strictly monotone
coordinate, one has $S\left(  x\right)  =S\left(  y\right)$ if and only if
the paths $x$ and $y$ are identically equal. This is the first instance of a
result, which we will revisit later in a more general form, which validates that the ST is injective. 

\begin{proposition}[Injectivity for time-augmented paths]
Suppose that $x$ and $y$ are two paths in $C_{p}\left([a,b],V\right)$, with $p \in [1,2)$, for
some choice of basis for $V,$ their expression in these coordinates $x=\left(
x^{1},...,x^{d}\right)  $ and $y=\allowbreak\left(  y^{1},...,y^{d}\right)  $
are such that
\[
x^{1}=y^{1}=\rho:\left[  a,b\right]  \rightarrow
\mathbb{R}
\]
for some strictly monotone path $\rho.$ Suppose that $x_{a}=y_{a}$, then
$$S\left(  x\right) =S\left(  y\right) \iff x=y.$$
\end{proposition}

\begin{proof}
It is obvious that two identical paths have the same ST. To prove the
converse we first notice, by considering $-x$ and $-y$ if necessary, that it
is sufficient to prove the result assuming that $\rho$ is strictly increasing.
Further by reparameterising $x$ and $y$ by the (necessarily continuous and
strictly increasing) function $\sigma=\rho^{-1}$ we can assume that the first
coordinate of both $x$ and $y$ is the identity function $i\left(  t\right)
\equiv t$. By making these simplification, we can write
\[
x_{t}=\left(  t,x_{t}^{-}\right)  \text{ and }y_{t}=\left(  t,y_{t}%
^{-}\right)
\]
where $x^{-}$ and $y^{-}$ are paths with values in $%
\mathbb{R}
^{d-1}.$ The aim is then to prove that $x^{-}$ and $y^{-}$ are equal whenever
$S\left(  x\right)  =S\left(  y\right).$ In the given basis of $V$ we have $S(x)^{w}=S(y)^{w}$ where $w$ is any multi-index of the form $w=(1,....,1,j)$ of length $k+2$. Written explicitly, this gives that  
\[
\int_{a}^{b}t^{k+1}d\left(  x_{t}^{-}\right)  ^{j}=\int_{a}^{b}t^{k+1}d\left(
y_{t}^{-}\right)  ^{j}%
\]
for every $k$ and every $j$. For simplicity we drop the index $j$ and refer only to $x^{-}$.
Integration-by-part then yields that
\[
b^{k+1}x_{b}^{-}-a^{k+1}x_{a}^{-}-\left(  k+1\right)  \int_{a}^{b}t^{k}%
x_{t}^{-}dt = b^{k+1}y_{b}^{-}-a^{k+1}y_{a}^{-}-\left(  k+1\right)  \int_{a}^{b}t^{k}%
y_{t}^{-}dt
\]
Noting that $x_{a}=y_{a}$ together with $x_{a,b}=S\left(  x\right)
_{a,b}^{\left(  1\right)  }=S\left(  y\right)  _{a,b}^{\left(  1\right)
}=y_{a,b}$ results in
\[
\int_{a}^{b}t^{k}x_{t}^{-}dt=\int_{a}^{b}t^{k}y_{t}^{-}dt\text{ for every
}k=0,1,2...
\]
and hence that $x^{-}-y^{-}$ as an element of $L^{2}\left(  \left[
a,b\right]  \right)$ is in the orthogonal complement $\mathcal{P}^{\perp}$ of
the space $\mathcal{P}$ of finite degree polynomials. Since $\mathcal{P}$ is
dense subspace of $L^{2}\left(  \left[  a,b\right]  \right)  $ it follows that
$\mathcal{P}^{\perp}=\left\{  0\right\}  $ and hence $x^{-}=y^{-}$ almost
everywhere. Using the continuity of $x^{-}$ and $y^{-}$ then gives that
$x^{-}$and $y^{-}$ are identically equal.
\end{proof}

So far we have seen that the ST satisfies two important analytic properties: the factorial decay allows for fast approximation of solutions of differential equations, while the invariance to reparameterization allows to remove possibly detrimental symmetries from the data. The ST also benefits from a rich algebraic structure which allow for efficient computations as we shall see next.

\section{Algebraic properties}

\subsection{Chen's relation}

At first sight, the ST looks like an object very difficult to compute in practice. However, it turns out that these computations can be carried out elegantly and efficiently by using simple classes of paths as building blocks for more complicated paths. To do this we will make use of an algebraic property of the ST called Chen's relation. We will later refine our view on this point but, for the moment, we assume that $x \in C_1([a,b],V)$ is a linear path, i.e. a straight line, so that for any $t \in [a,b]$ 
\begin{equation*}
    x_t = x_a + \frac{t-a}{b-a} (x_{b}-x_{a}).
\end{equation*}
Denote by $x_{a,b}:=i_1(x_b-x_a)$, where $i_1 : V \to T((V))$ is the canonical inclusion. Then, in this special case, we can calculate the ST of $x$ over $[a,b]$ very easily and show that 
\begin{align}\label{eqn:tensor_exp}
    S(x)_{[a,b]} = \exp (x_{a,b}),
\end{align}
where the map $\exp : T((V)) \to T_1((V))$ is the tensor exponential defined as
\begin{align}\label{eqn:tensor_exp_}
    \exp (v) := \sum_{k=0}^\infty \frac{v^{\otimes k}}{k!},
\end{align}
where $T_{1}\left(  \left(  V\right)  \right)  =\left\{  v\in T\left(  \left(
V\right)  \right)  :\pi_{0}v=1\right\}  .$ 

The proof of this observation follows because the derivative of $x$ is the constant vector $\frac{x_b-x_a}{(b-a)}$ on $[a,b]$, and therefore 
\begin{equation*}
S(x)^{(k)}_{[a,b]}=\frac{(x_b-x_a)^{\otimes k}}{(b-a)^{k}}\int_{a<t_{1}<...<t_{k}<b}dt_{1}...dt_{k}= \frac{(x_{a,b})^{\otimes k}}{k!}
\end{equation*}
Suppose we now have two paths $x$ in  $C_p([a,b],V)$ and $y$ in $C_{p}([b,c],V)$. Then there is a very natural binary operation on these paths which allows us to combine $x$ and $y$ into a single path. We define the \emph{concatenation} of $x$ and $y$, which we denote by $x \ast y$, to be the path in $C_p([a,c],V)$ given by
\begin{eqnarray*}
\left( x \ast y \right)_t &=&\left\{ 
\begin{array}{cc}
x_t & \text{if }t\in \left[ a,b\right] \\ 
y_t - y_b + x_b & \text{%
if }t\in \lbrack b,c]%
\end{array}%
\right.
\end{eqnarray*}
An important property is that the ST $S(x*y)$ can be described in terms of the ST of $x$ and the ST of $y$ and, in fact, it equals the tensor product $S(x) \cdot S(y)$. This property, which is called \emph{Chen's relation}, is the content of the following lemma.
\begin{lemma}[Chen's relation] \label{lemma:chen}
Let $1 \leq p <2$. Suppose that $x$ is in $C_p([a,b],V) $ and $y$ is in $C_p([b,c],V)$. Then 
\begin{equation}\label{eqn:chen}
S(x*y)_{[a,c]} = S(x)_{[a,b]} \cdot S(y)_{[b,c]}.
\end{equation}
\end{lemma}

\begin{proof}
We prove the result in the case $p=1$; the identity for arbitrary $1<p<2$ can be obtained by the now-familiar routine of approximating $x$ and $y$ in $C_{q}$ with $p<q<2$ by respective sequences bounded variation paths. See the final paragraph of the proof of \ref{lemma:repearam_invariance}. The identity (\ref{eqn:chen}) for $p=1$ then carries over to general $p$ by the continuity of Young integration in the $q$-variation topology. With this simplification being understood, we let $z :=x \ast y$. Then, for each $k\geq 1$, the level $k$ of the ST
\begin{align*}
S(z)_{[a,c]}^{(k)} &= \int_{a <u_1<...<u_k<c} dz_{u_1}  ... dz_{u_k}= \sum_{i=1}^k \int_{a <u_1<...<u_i<b<u_{i+1}<...<u_k<c} dz_{u_1}  ...  dz_{u_k}
\end{align*}
By Fubini's theorem, this is equivalent to
\begin{align*}
S^k(z)_{[a,c]} &= \sum_{i=1}^k \left( \int_{a <u_1<...<u_i<b} dx_{u_1} ...  dx_{u_i} \right) \otimes \left( \int_{b <u_{i+1}<...<u_k<c} dy_{u_{i+1}} ... dy_{u_k} \right)\\ 
& = \sum_{i=1}^k S(x)^{(i)}_{[a,b]} \otimes S(y)^{(k-i)}_{[b,c]}
\end{align*}
which concludes the proof.
\end{proof}

As we will discuss in details in later sections, the paths one deals with in many data science applications are obtained by piecewise linear interpolation of discrete time series\footnote{Other interpolation methods can be used, such as by cubic splines, rectilinear, Gaussian kernel smoothing etc.}.  More precisely, consider a time series $\mathbf{x}$ as a collection of points $\mathbf{x}_i \in \mathbb{R}^{m-1}$ with corresponding time-stamps $t_i \in \mathbb{R}$ such that $\mathbf{x} = ((t_0, \mathbf{x}_0), (t_1, \mathbf{x}_1), ..., (t_n, \mathbf{x}_n))$, and $t_0 < ... < t_n$ and let $T=t_n$. 

Let $x : [t_0, t_n] \to \mathbb{R}^m$ be the piecewise linear interpolation of the data such that $x_{t_i} = (t_i, \mathbf{x}_i)$. Using Chen's relation (\ref{eqn:chen}) one sees that
\begin{equation*}
    S(x)_{[t_0,t_n]} = S(x)_{[t_0,t_1]}\otimes S(x)_{[t_1,t_2]} \otimes ... \otimes S(x)_{[t_{n-1},t_n]}
\end{equation*}
On each sub-interval $[t_i,t_{i+1}]$ the path $x$ is linear, therefore \cref{eqn:tensor_exp} yields
\begin{equation}\label{eqn:comp_sig}
    S(x)_{[t_0,t_n]} = \exp_\otimes(x_{t_1}-x_{t_0})\cdot \exp_\otimes(x_{t_2}-x_{t_1}) \cdot ... \cdot \exp_\otimes(x_{t_n}-x_{t_{n-1}})
\end{equation}
Expression (\ref{eqn:comp_sig}) is leveraged in all the main  python packages for  computing signatures such as \texttt{esig} \cite{esig}, \texttt{iisignature} \cite{reizenstein2018iisignature} and \texttt{signatory} \cite{signatory}. 

We conclude this section by showing that at any level of truncation, the truncated tensor algebra coincides with the linear span of signatures truncated at that level. Several proofs of this statement can be found in the literature, see for example \cite[Lemma 3.4]{diehl2019invariants}. The one we propose below, (which one of us learnt from Bruce Driver), solely requires tools introduced so far in this chapter and the well-known fact that the exponential map in (\ref{eqn:tensor_exp_}) and the tensor logarithm $\log : T_1((V)) \to T((V))$ defined as 
\begin{align}\label{eqn:tensor_log}
\log(v) = \sum_{k=1}^{\infty}\frac{(-1)^{k+1}}%
{k}(v-\mathbf{1})^{k},
\end{align}
are mutually inverse.

\begin{lemma}
\label{Driver}For every $N$ in $\mathbb{N}$ the truncated tensor algebra $T_{N}\left(  V\right)  =$span$\left(
\mathcal{S}_{N}\right)  $ where $\mathcal{S}_{N}=\pi_{N}\mathcal{S}$. The set
$T\left(  V\right)  $ is a linear subspace of span$\left(  \mathcal{S}\right)
.$
\end{lemma}

\begin{proof}
(Driver \cite{Driver}) Let $a$ be an element in $V$ and let $x^a$ denote the linear path on $[0,1]$ defined as
$x^a_{t}=ta$. Recall for any $k \in \N$, that $\mathrm{i}_{k} : V^{\otimes k} \to T((V))$ denotes the canonical inclusion. Then, we have that
\begin{align*}
\mathrm{i}_1(a) =\log(\exp(\mathrm{i}_{1}(a))) 
 & =\log(S\left(  x^a\right)  )\\
& =\sum_{k=1}^{\infty}\frac{(-1)^{k+1}}{k}(S\left(  x^a\right)
-\mathbf{1})^{\otimes k}\\
& =\sum_{k=1}^{\infty}\sum_{l=0}^{k}\frac{(-1)^{k+l+1}}{k}\binom{k}{l}S\left(
x^a\right)  ^{\otimes l}\\
& =\sum_{k=1}^{\infty}\sum_{l=0}^{k}\frac{(-1)^{k+l+1}}{k}\binom{k}{l}S\left(
(x^{a})^{*l}\right)  \\
& =\sum_{l=0}^{\infty}\sum_{k=l}^{\infty}\frac{(-1)^{k+l+1}}{k}\binom{k}%
{l}S\left((x^{a})^{*l}\right)  \\
&:=\sum_{l=0}^{\infty}c_{l}S\left(  (x^{a})^{*l}\right),
\end{align*}
where $(x^{a})^{*l}$ is the path $x^a$ concatenated with itself $l$ times and where the fifth equality follows from the Chen's identity of \Cref{lemma:chen}. From this we have
\[
i_1^N(a)=\pi_{N} i_1(a)=\sum_{l=0}^{\infty}c_{l}\pi_{N}S\left((x^{a})^{*l}\right)  \in
\overline{\text{span}\left(  \mathcal{S}_{N}\right)  }=\text{span}\left(
\mathcal{S}_{N}\right),
\]
where $i_1^N : V \to T_N(V)$ is the canonical inclusion.
A variation on the same idea yields for $n\leq N$ that
\begin{align*}
    i_1^N(a_{1}...a_{n}) &= \pi_{N}i_n\left(  a_{1}...a_{n}\right)  \\
    &=\sum_{l_{1},...,l_{n}%
=0}^{\infty}c_{l_{1}}...c_{l_{n}}\pi_{N}S\left((x^{a_1})^{*l_1}\ast...\ast
(x^{a_n})^{*l_n}\right)\in\text{span}\left(  \mathcal{S}_{N}\right)  ,
\end{align*}
for any collection $a_{1},....,a_{n}$ in $V.$ 

From this it follows as above
that $T_{N}\left(  V\right)  \subset$ span$\left(  \mathcal{S}_{N}\right)  .$
\end{proof}

\begin{remark}
    In \Cref{prop:RKHS} in the next chapter we will show that this statement not only holds for the TST, but also for the (untruncated) ST.
\end{remark}

\subsection{The signature as a controlled differential equation}

Many of the important properties of the ST of a path $x$ can be derived most easily by observing that it is the solution of a specific CDE driven by $x$. Indeed, we will see that we could have used this as a starting point to \emph{define} the ST without ever introducing iterated integrals. Formally this CDE reads  
\begin{equation}\label{eqn:sig_cde}
dY_{t}=Y_{t} \cdot dx_{t}  \text{ on $[a,b]$, with } Y_{0}= \mathbf{1} \in T((V)) 
\end{equation}
where we continue to use $\cdot$ to denote the multiplication in $T((V))$. We will shortly introduce the analysis needed to ensure that the equation is well posed and then explore some of its properties. For the moment, we comment on the structure of the equation by observing that it is a linear CDE. 

If the tensor product $\cdot$ in equation (\ref{eqn:sig_cde}) were replaced by a commutative product, then we would expect the resulting solution to be the exponential of the driving signal $x$. This allows us to view the signature, at least impressionistically,  as a form of non-commutative exponential. The analysis below develops the tools needed to put this intuition on a firmer foundation. A key eventual goal will be to show, in a precise sense, the universality of the ST: under certain conditions, all continuous input-response relationships are expressible in terms of it. This class contains a wide family of CDEs, the canonical example with which we started this chapter. The reader may find is useful to look ahead at the statement of Theorem \ref{thm: universal approximation} at this stage. 

We begin by introducing some notation for the range of the signature map, which will be heavily used throughout the rest of this chapter.
\begin{notation}
Let $1\leq p< 2$. We use $\mathcal{S}_{p}\subset T((V))$ to denote the image of $C_{p}(V)$ under the ST.
\end{notation}
The following result then makes precise the CDE-formulation of the signature which we described above.

\begin{theorem}\label{sigCDE}
Suppose that $x$ is in $C_{p}([a,b],V)$ for some $p$ in $[1,2)$. Let $(E,\left\vert \left\vert \cdot \right\vert \right\vert) $ be a Banach algebra which is a subalgebra of $T((V))$ with the property that $\mathcal{S}_{p}\subset E$. Then the controlled differential equation 
\begin{equation}
dY_{t}=Y_{t} \cdot dx_{t}  \text{ on $[a,b]$, with } Y_{0}= A \in E \label{cde sig}
\end{equation}
admits a unique solution in $E$ which is given explicitly by $Y_{t}=A\cdot S(x)_{[a,t]}$. In particular, if $ A=\mathbf{1}=(1,0,...)$ then the ST uniquely solves (\ref{cde sig}).
\end{theorem}
\begin{remark}
For instance $E$ could be taken as the Banach algebra described in Example \ref{norm properties}.
\end{remark}
\begin{proof}
The map $f:E\rightarrow L\left(  V,E\right)  $ given by
$f \left(  Y\right)  =Y\cdot a$ is well defined; it takes values in $E$ as $E$ is a Banach algebra and, by virtue of being linear in $Y$, $f$ can
easily been seen to satisfy the conditions of the classical versions of Picard theorem for Young CDEs (e.g. \cite[Thm. 1.28]{lyons2007differential}). A solution to (\ref{cde sig}) thus exists uniquely in $E$.
On the other hand, any solution $Y_{t}=\left(  Y_{t}^{0},Y_{t}^{1},Y_{t}%
^{2},...\right)  $ to (\ref{cde sig}) must satisfy 
\[
Y_{t}^{0}=A^{0},\quad Y_t^{1}=A^{1}+A^{0}\int_{a}^tdx_u = (A\cdot S(x)_{[a,t]})^{1}
\]
and 
\[
Y_{u}^{k}
=A^{k}+\int_{a}^{t}Y_{u}^{k-1}dx_{u}\quad \text{ for }k=1,2,...
\]
An induction gives that $Y_{u}^{k}=(A \cdot S\left(  x\right)
_{\left[  0,u\right]  })^{k}$ for any $k$ in $\mathbb{N}$ and any $u$ in $\left[  a,b\right]$.
%
\end{proof}
We give an example of how this characterisation of the ST can be used to extract useful algebraic properties. To do so we first observe that, in addition to the binary operation of concatenating two paths, there is another natural unary operation consisting of running the path backwards starting from its end point. Symbolically, given $x \in C_{p}([a,b],V)$ we define the time reversal of $x$ to be $\overleftarrow{x}$ where 
\[
\overleftarrow{x_{t}}:=x_{a+b-t}  \text{ for $t \in [a,b]$}.
\]
It is easily verified that $\overleftarrow{x}$ again belongs to $C_{p}([a,b],V)$. The following result shows that the signature of $\overleftarrow{x}$ is the mutliplicative inverse of the ST of $x$ in $T((V))$.
\begin{lemma}\label{lemma:inverse}
Let $p \in [1,2)$ and let $x \in C_{p}([a,b],V)$. Then 
$$S(x)_{[a,b]}\cdot S(\overleftarrow{x})_{[a,b]}=S(\overleftarrow{x})_{[a,b]} \cdot S(x)_{[a,b]}=\bf{1}.$$
\end{lemma}
\begin{proof}

Let $A=S\left( x\right) _{[a.b]}$ then, by using Theorem \ref{sigCDE}, we have $%
U_{t}=A\cdot S\left( \overleftarrow{x}\right) _{[a,t]}$ uniquely solves the
CDE\ 
\[
dU_{t}=U_{t}\cdot d\overleftarrow{x}_{t},\text{ on }\left[ a,b\right] \text{%
, started at }U_{a}=A.
\]
We need to prove that $U_{b}=\mathbf{1.}$ To do this we define $
V_{t}:=U_{a+b-t}$ for $t$ in $\left[ a,b\right] $ and observe that $V$
solves 
\begin{equation}
dV_{t}=V_{t}\cdot dx_{t}\text{ started at }V_{a}=U_{b}. \label{cdeV}
\end{equation}
Because $V_{b}=S\left( x\right) _{\left[ a,b\right] }$ and $S\left(
x\right) _{\left[ a,\cdot \right] }$ solves (\ref{cdeV}) we have by
uniqueness that $V_{a}=S\left( x\right) _{\left[ a,a\right] }=\mathbf{1.}$
It follows that $S\left( x\right) _{\left[ a,b\right] }$ $\cdot S\left( 
\overleftarrow{x}\right) _{[a,b]}=\mathbf{1}$ for any \thinspace $x$.
The fact that $\overleftarrow{\overleftarrow{x}}=x$ completes the proof.    
\end{proof}
\begin{remark}
Readers who may be sceptical about the advantages of working with the analytical, CDE approach to the signature are invited to compare the proof of this result with the same proof attempted using algebraic methods.
\end{remark}

\subsection{The shuffle identity}

The tensor product $\cdot$ is not the only product that makes sense on $T((V))$. In this section we will discuss another product, called the \emph{shuffle product} $\shuffle$, which will turn out to be a fundamental tool for some more advanced manipulations of the ST, and is an important tool underpinning many version of the universal approximation theorems using signature features that arise in machine learning.

We first defined this product, and then explain through examples how it is used.

\begin{definition}[Shuffle product]\label{shuffle def}
    We define $\shuffle:T(V)\times T(V)\mapsto T(V)$ by 
\begin{align}
f\shuffle r = r \shuffle f & = rf \text{ for any $r\in \mathbb{R}$ and $f \in V,$} \text{and then  extend it inductively by } \nonumber \\
    f \shuffle g & =  (f_{-} \shuffle g)\cdot a + (f \shuffle g_{-})\cdot b
\end{align}
for any $f \in V^{\otimes k} $ and $g \in V^{\otimes l}$ of the form $f=f_{-}\cdot a $ and $g=g_{-} \cdot b$, where $a$ and $b$ are in $V$. Given this definition, $\shuffle$ extends uniquely to an algebra product on $T(V)\times T(V)$ by linearity in each term of the product.
\end{definition}

 For reasons that will be clear soon, it will be useful to work with the $\shuffle$ on the co-tensor algebra $T(V^{*})$. This does not affect the definition above, which holds when $V$ is an arbitrary vector space (and so, in particular can be taken to be $V^{*}$).

Why is the shuffle product relevant to the ST? One important tool which we encounter in foundational calculus is integration-by-parts. To put this familiar result in the present context, suppose that $x$ is a smooth path in $V$ and that $f$ and $g$ are linear functionals in $V^{*}$, then classical integration-by-parts is the identity
\[
(f,x_{a,b})(g,x_{a,b})=\int_{a}^{b}(f,x_{a,s})d(g,x_{s})+\int_{a}^{b}(g,x_{a,s})d(f,x_{s}),
\]
and this be rewritten in terms of shuffle products as
\[
(f,S(x)_{[a,b]}^{(1)})(g,S(x)_{[a,b]}^{(1)})=(f\shuffle g, S(x)_{[a,b]}^{(2)})
\]
This observation can be generalised: products of iterated integrals can be re-expressed as a linear combination of higher order iterated integrals using integration-by-parts. The shuffle product can be used to succinctly capture the algebraic identities which result from this observation.

\begin{theorem}[Shuffle identity]\label{lemma:shuffle_identity}
Suppose that $x \in C_p([a,b],V)$ for $p$ in $[1,2)$. For any two linear functionals $f$ and $g$ in $T((V))^{*}\cong T(V^{*})$ it holds that 
\begin{equation}\label{shuff id}
(f,S(x)_{[a,b]})(g,S(x)_{[a,b]})=(f\shuffle g, S(x)_{[a,b]})
\end{equation}
\end{theorem}
\begin{proof}

It suffices to prove the result for arbitrary $f$ $\in V^{\ast \otimes k}$
and $g\in V^{\ast \otimes l}$ the general result then follows from the
distributivity of the shuffle product. 

We will prove, by double induction, the following statement 
\begin{center}
$H(m,n)$:  $\forall 0\leq k\leq m,0 \leq l\leq n$, identity (\ref{shuff id}) holds $\forall f\in V^{\ast \otimes k}, g\in V^{\ast \otimes l}.$
\end{center}
We can check that $H(m,0)$ and $H(0,n)$ hold using elementary properties of scalar multiplication in $T\left( V^{\ast }\right)$. We assume therefore that $H(m-1,n)$ and $H(m,n-1)$ hold and verify $H(m,n)$. To do so, take $f\in V^{\ast \otimes
m}$ and $g\in V^{\ast \otimes
n}$ and use the definition of the shuffle product to see that 
\begin{equation}
\left( f \shuffle g,S\left( x\right) _{a,t}\right) =\left( ( f_{-} \shuffle g ) \cdot
p,S\left( x\right) _{a,t}\right) +\left( (f \shuffle g_{-}) \cdot
q,S\left( x\right) _{a,t}\right) ,  \label{shuff id 1}
\end{equation}%
where $f=f_{-}\cdot p$ and $g=g_{-}\cdot q.$ We can then simplify the first
term by using $H(m-1, n)$ and the second using $H(m, n-1)$ to obtain 
\begin{equation}
\left( f \shuffle g,S\left( x\right) _{a,t}\right)
=\int_{a}^{t}G_{u}dF_{u}+\int_{a}^{t}F_{u}dG_{u}  \label{IBP}
\end{equation}
with $F_{s}:=\left( f,S\left( x\right) _{a,s}\right) $ and $G_{s}:=\left(
g,S\left( x\right) _{a,s}\right) .$ Integration-by-parts then yields the
desired conclusion 
\[
\left( f \shuffle g,S\left( x\right) _{a,t}\right) =F_{t}G_{t}-F_{a}G_{a}=\left(
f,S\left( x\right) _{a,t}\right) \left( g,S\left( x\right) _{a,t}\right) .
\]

\end{proof}


\begin{remark}
Given a basis $\{f_{i}\}_{i=1}^{d}$ for $V^{*}$ we could also define the shuffle product directly on $T(V^{*})$ by letting $I=(i_{1},...i_{n})$ and $J=(i_{n+1},...i_{n+m})$ be arbitrary multi-indices in $\{i,...,d\}$ and then, recalling Notation \ref{tensor basis}, by defining 
\begin{equation*}
    f_{I} \shuffle f_{J} := \sum_{\sigma \in \text{Sh}(n,m)} f_{i_{\sigma(1)}}...  f_{i_{\sigma(n+m)}}.
\end{equation*}
The summation is over all $(n,m)$-shuffles ; that is over the subset of the permutations of the $\{1,...,n+m\}$ which satisfy $\sigma(1)<...<\sigma(n)$ and $\sigma(n+1)<...<\sigma(n+m)$. It can be verified that this defines an element of $T(V^{*})$ in a way that does not depend on the choice of basis. The resulting 
definition coincides with Definition \ref{shuffle def}.
\end{remark}
If we work with a particular basis $\{e_{i}\}_{i=1}^{d}$ of $V$, then the shuffle product identity for the ST can be expressed more tersely in terms coordinate iterated integrals by
\[
S(x)^{I}_{[a,b]}S(x)^{J}_{[a,b]}=S(x)^{I \shuffle J}_{[a,b]}.   
\]
This expression holds for any pair of multi-indices $I$ and $J$; the term $S(x)^{I \shuffle J}_{[a,b]}$ is defined to be $(e_{I}^*\shuffle e_{J}^{*}, S(x)_{[a,b]})$.


We elucidate this definition with an example.
\begin{example}\label{ex:shuffle}
\normalfont
Let $\{e_{i}\}_{i=1}^{d}$ be a basis of $V$ and let its dual basis be denoted by $\{e_{i}^{*}\}_{i=1}^{d}$. Suppose that $e_i^* e_j^*$ and $e_{k}^{*}$ are elements of $(V^{*})^{\otimes 2}$ and $V^{*}$ respectively, then
\begin{equation*} S(x)^{(i,j)}_{[a,b]} S(x)^{(k)}_{[a,b]} = \left(\int_{a\leq u_1 \leq u_2 \leq b}dx_{u_1}^{(i)}dx_{u_2}^{(j)}\right)\left(\int_{a\leq u \leq b}dx_{u}^{(k)}\right).
\end{equation*}
This expression can be rewritten as the sum of three third-order iterated integrals by "shuffling" the integrands
\begin{align*}
& \int_{a\leq u_1\leq u_2\leq u \leq b}dx_{u_1}^{(i)}dx
_{u_2}^{(j)}dx_{u}^{(k)} \\
& + \int_{a \leq u_1\leq u\leq u_2\leq
b}dx_{u_1}^{(i)}dx_{u}^{(k)}dx_{u_2}^{(j)} +\int_{a\leq
u\leq u_1\leq u_2\leq b}dx_{u}^{(k)}dx_{u_1}^{(i)}dx_{u_2}^{(j)}.
\end{align*}
In other words,
\begin{equation*}
\left(e_i^* e_j^*, S(x)^{(2)}_{[a,b]}\right)\left(e_k^*, S(x)^{(1)}_{[a,b]}\right) = \left(e_i^* e_j^*  e^*_k + e_i^*  e_k^*  e^*_j + e_k^*  e_i^*  e^*_j, S(x)^{(3)}_{[a,b]}\right).
\end{equation*}
Note that the tuples $(i,j,k), (i,k,j), (k,i,j)$ indexing the dual basis elements appearing in this expansion are obtained by shuffling $(i,j)$ and $(k)$ while retaining the order of the indices within each tuple, thus explains the nomenclature of the shuffle product.
\end{example} 
The shuffle property can be used to prove the following useful lemma.
\begin{lemma}
\label{lin func}Let $1\leq p<2$ and suppose $h_{1},..,h_{n}$ is a finite
collection of distinct elements of $\mathcal{S}_{p}\subset T\left( \left(
V\right) \right) .$ Then there exists a linear functional $f$ in $T\left(
\left( V\right) \right) ^{\ast }$ such that $f\left( h_{1}\right) =1$ and $%
f\left( h_{i}\right) =0$ for all $i=2,...,n.$
\end{lemma}

\begin{proof}
There exists $N \in \mathbb{N}$ finite such that $g_{i}=\pi_{\leq N} h_{i}$ remain distinct for $i=1,...,n$ under $\pi_{\leq N}: T((V)) \rightarrow T_{N}(V):= \oplus_{j=0}^{N} V^{\otimes j}$, the canonical projection to the truncated tensor algebra.  For any $i\neq 1,$ the vectors $g_{1}$ and $g_{i}$ are linearly independent.
To see this let $\pi _{0}:T_{N}\left( \left( V\right) \right) \rightarrow 
\mathbb{R}
$ be the canonical projection, then we have $\pi _{0}g_{1}=$ $\pi _{0}g_{i}=1
$ by the fact that $h_{1}$ and $h_{i}$ are both images of paths by the signature. While if $%
g_{1}=\lambda g_{i}$ for a scalar $\lambda ,$ it must also be true that $%
1=\pi _{0}g_{1}=\pi _{0}\left( \lambda g_{i}\right) =\lambda $; i.e. $%
g_{1}=g_{i}$, contradicting their distinctness. It follows from the
Hahn-Banach Theorem (e.g. \cite[Thm. 5.6]{folland1999real}) that for every $i=2,...,n$ a linear functional $l_{i}$
in $T_{N}\left( V\right) ^{\ast }$ exists with the property that $%
l_{i}\left( g_{1}\right) =1$ and $l_{i}\left( g_{i}\right) =0.$ Then $f_{i}:=l_{i}\circ \pi_{\leq N}$  for $i=1,...,n$ are well defined linear functional on $T((V))$ and this allows
us to define $f$ in $T\left( \left( V\right) \right) ^{\ast }$ by the
shuffle product $f=\shuffle_{j=2}^{n}f_{j}.$
It can be checked that $f$ has the required properties: the fact that each $%
h_{i}=S\left( x_{i}\right) $ for a path $x_{i}$ in $C_{p}\left( V\right) $
allows us to use Theorem \ref{lemma:shuffle_identity} to see that%
\[
f\left( h_{i}\right) =\prod_{j=2}^{n}f_{j}\left( h_{i}\right)
=\left\{ 
\begin{array}{cc}
1 & i=1 \\ 
0 & i\neq 1%
\end{array}%
\right. .
\]
\end{proof}
An important consequence of this is the following result. 
\begin{proposition}
\label{lin indep}Suppose $1\leq p<2$. Any finite set of distinct elements $\left\{
h_{1},..,h_{n}\right\} \subset \mathcal{S}_{p}$ is a linearly independent
subset of $T\left( \left( V\right) \right) .$
\end{proposition}

\begin{proof}
Suppose that $h=\sum_{i=1}^{n}\lambda _{i}h_{i}=0$ in $T\left( \left(
V\right) \right) $ and assume that $\lambda _{j}$ is non-zero for some $j$. By
relabelling the $h_{i}\,$s if necessary we can assume that $\lambda _{1}\neq
0.$ Then, letting $f$ be the linear functional in Lemma \ref{lin func}, we
arrive at the contradiction $0=f\left( h\right) =\lambda _{1}.$ It follows
that $\lambda _{i}=0$ for every $i=1,..,n.$
\end{proof}

As an immediate corollary we can obtain that any discrete $\mathcal{S}_{p}$%
-valued random variable is uniquely characterised by its expected value.

\begin{corollary}
Let $H$ be a $\mathcal{S}_{p}$-valued random variable defined on a
probability space $\left( \Omega ,\mathcal{F},\mathbb{P}\right) $ with
finitely-supported law $\mu $ described by $\mu =\sum_{i=1}^{n}p_{i}\delta
_{h_{i}}$, with distinct $h_{1},..,h_{n} \in \mathcal{S}_{p}$. Then $\mu$ is uniquely determined by its expectation $\mathbb{E}\left[ H\right] =
\sum_{i=1}^{n}p_{i}h_{i}$.
\end{corollary}

\begin{proof}
Let $G$ be another random variable with law $\nu =\sum_{j=1}^{m}q_{j}\delta
_{g_{j}}$ supported on distinct $g_{1},...,g_{m}$ in $\mathcal{S}_{p}.$ By Proposition \ref{lin indep} we see that  $\sum_{i=1}^{n}p_{i}h_{i}=
\sum_{j=1}^{m}q_{j}g_{j}$ if and only if both $n=m$ and the sets 
$$\left\{
\left( p_{i},h_{i}\right) :i=1,...,n\right\} \quad \text{and} \quad \left\{ \left(
q_{i},g_{i}\right) :i=1,...,n\right\} $$ 
are equal. It follows that $\mathbb{E%
}\left[ H\right] =\mathbb{E}\left[ G\right] $ if and only if $\mu =\nu .$
\end{proof}

\section{Unparameterised paths}
We have seen in Lemma \ref{lemma:repearam_invariance} that the ST of a path $x$ is invariant under reparameterisation. It is also immediate from its definition that when $x$ is translated by any constant path its signature will be unchanged. Nothing is lost therefore by only considering signatures of paths in $C_{0,p}(V)$, the subspace of $C_{p}(V)$ consisting of paths which start at the zero vector $0 \in V$, and we will do this throuhgout this section. A special role will be played by $o$, the path that is constantly zero $o \equiv 0$. 

The notion of ST allows us define a relation on $C_{0,p}(V)$ by declaring two paths to be related if they have the same signature. It is easily proved that this is an equivalence relation which we will denote by $\sim$. We denote the equivalence class containing a path $x$ by $[x]$. A natural question prompted by this discussion is: what elements does $[x]$ contain? We know already that distinct paths can have equal signatures; using Lemma \ref{lemma:repearam_invariance} any reparameterisation of $x$ will have the same signature as $x$ but will, in general, be distinct from $x$ as an element of $C_{p}(V)$.  It can happen however that two paths share the same signature while not being re-parameterisations of one another in the strict sense of Lemma \ref{lemma:repearam_invariance}. For example, if $x \neq o$ we have that $x*\overleftarrow{x} \neq o$, while from Lemma \ref{lemma:inverse} we have 
\begin{equation}\label{retracing}
S(o)=\mathbf{1}=S(x*\overleftarrow{x}).
\end{equation}
More generally, by Chen's identity any path which contains segments that exactly retrace themselves will have the same signature as the path obtained by excising both of those segments. We note also that, again in view of Lemma \ref{lemma:inverse}, $\sim$ could also be defined by saying that $x \sim y$ if 
\begin{equation}
S(x*\overleftarrow{y})=\mathbf{1},
\end{equation}
which, when it holds, also means that $S(y*\overleftarrow{x})=\mathbf{1}$. Elements of the equivalence class $[o]$ are called \emph{tree-like paths}.

It turns out that $\sim$-equivalence coincides with an ostensibly-unconnected equivalence relation called tree-like equivalence. This latter notion makes no direct reference to the ST and is formulated purely in term of the properties of the path. Nonetheless the two notions are the same and can be viewed as a general notion of reparameterisation which simultaneously captures both the classical concept of Lemma \ref{lemma:repearam_invariance} and the idea of excising retracings expressed in the example (\ref{retracing}). We will not use the notion of tree-like equivalence in what follows; the key result is the following and the interested reader can consult the references given. 
\begin{theorem}[\cite{hambly2010uniqueness} $p=1$, \cite{boedihardjo2016signature} $p>1$]
\label{HL}Let $1 \leq p <2 $, then tree-like equivalence defines an equivalence relation on $C_{0,p}(V)$. It coincides with the equivalence relation $\sim$ defined by the equality of signatures. In the case $p=1$, there exists an element of each tree-like equivalence class $[x]$ which has minimal length. This element in unique up to parameterisation and is called the tree-reduced representative of $[x]$.
\end{theorem}
We now introduce the space of unparameterised paths.
\begin{definition}[Unparameterised paths]
Suppose $1 \leq p <2 $. We define the space of $p$-unparameterised paths, denoted by $\mathcal{C}_p$, to be the quotient space 
$$\mathcal{C}_p := C_{0,p}/_\sim = \left\{[x]:x\in C_{0,p}\right\}.$$
\end{definition}
There is a natural group structure on the space of unparameterised paths.
\begin{lemma}\label{lemma: path group}
Suppose $1 \leq p <2 $. The concatenation operation $*$ on $C_{p}(V)$ induces a well defined binary operation on $\mathcal{C}_p$ by $[x]*[y]=[x*y]$. Moreover $(\mathcal{C}_p,*)$ is a group in  which the identity element is $[o]$ and where $[\overleftarrow{x}]$ is the well-defined inverse of $[x]$.
\end{lemma}
\begin{proof}
That $*$ is well defined on $\mathcal{C}_p$ amounts to showing that $[x*y]=[x'*y']$ for any $x' \in [x]$ and $y' \in [y]$ and this is a direct consequence of Chen's relations, Lemma \ref{lemma:chen}. The group axioms can be verified similarly using properties of the tensor product, and the final assertion about the inverse is obtained from Lemma \ref{lemma:inverse}.
\end{proof}

\subsection{Functions on unparameterised paths}
We will be interested to use the ST representation of a path to learn functions on path space. An immediate obstacle to implementing this idea is that the ST cannot distinguish between different paths in the same tree-like equivalence class. This compels us to work for the moment on functions defined on unparameterised path space. 
\begin{notation}
Suppose that $X$ is a set and let $1\leq p<2$. We use $X^{\mathcal{C}_{p}}$
to denote the set of functions from $\mathcal{C}_{p}$ into $X.$
\end{notation}
\begin{remark}
If $X$ is an algebra over $\mathbb{R}$ then so is $X^{\mathcal{C}_{p}}$ with a product given by the pointwise product of functions. 
\end{remark}
One way to obtain functions in $X^{\mathcal{C}_{p}}$ is to study the subset of functions on the original space $C_{0,p}(V)$ which are invariant under tree-like equivalence. Fortunately there is an abundance of such functions; the examples covered by the results below describe an immediately-useful collection of them.
\begin{lemma}\label{algebra}
Let  $1\leq p<2.$ Suppose that elements of the dual space $T\left( \left(
V\right) \right) ^{\ast }$  are identified with elements of $\mathbb{R}^{%
\mathcal{C}_{p}}$ by defining for $f \in T((V))^{*}$ the function
\begin{equation}\label{linfunc}
\Phi _{f}:\left[ x\right] \mapsto \left( f,S\left( x\right) \right),
\end{equation}
where as usual $(\cdot,\cdot)$ denotes the canonical pairing of $T((V))^{*}$ and $T((V))$. Then the class $\mathcal{A=}\left\{ \Phi _{f}:f\in T\left( \left( V\right)
\right) ^{\ast }\right\} $ is a subalgebra of $\mathbb{R}^{\mathcal{C}_{p}}$
which contains the constant
functions and separates points in $\mathcal{C}_{p}$, i.e. for any two distinct paths $[x] \neq [y]$ in $\mathcal{C}_{p}$, there exists a linear functional $f \in T((V))^*$ such that $\Phi_f(x) \neq \Phi_f(y)$.
\end{lemma}

\begin{proof}
The linear functional $e=\pi_{0}$, i.e. $e\left( A\right) =A^{0}$ for $A=(A^{0},A^{1},...)\in T((V))$, has the property that $\Phi
_{\lambda e}\equiv \lambda $ for any $\lambda $ in $\mathbb{R}$ ensuring that $\mathcal{A}$ contains all constant functions. That $\mathcal{A}$
separates points is immediate from the definition of the equivalence classes
along with the fact that $S\left( x\right) =S\left( y\right) $ if and only
if $\left( f,S\left( x\right) \right) =(f,S\left( y\right) )$ for all $f$ in 
$T\left( \left( V\right) \right) ^{\ast }.$ It is self-evident using
linearity that  $\Phi _{\lambda f}=\lambda \Phi _{f}$ and $\Phi _{f+g}=\Phi
_{f}+\Phi _{g}$ for all $f$ and $g$ in $T\left( \left( V\right) \right)
^{\ast }$ and $\lambda $ in $\mathbb{R}$. \ Finally $\mathcal{A}$ is closed
under taking the pointwise product of its elements since $\Phi _{f}\Phi
_{g}=\Phi _{f \shuffle g}$ by Lemma \ref{lemma:shuffle_identity}. 
\end{proof}

The following fundamental result underpins the conceptual framework of using the signature features to approximate continuous functions on unparameterised paths.

\begin{theorem}[Universal approximation with signatures]\label{thm: universal approximation} Let $1\leq p <2$ and suppose that $\chi$ is a collection of subsets which form a topology on $\mathcal{C}_{p}$. Assume that $K$ is a compact subspace of $\mathcal{C}_{p}$ and, for $f \in T((V))^{*}$, let $\Phi_{f}|_{K}$ denote the restriction of the function (\ref{linfunc}) to $K$. Let $C(K)$ denote the space of continuous functions on $K$ with the topology of uniform convergence. If the set $\mathcal{A}_{K}:=\{\Phi_{f}|_{K}: f \in T((V))^{*}\}$ is a subset of $C(K)$, then it is a dense subset.
\end{theorem}
\begin{proof}
By using an identical argument as in the proof of Lemma \ref{algebra}, it is easily seen that $\mathcal{A}|_{K}$ is a subalgebra of $C(K)$ which separates points and  contains the constant functions. That $\mathcal{A}|_{K}$ is dense then follows from the Stone-Weierstrass Theorem.
\end{proof}

\begin{remark}\label{rq1}
Versions of the Stone-Weierstrass Theorem hold under weaker assumptions on $K$. For example, if $K$ is only assumed to be locally compact then the same reasoning allows one to prove that $\mathcal{A}|_{K}$ is dense in the uniform topology in the space of continuous functions which vanish at infinity.
\end{remark}

\begin{example}[CDEs as functions on unparameterised path space]\label{prop: ode}
We return to the motivating example which which we started this chapter, namely the controlled differential equation (\ref{differential})
\begin{equation}
dy_{t}=f_{dx_{t}}\left(  y_{t}\right)  \quad \text{ started at }y_{a} \in W%
\end{equation}
in which recall that $W$ is a finite-dimensional vector space and $f:V\rightarrow 
\mathcal{T}\left( W\right) $ is a linear map into the space of smooth vector fields on $W$. In the context of our present discussion it is relevant to ask whether the function which maps $x$ to the solution of the CDE at $t=b$,
\begin{equation}\label{eqn:ito_map}
 C_{p}([a,b],V) \ni x \mapsto y_{b} \in W,
\end{equation}
descends to a function from $\mathcal{C}_{p}$ into $W$. This amounts to showing that the function (\ref{eqn:ito_map}) is constant on on every equivalence class of $\sim_{\tau}$. 

To understand when this might be the case, we make further assumptions on $f$. Assume there exists $C<\infty$ such that for every $k\in\mathbb{N} $ the derivatives $f^{k}$ satisfy 
\begin{equation}
\sup_{y\in W}\left\vert\left\vert f^{\left( k\right) }I\left( y\right) \right\vert\right\vert\leq C^{k}, 
\label{derivative bound}
\end{equation}
where $\left\vert \left\vert \cdot \right\vert
\right\vert $ denotes the operator norm and $I:W\to W$ is the identity function on $W.$ Then one can easily show by induction that 
\[
y_{b} =y^{N}_{b} +\int_{a< t_{1}<\dots<
t_{N+1}<b}f_{dx_{t_{1}}dx_{t_{2}}\dots dx_{t_{N}}dx_{t_{N+1}}}^{\left(
N+1\right) }I\left( y_{t_{1}}\right) , 
\]
where 
\[
y^{N}_{b}
:=y_{a}+\sum_{k=1}^{N}f_{S\left( x\right)^{(k)} }^{\left( k\right)
}I\left( y_{a}\right). 
\]
Condition (\ref{derivative bound}) together with the estimate of Lemma \ref{fac decay} ensures that
\[
\left\vert y_{b} -y^{N}_{b} \right\vert \leq\frac{%
C^{N+1}L_{x}^{N+1}}{\left( N+1\right) !}.
\]
By taking the limit as $N \to \infty$ we have that $y_{b}$ is the convergent series 
\begin{equation}
y_{a}+\sum_{k=1}^{\infty}f_{S\left( x\right)^{(k)} }^{\left( k\right)
}I\left( y_{a}\right)    \label{series}
\end{equation}
and therefore in particular the map $x \mapsto y_{b}$ will be constant on every set $[x]$.
\end{example}

\subsection{Topology on unparameterised paths}\label{topology}

To use Theorem \ref{thm: universal approximation} we need to make a choice a topology on $\mathcal{C}_{p}$. There is no canonical way to do this -- we must choose one which is suited to the task at hand -- and different choices will present different classes of continuous functions and of compact subspaces. In the next chapter on signature kernels we will see a situation where a choice of topology is suggested by the application but, for the moment, we work with a minimal selection which reflects the fact that Theorem \ref{thm: universal approximation} requires that all the functions in the class $\mathcal{A}$ must all be continuous.

\begin{definition}
    The product topology on $T((V)):=\prod_{i=0}^{\infty}V^{\otimes i}$ is the weakest topology (i.e. the one having fewest open sets) such that all the canonical projections $\pi_{k}:T((V)) \mapsto V^{\otimes k}$ are continuous. This then induces a topology on $\mathcal{C}_{p}$ by equipping  $\mathcal{S}_p$ with the subspace topology of the product topology on $T((V))$, and then by requiring that signature is a topological embedding when viewed as a map from  $\mathcal{C}_{p}$ onto ${S}_{p}$. With a mild abuse of terminology, we will refer to this as the product topology on $\mathcal{C}_{p}$ and use the notation $\chi_{_{\text{pr}}}$ to refer to the collection of open subsets it defines. 
\end{definition}

The continuity of the canonical projections defining the product topology immediately gives that linear functional on $T((V))$ induces a continuous function on $\mathcal{C}_{p}$.

 \begin{lemma}
Every function of the form (\ref{linfunc}) is continuous from the topological space $(\mathcal{C}_{p},\chi_{_{\text{pr}}})$ into $\mathbb{R}$ (when  $\mathbb{R}$ is equipped its usual topology).
 \end{lemma}
 \begin{proof}
 With respect to a dual basis we can write any $f \in T((V))^{*}$ as a sum $f=\sum_{I} a_{I}e_{I}^{*} $ in which all but finitely many of the $a_{I}$ are zero. It then follows that 
 \[
 \phi_{f}([x])=\sum_{I}a_{I}S(x)^{I}.
 \]
 Each coordinate iterated integral is a function from $\mathcal{C}_{p}$ to $\mathbb{R}$ and can be expressed as the composition
 $e_{I}^{*}\circ \pi_{|I|} \circ S$. By definition of the topology $\chi_{_{\text{pr}}}$, $S$ is continuous from $\mathcal{C}_{p}$ onto $\mathcal{S}_{p}$. As the projection $\pi_{|I|}:\mathcal{S}_{p} \mapsto V^{\otimes|I|}$ and $e_{I}^{*}:V^{\otimes|I|}\mapsto \mathbb{R}$ are continuous so is $\phi_{f}$.
 \end{proof}
 \begin{corollary}
$(\mathcal{C}_{p},\chi_{_{\text{pr}}})$ is a Hausdorff space.
 \end{corollary}
 \begin{proof}
This follows at once from the fact that the functions in $\mathcal{A}$ separate points in $\mathcal{C}_{p}$ and are continuous with respect to $(\mathcal{C}_{p},\chi_{_{\text{pr}}})$ 
 \end{proof}
 The following lemma captures some further basic facts about this product topology.
 \begin{lemma}\label{lem: sep and met}
 The topological space $(\mathcal{C}_{p},\chi_{_{\text{pr}}})$ is both separable and metrisable.
 \end{lemma}
\begin{proof}
We need to prove that $\mathcal{S}_{p}$, endowed with the subspace topology in $\prod_{i=0}^{\infty}V^{\otimes i}$, has these properties. Recall that a Polish space is a topological space that is separable and completely metrisable. Then the fact that the product of a countable collection of Polish spaces is itself a Polish space (with the product topology) shows that $\prod_{i=0}^{\infty}V^{\otimes i}$ is Polish.  That any subspace of a Polish space is also separable and metrisable allows us to conclude the argument.
\end{proof}
Convergence of a sequence $([x_{n}])_{n=1}^{\infty}$ in $(\mathcal{C}_{p},\chi_{_{\text{pr}}})$ is characterised by pointwise convergence of the terms in the ST, i.e. $[x_{n}] \rightarrow [x]$ as $n \rightarrow \infty$ if and only if
\[
S(x_{n})^{(k)} \rightarrow S(x)^{(k)} \text{ as $n \rightarrow \infty$,  for every $k \in \mathbb{N}$}
\]
This choice of topology also ensures that the group operations on unparameterised path sapce are continuous.

 \begin{proposition}\label{prop: group cts} The space
$(\mathcal{C}_{p},\chi_{_{\text{pr}}})$ with the operations $\left[ \gamma\right] \cdot\left[ \sigma\right] :=\left[\gamma\ast\sigma\right]$ and $\left[\gamma\right]^{-1}:=\left[\overleftarrow{\gamma}\right]$ is a topological group.
\end{proposition}

\begin{proof}
We provide a proof from first principles. First for $\left[ \cdot\right] ^{-1}$, we define a map on $T_{1}((V))=\{a\in T((V)): a_{0}=1\}$ by:
\[
\psi:T_{1}((V)) \longrightarrow T_{1}((V)), \text{such that }  
\psi: a \mapsto \mathbf{1} + \sum_{n=1}^\infty (\mathbf{1}-a)^n,
\]
which has the property that $(\psi \circ S)(x)=S(\overleftarrow{x})$ allowing us to write 
\[
\left[ \cdot \right]^{-1} =S^{-1} \circ \psi \circ S : \mathcal{C}_{p}\xrightarrow{} \mathcal{C}_{p}
\]
where $S^{-1}$ denote the inverse of the signature  when regarded as function from $\mathcal{C}_{p}$ onto $\mathcal{C}_{p}$. The definition of the product topology ensures that $\mathcal{S}$ and $S^{-1}$ are both continuous. It therefore suffices to show continuity of $\psi$ on $T_{1}((V))$ again in the subspace topology with respect to $T((V))$ and, by definition of the product topology, the amounts to show continuity of the projection onto each $V^{\otimes n}$ of $\psi$. It is not difficult to see that if $a=(1,a_{1},a_{2},...)$ then $(\pi_{n} \circ \psi)(a)$ depends only on $a_{1},....a_{n}$, indeed we have the explicit formula
\[
(\pi_{n} \circ \psi)(a)=\sum_{m=1}^{n}(-1)^{n}\sum_{(k_{1},...,k_{m}) \in P(k,m)}a_{k_{1}}a_{k_{2}}...a_{k_{m}},
\]
where the sum is over the set $P(k,m) \subset (\mathbb{N}_{\geq 1})^{m}$ of multi-indices $(k_{1},...,k_{m})$ with sum equal to $n$. The continuity of $\pi_{n} \circ \psi$ then follows from the continuity of the following operations: the tensor product, the canonical isomorphisms $V^{\otimes k}\otimes V^{\otimes l}\cong V^{\otimes k+l}$, the canonical inclusions of $V^{\otimes n}$ into $T((V))$, and of addition in $T((V))$. The continuity of group multiplication follows by a similar argument.
\end{proof}

\subsection{Compact sets in the product topology}

 To work with Theorem \ref{thm: universal approximation} it is important to have an understanding of the compact subsets of $\mathcal{C}_{p}$ in the chosen topology. To describe particular examples of such sets in the case $p=1$, we recall from Theorem \ref{HL} that each equivalence class $[x]$ contains a tree-reduced representative of minimal length. This tree-reduced path is unique up to reparametisation; we will make frequent use of the its constant-speed parameterisation and so we adopt a notation to refer to it.
 \begin{notation}
 We use $x^{*}$ to denote the tree-reduced representative of the equivalence class $[x]$, which is parameterised at constant speed.     
 \end{notation}
We note that $x^{*}$ will be Lipschitz continuous with Lipschitz norm bounded by the tree-reduced length of $x$; see Remark \ref{Lipschitz reparam}. We now define the subset $B(r) \subset \mathcal{C}_{1}$ to be the set of unparameterised paths for which the tree-reduced length is bounded by the prescribed potive real number $r$.
That these sets are compact subspaces of $(\mathcal{C}_{p},\chi_{_{\text{pr}}})$ is the context of the next result.
 \begin{proposition}\label{prop: B(r)}
Let $p$ be in $[1,2)$ and $r>0$. Then the set $B\left( r\right)$
is a compact subspace of $(\mathcal{C}_{p}, \chi_{\text{pr}
})$.  
\end{proposition}

\begin{proof}
We have seen in Lemma \ref{lem: sep and met} that the product topology is metrisable and it therefore suffices to prove that $B(r)$ is sequentially compact. Let $\left( \left[ x_{n}\right] \right) _{n=1}^{\infty}$ 
be a sequence in $B\left( r\right) $ then the sequence $\left(
x_{n}^{\ast}\right) _{n=1}^{\infty}$ of tree-reduced representatives satisfies%
\begin{equation}
\left\vert x_{n}^{\ast}\left( t\right) -x_{n}^{\ast}\left(
s\right) \right\vert \leq r\left( t-s\right) \text{ for all }s\leq t\text{
in }\left[ a,b\right] \text{ and }n\in%
\mathbb{N},
  \label{equi}
\end{equation}
and this collection of paths is therefore equicontinuous. The Arzela-Ascoli theorem gives that there exists a uniformly convergent subsequence of $\left(
x_{n}^{\ast}\right) _{n=1}^{\infty}$ which we again call $\left(
x_{n}^{\ast}\right) _{n=1}^{\infty}$ for convenience. 

We denote this uniform limit by $x$
 and observe, by using (\ref{equi}), that $\left\vert \left\vert
x\right\vert \right\vert _{1}\leq r$ so that $\left\vert
\left\vert x^{\ast}\right\vert \right\vert _{1}\leq r$, and therefore $\left[
x\right] $ is in $B\left( r\right) .$ Let $1<q<2,$ then a standard 
inequality interpolating between $1-$ and $q-$ variation norms gives that
\[
\left\vert \left\vert x_{n}^{\ast}-x_{m}^{\ast}\right\vert
\right\vert _{q}\leq 2r ^{1/q}\left\vert \left\vert x
_{n}^{\ast}-x_{m}^{\ast}\right\vert \right\vert _{\infty}^{1-1/q}.
\]
This allows us to deduce that conclude that $\left( x_{n}^{\ast}\right) _{n=1}^{\infty }
$ is a Cauchy sequence and thus convergent to $x$ in $q$-variation. For
each $m$ the map%
\[
C_{q}\ni\sigma\mapsto S_{m}\left( \sigma\right) \in V^{\otimes m} 
\]
is continuous. It thus holds that $S_{m}\left( x_{n}^{\ast
}\right) \rightarrow S_{m}\left( x\right) $ as $%
n\rightarrow\infty$ and hence $\left[ x_{n}\right] \rightarrow%
\left[ x\right] $ in $B\left( r\right) $ as $n\rightarrow \infty$ in
the product topology. 
\end{proof}

\begin{corollary}
The space $(\mathcal{C}_{1}, \chi_{\text{pr}})$ is $\sigma$-compact.
\end{corollary}
\begin{proof}
This is immediate from the previous proposition and the fact that $\mathcal{C}_{1}=\cup_{r=1}^{\infty} B(r)$. The latter being the case since, in view of Theorem \ref{HL}, we know that every continuous and finite length path has a tree-reduced representative (which trivially must also have finite length).
\end{proof}
With the understanding developed through the previous results, we now revisit the setting of Example \ref{prop: ode} and investigate whether the family of controlled differential equations, presented there as functions on $\mathcal{C}_{1}$, are continuous on the sets $B(r)$.
\begin{example}[Continuity of CDEs as functions on unparameterised paths]
Suppose that $W$ is a finite-dimensional vector space. Let $f:V\rightarrow 
\mathcal{T}\left( W\right) $ be a linear map which satisfies the conditions described in Example \ref{prop: ode}. There we considered the map $x \mapsto y_{b}$, showing that
\[
\left\vert y_{b} -y^{N}_{b} \right\vert \leq\frac{%
C^{N+1}L_{x}^{N+1}}{\left( N+1\right) !},
\]
where 
\[
y^{N}_{b}
:=y_{a}+\sum_{k=1}^{N}f_{S\left( x\right)^{(k)} }^{\left( k\right)
}I\left( y_{a}\right). 
\]
To prove continuity on $B\left(
r\right) $ we take a sequence $\left[ x_{n}\right] \rightarrow\left[
x \right] $ in $B\left( r\right) $, and then making the dependence of $y$ on $[x]$ explicit by writing $y([x])$, we 
note that for any $N\geq1$ the above estimates easily yield that
\[
\left\vert y_{b}([x]) -y_{b}([x_{n}]) \right\vert
\leq \left\vert\sum_{k=1}^{N} f_{S\left( x^{\ast}\right)^{(k)} }^{\left(
k\right) }I\left( y_{0}\right) -f_{S\left( x_{n}^{\ast}\right)^{(k)}
}^{\left( k\right) }I\left( y_{0}\right) \right\vert +\frac{2C^{N+1}r^{N+1}}{%
\left( N+1\right) !}. 
\]
By letting \thinspace$n\rightarrow\infty$ and then $N\rightarrow\infty$ in this estimate we establish the convergence of $y([x_{n}])$ to $y([x])$ as $n \to \infty$.
\end{example}
The following lemma provides a class of examples of tree-reduced paths.
\begin{lemma}[\cite{cass2022topologies}]\label{tree-reduced}
For $v_1,v_2,...,v_m \in V$, let $x_{v_i}$ denote the linear path in $C_{1}$ whose derivative $x'_{v_i} \equiv v$ on $[0,1]$. Let $x:\left[ 0,1\right] \rightarrow V$ in $
C_{1}$ be the piecewise linear path obtained by concatenating these linear pieces $\gamma=\gamma_{v_{1}}*\gamma_{v_{2}}*\dots*\gamma_{v_{m}}$. Assume further that for every consecutive pair of vectors $\left\langle v_{i},v_{i+1} \right\rangle=0$ for $i=1,\dots,m-1$. Then $\gamma$ is tree reduced.
\end{lemma}
\begin{exercise}
Find a proof of this lemma.    
\end{exercise}
The following construction can be found in \cite{lyons2018inverting}. It illustrates how the class of paths considered in the previous lemma can be used to construct a sequence of pairs of distinct paths whose signatures agree up to some level; the details are worked through in Exercise \ref{ex:concatenation_paths}. The example can be used to exhibit a feature of the product topology.
\begin{example} [\cite{lyons2018inverting}] \label{LyonsXu}
Let $W$ be a two-dimensional vector subspace of $V$ which is identified with $\mathbb{R}^{2}$ through an orthonormal basis $\{v_{i}:i=1,2\}$. We define two sequences $\{\rho_{n}:n=1,2,\dots\}$ and $\{\sigma_{n}:n=1,2,\dots\}$ of so-called axis paths; that is, piecewise linear paths which always move parallel to one of the coordinate axes. These are defined by $\rho_{1}:=\gamma_{v_{1}}*\gamma_{v_{2}}$,  $\sigma_{1}:=\gamma_{v_{2}}*\gamma_{v_{1}}$ and then for $n=2,3,\dots$ by 
\[
\rho_{n}=\rho_{n-1}*\sigma_{n-1} \text{ and }
\sigma_{n}=\sigma_{n-1}*\rho_{n-1}.
\]
See Figure \ref{fig: fig2} below for visualisations of the first two pairs of paths in this sequence. If two consecutive line segments $v_i$ and $v_{i+1}$ are positively co-linear, then by replacing $\gamma_{v_{i}}*\gamma_{v_{i+1}}$ with $\gamma_{v_i+v_{i+1}}$, we may write each path in the form required for Lemma \ref{tree-reduced}. By this same Lemma, each of these paths is tree-reduced. For every $n$, the paths $\rho_{n}$ and $\sigma_{n}$ have length $2^{n}$ and it was further shown in \cite{lyons2018inverting} that the terms in their STs coincide up to degree $n$, i.e.
\[
S(\rho_{n})^{(k)}=S(\sigma_{n})^{(k)} \text{ for $k=1,2,\dots,n$ }.
\]
Moreover, it can be shown that $S_{n+1}(\rho_{n})$ and $S_{n+1}(\sigma_{n})$ are not equal for any $n$. Consequently, each tree-reduced path $\Gamma_n:= \sigma_n*\overleftarrow{\rho}_{n}$ has length $2^{n+1}$ and satisfies
\[
S\left( \Gamma_{n}\right)^{(m)}
=S\left(\sigma_n*\overleftarrow{\rho}_{n}\right)^{(m)}=0 \text{ for  $m=1,2,\dots,n$, while } S\left( \Gamma_{n}\right) \neq\mathbf{1},
\]
\begin{figure}
    \centering
    \includegraphics[height=0.25\paperheight,keepaspectratio]{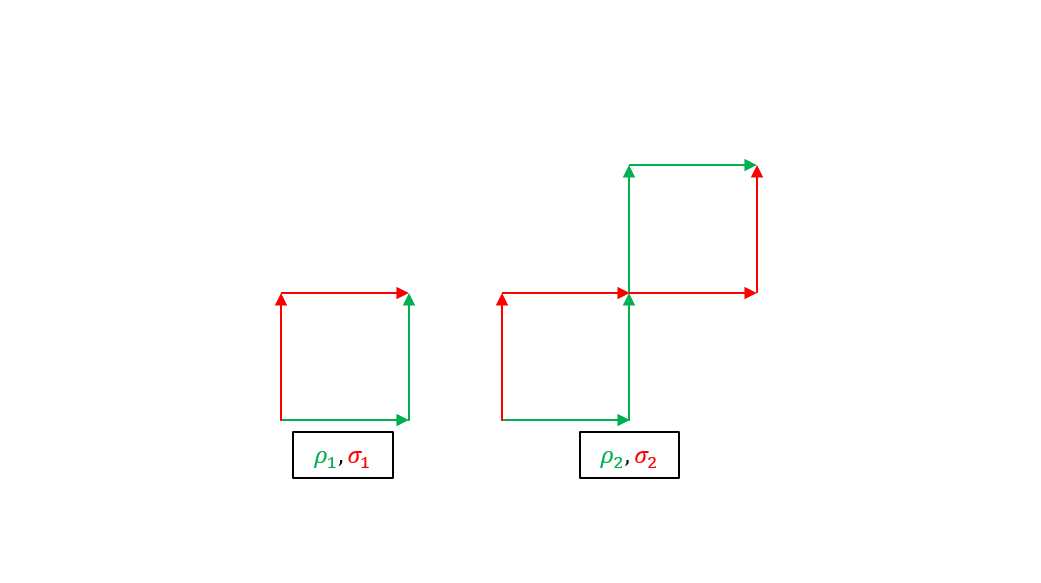}
    \caption{ Distinct planar axis paths having the same signatures when truncated at levels one and two respectively (left to right). } \label{fig: fig2} 
\end{figure}
\end{example}
We can use this construction to prove the following property of the product topology, which we will later use.

\begin{proposition}\label{prop: unbounded nhds}
The function $d([x])=||x^{*}||$ is unbounded on every non-empty open set in $(\mathcal{C}_{1},\chi_{\text{pr}})$
\end{proposition}
\begin{proof}
Consider the two-dimensional subspace of $V$ spanned by two orthonormal vectors $\{v_1,v_2\}$, and let $(\Gamma_n)_{n=1}^\infty$ be the sequence of axis paths defined in the example just considered, Example \ref{LyonsXu}. Each $\Gamma_n$ is tree reduced by Lemma \ref{tree-reduced} so that $d([\Gamma_n])=2^{n+1}$. On the other hand, because for every $n \in \mathbb{N}$ we have that $\lim_{m \to \infty} S(\Gamma_{n})^{(m)}=0$ the sequence $(\Gamma_n)_{n=1}^\infty$ converges to $[o]$ in the product topology. Every open neighbourhood of $[o]$ therefore contains all but finitely many terms of the sequence $(\Gamma_n)_{n=1}^\infty$ and $d$ is unbounded on this set. 
This shows that every open neighbourhood of $[o]$ is unbounded in $d$. The same then holds for any open neighbourhood of an arbitrary $[x]\in\mathcal{C}_1$ by using the fact, which follows from Proposition \ref{prop: group cts}, that the multiplication map $m_{[x]}:\mathcal{C}_1 \to \mathcal{C}_1$, $[y]\mapsto[x]*[y]$ is a homeomorphism. 
\end{proof}

\subsection{Probability on unparameterised paths}
A natural next step is to develop tools for studying probability measures on the measurable space $(\mathcal{C}_{p},\mathcal{B}(\mathcal{C}_{p}))$ which is obtained from the Borel $\sigma$-algebra of the topology $\tau_{\text{pr}}$. We first gather some fundamental facts from elementary topology. Recall that a Polish space is a separable topological space that is completely metrisable, while a Lusin space is topological space that is homeomorphic to closed subset of a compact metric space. Every Polish space is a Lusin space, and every subspace of Lusin space is a Lusin space if and only if it is a Borel set. We have the following.
\begin{lemma}
The space $(\mathcal{C}_{1},\tau_{\text{pr}})$ is a Lusin space.  
\end{lemma}
\begin{proof}
Any Borel subset of a Polish space is a Lusin space, and so noting that the space $X=\prod_{i=0}^{\infty} V^{\otimes i}$ is a Polish space, as we have already seen in the proof of Lemma \ref{lem: sep and met}, we aim to show that $\mathcal{S}$ is a Borel subset of $X$. The inclusion map $\iota:\mathcal{S}\to X$ is continuous and hence the sets $B(r)$ from Proposition \ref{prop: B(r)} are also compact subsets of $X$. They are therefore closed in $X$ and hence also Borel subsets of $X$. Thus $\mathcal{S}$ being a countable union of Borel subsets is itself a Borel subset. 
\end{proof}

Assuming (at least) that $X$ is a Lusin space, we will use $\mathcal{P}(X)$ to denote the space of probability measures on $(X,\mathcal{B}(X))$. A widely-used topology on $\mathcal{P}(X)$ is the topology of weak convergence, and we recall it definition here. As preparation we introduce $C_{b}(X)$ denote the set of continuous, bounded real-valued function on $X$ and use the notation $\mu(f)$ to denote the integral 
\[
\mu(f)=\int f d \mu.
\]
\begin{definition}[Topology of weak convergence]
Suppose that $\nu$ is in $\mathcal{P}(X)$. We define a topology on $\mathcal{P}(X)$ by letting a neighbourhood base of the point $\nu$ consist of set of the form
\[
\{\mu \in \mathcal{P}(X): |\mu(f_{i})-\nu(f_{i})|<\epsilon_{i} \},
\]
where $n\in \mathbb{N}$, each $\epsilon_{i}>0$ and each $f_{i} \in C_{b}(X)$. We call this topology the topology of weak convergence on $\mathcal{P}(X)$.      
\end{definition}
It is an important fact that for every Lusin space $X$, the topology of weak convergence on $\mathcal{P}(X)$ is metrisable. This means that we can characterise the topology of weak convergence using sequences where $(\mu_{n})_{n=1}^{\infty}$ converges to $\mu$ if $\mu_{n}(f) \rightarrow \mu(f) $ as $n \rightarrow \infty$ for every $f \in C_{b}(X)$. As is usual in this case, we write $\mu_{n} \Rightarrow  \mu$ as $n \rightarrow \infty$.
The following result, known as Prohorov's Theorem, describes a sufficient condition for a subset of probability measure to be relatively compact in the the topology of weak convergence. When the underlying space $X$ is a Polish space, then this condition is also necessary.
\begin{theorem}[Prohorov's Theorem]
Suppose $X$ is a Lusin space. Then a subset $H$ of  $\mathcal{P}(X)$ is relatively compact in the topology of weak convergence if $H$ is tight in the sense that for any $\epsilon>0$ there exists compact subset $K$ of $X$ such that $\mu(K^{c})<\epsilon$ for all $\mu$ in $H$.
\end{theorem}

From the point of view of these results, it would desirable for the property of Polishness to hold in the case $X=\mathcal{C}_{p}$. This fails however for the product topology as the following result shows. In fact we show more, namely that for $p=1$ the topological space $(\mathcal{C}_{1},\tau_{\text{pr}})$ is not a \emph{Baire space}. We recall this means that any countable union of closed sets with empty interior itself has empty interior. Any (pseudo-)metrisable space is a Baire space, and in particular any Polish space is a Baire space. The proof we give here relies on the Lyons-Xu construction discussed in Example \ref{LyonsXu}, which the reader may wish to re-read before diving into the details.
\begin{theorem}\label{thm: topologies}
The topological space $(\mathcal{C}_{1},\tau_{\text{pr}})$ is not a Baire space and so not completely metrisable. In particular, it is not a Polish space.
\end{theorem}

\begin{proof}
Proposition \ref{prop: B(r)} gives us that $\mathcal{C}_1 = \bigcup_{r=1}^{\infty}B(r)$ where $B(r)$ is the compact subspace of $\chi_{\text{pr}}$. As $(\mathcal{C}_{1},\tau_{\text{pr}})$ is a Hausdorff space, these sets must also be closed. By definition, each element of $B(r)$ has tree-reduced length bounded by $r$, and, by Proposition \ref{prop: unbounded nhds}, $B(r)$ cannot therefore contain any non-trivial open subset. It follows that $\mathcal{C}_1$ is the countable union of closed sets each having empty interior and thus is not a Baire space.
\end{proof}

One might ask whether by selecting a different induced topology more desirable properties could be arrived at. The following lemma shows that this is not the case; the conclusion remains stable across a generally class of metrisable topologies induced via embedding $\mathcal{S}$ as a subspace of $T((V))$.
\begin{lemma}\label{lem: general metric}
Assume that $E \subset T((V))$ and that $\tau$ defines a metrisable topology on $E$. Suppose that there exists $p$ in $(1,2)$ such that  $\mathcal{S}_{p} \subset E$ and such that $S:C_p\to E$ is continuous. Let $\chi$ denote the induced topology on $\mathcal{C}_1$ defined as the unique topology which makes $S:\mathcal{C}_1 \rightarrow \mathcal{S}$ a homeomorphism  when $\mathcal{S}$ is given the subspace topology of $\tau$. Then $(\mathcal{C}_1, \chi)$ is $\sigma-$compact but is not a Baire space, and so is neither a Polish space nor a locally compact space.
\end{lemma}
\begin{proof}
The same arguments found in Proposition \ref{prop: B(r)} show the sets $B(r)$ are compact in $\chi$. We can show again that every $\chi$-open set must contain unparameterised paths of arbitrarily large tree-reduced length. This  again shows that they have empty interior in $\chi$ and the remaining conclusion  settled as above.
\end{proof}

\section{Learning with the signature transform}

One main motivation of Theorem \ref{thm: universal approximation} is to provide a theoretical justification for the \emph{signature method for time series regression}. We end this introductory chapter 
by exploring the relationship between this method and the foundations developed above. 
Suppose that $\Gamma$ is a random variable taking values in the space
$\left(  \mathcal{S},\chi_{\text{pr}}\right)$. Let $Y$ be another random variable, defined on the same
probability space as $\Gamma$, which takes values in a finite dimensional
vector space $W$. Then the goal of regression analysis is to learn the
conditional expectation $\mathbb{E}\left[  \left.  Y\right\vert \Gamma\right]
$, where $Y$ is interpreted as the response of some system to the input
$\Gamma.$ Another way of saying this is that we want to approximate the
Borel-measurable function $f:\mathcal{S}\rightarrow W$ defined by
\begin{equation}
f\left(  g\right)  =\mathbb{E}\left[  \left.  Y\right\vert \Gamma=g\right]
.\label{cond exp}%
\end{equation}
Suppose the law of $\Gamma$ is given by a Borel probability measure on $\mathcal{S}$ and is denoted by $\mu.$ Then, as we will see in Corollary \ref{cor: lusin}, by a version of Lusin's Theorem \cite{bogachev2007}, the function $f$ in (\ref{cond exp}) is almost continuous in the sense
that for any $\delta>0$ there exists a compact set $K=K_{\delta}$ such that
$\mu\left(  \mathcal{S}\setminus K\right)  <\delta$ and such that
$f$ is continuous on $K.$ By Theorem \ref{thm: universal approximation}, it
is then reasonable to adopt the model
\[
Y=L\left(  \Gamma\right)  +\epsilon,
\]
where $L$ is the restriction to $\mathcal{S}$ of a linear function from $T\left(
\left(  V\right)  \right)  $ to $W$ and $\epsilon$ is a $W$-valued random
variable satisfying $\mathbb{E}\left[  \left.  \epsilon\right\vert
\Gamma\right]  =0\,.$ This obviates the need to find an explicit compact set
and continuous function relating the independent and dependent variables. 

\begin{corollary}\label{cor: lusin}
Consider the setup as in Proposition \ref{prop: ode}, and let $\mu$ be a Borel probability measure on $(\mathcal{C}_{1},\chi_{_{\text{pr}}})$. Then, for every $\delta>0$, there is a compact set $K\subset \mathcal{C}_1$ such that $\mu(\mathcal{C}_1\setminus K)<\delta$ and the function $[x] \mapsto y_{b}$ is continuous on $K$.
\end{corollary}

\begin{proof}
Since $\mathcal{C}_1$ is $\sigma-$compact with respect to $\chi_{_{\text{pr}}}$, $\mu$ is tight. Separability of $(\mathcal{C}_{1},\chi_{_{\text{pr}}})$ and \cite[Proposition 7.2.2]{bogachev2007} then imply $\mu$ is a Radon measure. Measurability of $\Psi_{y_0,f}$ and Lusin's Theorem for finite Radon measures \cite[Theorem 7.1.13]{bogachev2007} shows the existence of compact sets with the desired property.
\end{proof}

Assume for simplicity that $W=\mathbb{R}$, and consider a sample $\{(\gamma_1,y_1),...,(\gamma_N,y_N)\}$ of input-output pairs such that $\gamma_i : [0,1] \to \mathbb{R}^d$ is a continuous path of bounded variation obtained by piecewise linear interpolation of a time series $\mathbf{x}_i$ and $y_i = f(\gamma_i)$, for $i\in\{1,...,N\}$. 

Recall that for any level of truncation $m \in \mathbb N \cup \{0\}$, the space $T^m(\mathbb{R}^d) = \prod_{j=0}^m(\mathbb{R}^d)^{\otimes j}$ and its algebraic dual $T^m(\mathbb{R}^d)^*$ are both isomorphic to $\mathbb{R}^M$ where $M=1 + d + d^2 + ... + d^m = d^{m+1}/(m+1)$. Therefore, we can identify TSTs truncated at level $m$ and linear functionals acting on them as vectors in $\mathbb{R}^M$. 

The ST method for time series regression consists of minimising the mean squared error
\begin{equation}\label{eqn:MSE}
    \min_{\omega \in \mathbb{R}^M} \frac{1}{N}\sum_{i=1}^N (\omega^\top S^m(\gamma_i) - y_i)^2 
\end{equation}
The solutions $\omega^*$ to the least squares problem (\ref{eqn:MSE}) are classicaly characterised by the normal equations $GG^\top \omega^* = Gy$, where $G$ is the $M \times N$ matrix whose $i^{th}$ column is given by $S^m(\gamma_i)$, for $i\in\{1,...,N\}$, and $y=(y_1,...,y_N) \in \mathbb{R}^N$.

\begin{remark}
    In the overparameterised regime where $M>N$, the $M \times M$ matrix $GG^\top$ is never invertible, therefore if a solution exists it is not unique. In the underparameterised regime where $M\leq N$, if $GG^\top$ is invertible the solution $\omega^* = G(G^\top G)^{-1}y$ is unique.
\end{remark}

\begin{remark}
    Note that in the multivariate case where $W = \mathbb{R}^n$ for $n > 1$, one can instead consider the simple variation of (\ref{eqn:MSE}) that minimises the $L^2$ error on $W$ over $n \times M$ matrices
    \begin{equation}\label{eqn:ols_multivariate}
        \min_{\Omega \in \mathbb{R}^{n \times M}} \frac{1}{N}\sum_{i=1}^N \norm{\Omega S^m(\gamma_i)  - y_i}_2^2 
    \end{equation}
    which can be solved in an analogous way as the univariate case.
\end{remark}

\begin{remark}
    To promote simpler solutions it is typical to add a reguraliser to the mean squared error (\ref{eqn:MSE}), for example in the form of Tikhonov regularisation
    \begin{equation*}
        \min_{\omega \in \mathbb{R}^M} \frac{1}{N}\sum_{i=1}^N (\omega^\top S^m(\gamma_i)  - y_i)^2 + \lambda \norm{\omega}_2^2
    \end{equation*}
    for some tuning parameter $\lambda > 0$, which admits a unique solution $\omega^* = (GG^\top + \lambda^2 I_M)^{-1}Gy$. 
\end{remark}

\begin{remark}
    The issue of overfitting  arises when the model is too complex given the amount of available data, and is typical in the overparameterised regime. One standard way to mitigate the effect of overfitting is to use a $L^1$ penalty on the weights $\omega$ instead of penalising using the $L^2$ norm.  This leads to the LASSO regularisation
    \begin{equation}\label{eqn:lasso}
        \min_{\omega \in \mathbb{R}^M} \frac{1}{N}\sum_{i=1}^N (\omega^\top S^m(\gamma_i)  - y_i)^2 + \lambda \norm{\omega}_1^2
    \end{equation}
    Indeed, it can be shown that  (\ref{eqn:lasso}) induces a sparse solution with a few nonzero entries (exercise).
\end{remark}

\begin{remark}
If the labels $\{y_i\}_{i=1}^N$ are categorical then the inference task becomes a classification task, which can be addressed by a similar regression approach by simply predicting a vector of likelihoods of belonging to a given class. For example, in a binary classification problem with labels $y_1,...,y_N \in \{-1,+1\}$ one can consider a logistic loss and minimise the negative log-likelihood
\begin{equation*}
    \min_{\omega \in \mathbb{R}^M} \frac{1}{N}\sum_{i=1}^N \log_2(1+\exp(-y_i\omega^\top S^m(\gamma_i) ))
\end{equation*}
\end{remark}

\begin{remark}
    We note that the ST has been used in machine learning in a variety of contexts dealing with time series, well beyond simple linear regression. For example, \cite{bonnier2019deep} embed the ST as an autodifferentiable non-linearity within path-preserving layers of a deep neural network.
\end{remark}

%% file: ex_chapter1.tex
\section{Exercises}

\begin{exercise}
Show directly that the right-hand side of (\ref{hphi}) is independent of the
basis chosen.
\end{exercise}

\begin{exercise}
Show that the Lie bracket satisfies: (anti-symmetry) $\left[  g,h\right]
=-[h,g]$ and (Jacobi's identity) $\left[  g,\left[  h,k\right]  \right]
+\left[  k,\left[  g,h\right]  \right]  +\left[  h,\left[  k,g\right]
\right]  =0.$
\end{exercise}

\begin{exercise}
Show that any two vector spaces $S,T$ which satisfy the property in \Cref{tensor_product} must be
isomorphic, and so that \Cref{tensor_product} characterises $S\otimes T$ up to
isomorphism. 
\end{exercise}

\begin{exercise}
Assume that $\{s_{a}:a \in A \}$ and $\{t_{b}:b \in B \}$ are bases for $S$ and $T$ respectively. Show that $\{s_{a}\otimes t_{b} : a\in A,b\in B\}$ is a basis for $S\otimes T$.
\end{exercise}

\begin{exercise}
Let $(V,\langle \cdot, \cdot \rangle_V)$ and $(W,\langle \cdot, \cdot \rangle_W)$ be two inner product spaces. Show that  $\left\langle a,b\right\rangle _{V\otimes W}$ is well defined,
independent of the representations of $a$ and $b,$and check that it is an
inner-product. Prove that the norm $\left\vert \left\vert a\right\vert
\right\vert_{V\otimes W} :=\left\langle a,a\right\rangle_{V\otimes W}^{1/2}$ which is derived from
$\left\langle \cdot,\cdot\right\rangle _{V\otimes W}$ is a cross norm on
$V\otimes W$, i.e.
\begin{equation}
    \left\vert \left\vert vw\right\vert \right\vert_{V\otimes W} =\left\vert \left\vert
    v\right\vert \right\vert_{V}\left\vert \left\vert w\right\vert \right\vert
    _{W}\text{ for all }v\in V\text{ and }w\in W.\label{cross norm}%
    \end{equation}
\end{exercise}

\begin{exercise}
\label{ex:Lip} In the setting of Remark \ref{Lipschitz reparam}, show that
$x^{\tau}\equiv x\circ\tau$ is Lipschtiz continuous with
\begin{equation}
\left\vert x^{\tau}\left(  t\right)  -x^{\tau}\left(  s\right)  \right\vert
\leq\left\vert t-s\right\vert \text{ for all }s\text{ and }t\text{ in }\left[
0,L\right]  .\label{Lip norm 1}%
\end{equation}
\end{exercise}

\begin{exercise}
Recall the notation $\mathcal{S}_{p}\subset T((V))$ to denote the image of $C_{p}(V)$ under the ST. Prove the strict inclusion $\mathcal{S}_{p}\subset \mathcal{S}_{q}$ for any $p<q$.
\end{exercise}

\begin{exercise}
    Prove that $\shuffle$ is an associative and commutative product on $T(V)$
\end{exercise}

\begin{exercise}\label{ex:truncated_sig_cde}
    Let $C_{1,0,t}$ be the set of paths defined in \Cref{ex:univ_time}. Recall that, for any $N \in \mathbb N \cup \{0\}$, the maps $\pi_{\leq N}: T((\mathbb R^d)) \to T^N(\mathbb R^d)$ and $\mathfrak{i}_{\leq N} : T^N(\mathbb R^d) \hookrightarrow T((\mathbb R^d))$ denote the canonical projection and inclusion respectively. Show that for any path $x \in C_{1,0,t}$, the truncated signature $S^N(x)_{0,t} = \pi_{\leq N} \circ S(x)_{0,t}$ is the unique solution to the linear CDE
    \begin{equation*}
        dy_t = f_N(y_t)dx_t, \quad \text{started at } y_0 = (1,0,...,0) \in T^N(\mathbb R^d)
    \end{equation*}
    where the linear map $f_N \in \mathcal{L}(T^N(\mathbb R^d), \mathcal{L}\left( \mathbb R^d, T^N(\mathbb R^d) \right))$ is given by
    \begin{equation*}
        f_N \left( y \right)a = \pi_{\leq N}(\mathfrak{i}_{\leq N}(y) \cdot \mathfrak{i}_1(a)).
    \end{equation*}
\end{exercise}

\begin{exercise}\label{ex:prob_sig}(\cite{cass2023signature})
Let $x:[0,1] \to \mathbb{R}^d$ be a continuously differentiable path with unit speed parameterisation, i.e. $|\dot x_t| = 1$ for all $t \in [0,1]$, and such that $\dot x$ is
    Lipschitz continuous 
    \begin{equation*}
    |\dot x_t - \dot x_s| \leq c |t-s|, \quad \forall s,t \in [0,1], s<t,
    \end{equation*}
    where $c \in (0,1)$. 
    \begin{enumerate}[(i)]
        \item  Let $0\leq U_{\left( 1\right) }\leq U_{\left( 2\right) }\leq ....\leq U_{\left( n\right) }\leq 1$ be a sample of $n$ independent, ordered, uniform random variables on $[0,1]$. Prove that the $n^{\text{th}}$ level of the ST 
        \begin{equation*}
        S^n(x)_{[0,1]}=\frac{1}{n!}\mathbb{E}
        \left[\dot x_{U_{(1)}} \otimes ...\otimes \dot x_{U_{(n)}}\right]
        \end{equation*}
        \item Deduce that 
        \begin{equation*}
            n!\norm{S^n(x)_{[0,1]}} =  \mathbb{E}\left[\prod_{k=1}^{n}\left\langle \dot x_{U(k)}, \dot x_{V(k)} \right\rangle \right]^{1/2}
        \end{equation*}
        where $\left( V_{\left(
        i\right) }\right) _{i=1}^{n}$ is another independent samples of order statistics from the uniform distribution on $[0,1]$, also independent of $\left(U_{\left(
        i\right)}\right) _{i=1}^{n}$. 
        \item Show that
        \begin{equation*}
            \left( n!\norm{S^n(x)_{[0,1]}}\right)^2 \geq \mathbb{E}\left[ \prod_{i=1}^{n} \left(1-c|U_{(i)} - V_{(i)}|\right)\right].
        \end{equation*}
        \item It can be shown that for all $\epsilon >0$ and $n$ 
        \begin{equation*}
        \mathbb{P}\left( \max_{i=1...,n}|U_{\left( i\right) }-V_{\left( i\right)
        }|>\epsilon \right) \leq 2\exp \left( -\frac{n\epsilon ^{2}}{2}\right) .
        \end{equation*}%
        Use this fact to show that 
        \begin{equation*}
        \underset{n\rightarrow \infty }{\lim \inf }\left( n!\left\vert \left\vert
        S^n\left( x\right) _{[0,1]}\right\vert \right\vert \right) ^{1/n}\geq 1
        \end{equation*}
        \item Explain why it can be deduced that 
        \begin{equation*}
        \lim_{n\rightarrow \infty }\left( n!\left\vert \left\vert S\left( x
        \right) _{0,1}^{n}\right\vert \right\vert \right) ^{1/n}=1.
        \end{equation*}
    \end{enumerate}
\end{exercise}

\begin{exercise}\label{ex:concatenation_paths}(\cite{lyons2018inverting})
Let $v \in \mathbb{R}^2$ be a fixed vector and denote by $X_v$ the straight line $X_v(t) = tv$ for $t \in [0,1]$. Consider an axis path $X$ defined as the concatenation of $n$ straight lines
\begin{equation*}
    X = X_{\lambda_1 v_1} * ... * X_{\lambda_n v_n} : [0,n] \to \mathbb{R}^2
\end{equation*}
where $v_1,...,v_n \in \{e_1,e_2\}$ are one of the two standard basis vectors in $\mathbb{R}^2$ and $\lambda_1,...,\lambda_n \in \mathbb{R}$ are scalars. For any multi-index $(i_1,...,i_n) \in \{1,2\}^n$, the basis element $e_{i_1} \otimes ... \otimes e_{i_n}$ is called square-free if $i_k \neq i_{k+1}$ for all $k=1,...,n-1$. 
    \begin{enumerate}[(i)]
        \item Suppose that $v_1=e_{i_1}, ..., v_n = e_{i_n}$. Define the following quantity
        \begin{equation*}
            C_X(i_1,...,i_n) := \left(e_{i_1}^*\otimes ... \otimes e_{i_n}^*, S(X)_{[0,n]}\right)
        \end{equation*}
        Show that $C_X(i_1,...,i_n)=\lambda_1...\lambda_n$.
        \item Assume that $e_{j_1}\otimes ... \otimes e_{j_m}$ is the unique, longest, square-free basis element such that $C_X(j_1,...,j_m) \neq 0$. Show that $m=n$.
        \item Show that
        \begin{equation*}
            \lambda_k = \frac{2C_X(i_1,...,i_{k-1},i_k,i_k,i_{k+1},...,i_n)}{C_X(i_1,...,i_n)}, \quad \text{for } k=1,...,n.
        \end{equation*}
        \item Define two sequences of paths as follows: $Y^1 = (t,0), Z^1 = (0,t)$ and for any $n \geq 1$
        \begin{equation*}
            Y^{n+1} = Y^n * Z^n, \quad Z^{n+1} = Z^n * Y^n
        \end{equation*}
        Show that $Y^{n+1}$ and $Z^{n+1}$ are distinct axis paths.
        \item Prove that the ST of $Y^n$ and $Z^n$ coincide up to level $n$. 
    \end{enumerate}
\end{exercise}

\begin{exercise}[\textbf{Expected ST of Brownian motion}]\label{ex:expected_sig}
Let $B_{t}=\left(
B_{t}^{1},...,B_{t}^{d}\right) _{t\in \left[ 0,1\right] }$ be a standard $d-$%
dimensional Brownian motion. Suppose that 
\begin{equation*}
S_{N}\left( B\right) _{0,1}=(1,...,\int_{0<t_{1}<....<t_{n}<1}\circ
dB_{t_{1}}\otimes ....\otimes \circ
dB_{t_{n}},...\int_{0<t_{1}<....<t_{N}<1}\circ dB_{t_{1}}\otimes ....\otimes
\circ dB_{t_{N}})
\end{equation*}

is the step-$N$ TST of Brownian motion using Stratonovich
integration (denoted by $\circ $). Remember the relationship between Ito and
Stratonovich integration:%
\begin{equation*}
\int M\circ dN=\int MdN+\frac{1}{2}\left[ M,N\right] ,
\end{equation*}%
for continuous semimartingales $M$ and $N,$ where $\left[ M,N\right] $ is
their quadratic covariation.

\begin{enumerate}[(i)]
\item Using the fact that $S_{N}\left( B\right) _{0,1/2}\otimes
S_{N}\left( B\right) _{1/2,1}=$ $S_{N}\left( B\right) _{0,1}$ prove that the
expected ST 
\begin{equation*}
\mathbb{E}\left[ S_{N}\left( B\right) _{0,1}\right] =\mathbb{E}\left[
S_{N}\left( B\right) _{0,1/2}\right] \otimes \mathbb{E}\left[ S_{N}\left(
B\right) _{1/2,1}\right] .
\end{equation*}

Writing 
\begin{equation*}
\mathbb{E}\left[ S_{N}\left( B\right) _{0,1}\right] =\left(
1,A_{1},A_{2},...,A_{N}\right) ,
\end{equation*}%
where $A_{k}\in \left(\mathbb{R}^{d}\right) ^{\otimes k}$ for all $k,$ show that 
\begin{equation*}
\mathbb{E}\left[ S_{N}\left( B\right) _{0,1/2}\right] =\mathbb{E}\left[
S_{N}\left( B\right) _{1/2,1}\right] =\left( 1,\frac{1}{\sqrt{2}}A_{1},\frac{%
1}{2}A_{2},...,\frac{1}{2^{N/2}}A_{N}\right)
\end{equation*}

\item Prove that $A_{k}=0$ for all odd $k$. (Hint: use the fact that $B$
and $-B$ have the same distribution).

\item Use part (a) to prove the recurrence relation 
\begin{equation*}
\left( 2^{n-1}-1\right) A_{2n}=\frac{1}{2}\sum_{k=1}^{n-1}A_{2k}\otimes
A_{2n-2k}, 
\end{equation*}

for $n\leq \left\lfloor N/2\right\rfloor $ and hence prove that 
\begin{equation*}
A_{2n}=\frac{1}{n!}A_{2}^{\otimes n}.
\end{equation*}

\item Use your answers to deduce a formula for $\mathbb{E}\left[
S_{N}\left( B\right) _{0,1}\right] .$
\end{enumerate}

The expected signature of other Gaussian process can be computed as well; see \cite{cass2024wiener} for details. See also \Cref{ex:expected_sig_continued} in the next chapter for a continuation of this exercise. 
\end{exercise}

\begin{exercise}
    Let $X$ be a $m$-dimensional It\^o diffusion on $[0,1]$ with linear drift and diffusion functions started at $X_0 = x$ and driven by $n$-dimensional Brownian motion, i.e. with coordinates satisfying the linear It\^o SDE
    \begin{equation*}
        dX^i_t = a_i^\top X_tdt + \sum_{j=1}^n b_{i,j}^\top X_tdW^j_t
    \end{equation*}
    with vectors $a_i, b_{i,j} \in \R^m$ for any $i=1,...,m$ and $j=1,...,n$.

    Assume that the function $\Phi : [0,1] \times \R^m \to T((\R^m))$ defined as
    \begin{equation*}
        \Phi(t,x) := \mathbb E[S(X)_{0,t} | X_0 = x]
    \end{equation*}
    is well-defined on $[0,1] \times \R^m$ and for any multi-index $I = (i_1,...,i_k) \in \{1,...,m\}^k$ the function $\pi_I \Phi \in C^{1,2}([0,1]\times \R^m)$, where $\pi_I$ is the canonical projection of $T((\R^m))$ onto the linear span of the basis element $e_{i_1}\otimes...\otimes e_{i_k}$.
    \begin{enumerate}[(i)]
        \item Show that for any $t\in [0,1]$ and $\tau \in [0,t]$ the process
        \begin{equation*}
            M_\tau := E[S(X)_{0,t}| \mathcal{F}_\tau]
        \end{equation*}
        is a martingale, where $\mathcal{F}_\tau = \sigma(X_u, u \leq \tau)$ denotes the natural filtration generated by the process $X$ up to time $\tau$. Furthermore, show that 
        \begin{equation*}
            M_\tau = S(X)_{0,\tau} \otimes \Phi(t-\tau,X_\tau)
        \end{equation*}
        where the ST $S(X)$ is  defined as the solution of the Stratonovich SDE
        \begin{equation*}
            dS(X)_{0,\tau} = S(X)_{0,\tau}\otimes \circ dX_\tau.
        \end{equation*}
        \item Show that the process $M_\tau$ satisfies the following It\^o SDE
        \begin{align*}
            dM_\tau &= S(X)_{0,\tau} \otimes d\Phi(t-\tau,X_\tau) + dS(X)_{0,\tau} \otimes \Phi(t-\tau,X_\tau) \\
            &+ \sum_{i,j=1}^d S(X)_{0,\tau}\otimes e_i \otimes \frac{\partial \Phi(t-\tau, X_\tau)}{\partial x_j} d\langle X^{(i)}, X^{(j)}\rangle_\tau 
        \end{align*}
        where $\langle X^{(i)}, X^{(j)}\rangle_\tau$ denotes the quaratic variation of $X^{(i)}$ and $X^{(j)}$.
        \item Show that the drift term of $M_\tau$ can be written as
        \begin{equation*}
            S(X)_{0,\tau}\otimes\mathcal{R}(\tau,X_\tau)
        \end{equation*}
        and justify why it is equal to $0$ for any $\tau \in [0,t]$.
        \item Conclude that $\Phi(t,x)$ satisfies the following PDE on $T((\R^m))$
        \begin{align*}
            &\Big(-\frac{\partial}{\partial t} + \mathcal{L}\Big)\Phi(t,x) + \sum_{i=1}^m\Big(\sum_{j=1}^m b_{j,i}^\top xe_j\Big)\otimes\frac{\partial\Phi(t,x)}{\partial x_i} \\
            &+\Big(\sum_{i=1}^m a_i^\top x e_i + \frac{1}{2}\sum_{j,k=1}^m b_{j,k}^\top e_j\otimes e_k\Big)\otimes\Phi(t,x) = 0
        \end{align*}
        with initial condition $\Phi(0,x) = e$ and additional condition $\pi_e(\Phi(t,x)) = 1$ for any $(t,x) \in [0,1] \times \R^m$, and where $\cL$ is the infinitesimal generator of $X$. 
    \end{enumerate}
\end{exercise}

\begin{exercise}\label{ex:half-shuffle}
    Define the half-shuffle product $\prec:T(V^*)\times T(V^*)\mapsto T(V^*)$ by 
\begin{align*}
f\prec r = rf \text{ and } r \prec f & = 0 \text{ for any $r\in \mathbb{R}$ and $f \in V$} \\
&\text{and then extend it inductively by } \nonumber \\
    f \prec g & =  a\cdot(f_- \prec g + g \prec f_-) 
\end{align*}
for any $f \in V^{\otimes k} $ and $g \in V^{\otimes l}$ of the form $f=a\cdot f_{-}$, where $a \in V$. Given this definition, $\prec$ extends uniquely to an algebra product on $T(V^*)\times T(V^*)$ by linearity. 

Define the product $\text{area} : T(V^*) \times T(V^*) \to T(V^*)$ as follows
\begin{equation*}
    \text{area}(f, g) = f \prec g - g \prec f.
\end{equation*}

Let $f,g,h \in T(V^*)$ such that $\langle f,e\rangle = \langle g, e\rangle = \langle h, e\rangle = 0$. Show that the following three identities hold:
\begin{enumerate}[(i)]
    \item $f \shuffle g = f \prec g + g \prec f;$
    \item $f \prec (g \shuffle h) = (f \prec g) \prec h;$
    \item $f \prec (g \prec h) = (f \prec g) \prec h + (g \prec f) \prec h.$
    \item For any $i,j \in \{1,...,d\}$, $i \neq j$ and $n\geq 1$, define the following elements of $T(V^*)$
    \begin{align*}
            h_{i,j}^n := (...((e_i^* \prec \underbrace{e_j^*) \prec e_j^*)\prec ... \prec e_j^*}_{n \text{ times}}) \qquad a_{i,j}^n := \text{area}(...\text{area}(\text{area}(e_i^*, \underbrace{e_j^*), e_j^*),...,e_j^*}_{n \text{ times}}).
    \end{align*}
    Show that the following identity holds
    \begin{equation*}
        h_{i,j}^n  = \frac{1}{n+1} \sum_{k=0}^n(e_j^*)^{\shuffle k} \shuffle a_{i,j}^{n-k}.
    \end{equation*}
\end{enumerate}
\end{exercise}

\begin{exercise}[\textbf{Inverse ST via Legendre polynomials}]\label{ex:legendre}
    Let $x : [0,1] \to \mathbb{R}$ be a continuous path such that $x(0) = 0$ and regular enough so that the following series of functions converges pointwise 
    \begin{equation*}
        x(t) = \sum_{n=0}^\infty \alpha^x_n P_n(t) \quad \text{where} \quad \alpha^x_n = (2n+1)\int_0^1x(t)P_n(t)dt
    \end{equation*}
    and $P_n$ is the $n^{th}$ shifted Legendre polynomial. 

    Let $\hat x: [0,1] \to \mathbb{R}^2$ be the time-augmented path $\hat x(t) = (t, x(t))$. Show that for any $n \in \mathbb{N} \cup \{0\}$ there exists a unique $\ell_n \in T((\mathbb{R}^2))^*$ such that 
    $$\alpha_n^x = \langle \ell_n, S(\hat x)_{0,1}\rangle $$
    and that the sequence  $\{\ell_n\}_{n \geq 0}$of functionals satisfies the following recursion
    \begin{equation*}
        \ell_n = \frac{2n+1}{n}\left( - \frac{n-1}{2n-3}\ell_{n-2} + \ell_{n-1}  -2(e_1^* \prec \ell_{n-1})\right)
    \end{equation*}
    with $\ell_0 = e_2^*e_1^*$ and $\ell_1 = 3 \ e_2^*e_1^* - 6 e_1^*e_2^*e_1^*  - 6 \ e_2^*e_1^*e_1^*$ and where $\prec$ is the half-shuffle product defined in \Cref{ex:half-shuffle}.
    
    [Hint: you may use the fact that shifted Legendre polynomials satisfy the following recursive relation on $[0,1]$
    \begin{equation*}
        (n+1)P_{n+1}(t) = (2n+1)(1-2t)P_n(t) - nP_{n-1}(t)
    \end{equation*}
    with $P_0(t) = 1$ and $P_1(t) = 1-2t$.]
\end{exercise}

\begin{exercise}
    Repeat the same analysis as in \Cref{ex:legendre} replacing Legendre polynomials by Chebyshev polynomials.
\end{exercise}

\begin{exercise}
    Suppose $x \in C_p([a,b],V)$ is time-augmented in the sense that there exists a basis for $V$ such that in one of the coordinates $x$ is the path $[a,b] \ni t \mapsto t$. Prove that $x$ must be a tree-reduced path.
\end{exercise}

\begin{exercise}\label{ex:univ_time}
    Let $C_{1,0,t} \subset C_1$ be the subset of $\mathbb R^d$-valued continuous bounded variation paths over $[0,1]$, started at $0$, and with one coordinate-path corresponding to time, without loss of generality, the first, i.e. $x^{(1)}_t = t$. Recall that, for any $N \in \mathbb N \cup \{0\}$, the map $\pi_{\leq N}: T((\mathbb R^d)) \to T^N(\mathbb R^d)$ denotes the canonical projection. For any $f \in T^N(\mathbb R^d)^{*}$, define the function $\Phi_f \in \mathbb{R}^{C_{1,0,t}}$ as $\Phi_{f}: x \mapsto \left( f,\pi_{\leq N} \circ S\left( x\right) \right)$.
    \begin{enumerate}
        \item Show that the class $\mathcal{A=}\left\{\Phi_{f}:f\in T^N(\mathbb R^d)^{*}, N \in \mathbb N \cup \{0\}\right\}$ 
    \begin{enumerate}
        \item [(i)] is a subalgebra of $\mathbb{R}^{C_{1,0,t}}$,
        \item [(ii)] contains the constant functions,
        \item [(iii)] separates points in $C_{1,0,t}$.
    \end{enumerate}
    \item Let $K$ be a compact subset of $\mathbb{R}^{C_{1,0,t}}$ with respect to the $1$-variation topology. Show that $\mathcal{A}_{K}:=\{\Phi_{f}|_{K}: f\in T^N(\mathbb R^d)^{*}, N \in \mathbb N \cup \{0\}\}$ is a dense subset of $C(K)$ with the topology of uniform convergence.
    \end{enumerate}
\end{exercise}

\begin{exercise}
    Prove \Cref{thm: topologies} when $p \in (1,2)$.
\end{exercise}

\begin{exercise}\label{exercise:code}
Consider a continuous path of bounded variation $\gamma : [0,1] \to \mathbb R^2$ 
The position $p_t$ at time $t \in [0,1]$ of a point on the surface of a ball of unit radius started at $p_0 \in \mathbb R^3$ is given by the solution of the linear CDE
\begin{equation}\label{eqn:ball}
    dp_t = A(d\gamma_t)p_t, \quad p_0 = (1,0,0)^\top,
\end{equation}
where $A \in \mathcal{L}(\mathbb R^2,\mathbb R^{3 \times 3})$ is defined as follows
\begin{equation*}
    A(x) = x_1 \begin{pmatrix}
        0 & 0 & 1\\
        0 & 0 & 0\\
        -1 & 0 & 0
    \end{pmatrix} 
    +   
    x_2 \begin{pmatrix}
        0 & 0 & 0\\
        0 & 0 & 1\\
        0 & -1 & 0
    \end{pmatrix}.
\end{equation*}
Set $\delta = 10^{-3}$ and consider the following uniform partition of $[0,1]$
\begin{equation*}
    \mathcal{D} = \{t_k = k\delta : k=0,...,\floor{1/\delta}\}.
\end{equation*}
\begin{enumerate}
    \item Write a python function \texttt{generate\_BM\_paths(N, $\delta$, $\rho$)} which generates $N \in \mathbb N$ piecewise linear approximations of a two-dimensional Brownian motion over the grid $\mathcal{D}$ and with correlation parameter $\rho \in (0,1)$.
    \item Write the $n^{th}$ Picard iterate and provide an approximation to the solution of the CDE (\ref{eqn:ball}) using the signature of $\gamma$ truncated at level $n \in \mathbb N$ and quantify the local error over $[s,t] \subset [0,1]$ in terms of $\norm{\gamma}_{p,[s,t]}$ and $n$.
    \item Write a python function \texttt{CDE\_Picard\_solve($\gamma$, A, n)} to solve equation (\ref{eqn:ball}) numerically using the Picard iteration from the previous question.
    \item Using $70\%, 20\%, 10\%$ splitting ratios, split into training, test and evaluation the dataset of $N$ input-output pairs $\{(\gamma^1,p^1),...,(\gamma^N,p^N)\}$ where $p^i = p^i_1 \in \mathbb R^3$ is the approximate solution at $t=1$ of the CDE (\ref{eqn:ball}) driven by the path $\gamma^i$, for $i\in\{1,...,N\}$ obtained using \texttt{CDE\_Picard\_solve($\gamma$, A, n)} with $n=15$.
    \item Using any publicly available signature python package, run a LASSO regression from \href{https://scikit-learn.org/stable/modules/generated/sklearn.linear_model.Lasso.html}{scikit-learn}, with penalty parameter $\lambda$, from the signature of $\gamma^i$ truncated at level $n=0,1,...,8$ to $p^i$ by performing a grid-search over the hyperparameters $n,\lambda$ (or any other you might need) using only the training and test datasets. 
    \item Using the calibrated signature linear regression model, report mean-squared-error between the predicted $p^i$ and the true $p^i$ on the validation set.
\end{enumerate}














\end{exercise}

%% file: chapter2.tex
\chapter{Signature kernels}\label{sec:signature_kernels}

Kernel methods form a well-established class of algorithms that constitute the essential building blocks of several machine learning models such as Support Vector Machines (SVMs) and Gaussian Processes (GPs) in Bayesian inference. These models have been successfully employed in a wide range of applications including bioinformatics \cite{scholkopf2004kernel, leslie2001spectrum}, genetics \cite{noble2006support}, natural language processing \cite{lodhi2002text} and speech recognition \cite{cuturi2007kernel}.

The central idea of kernel methods is to transform input data points in a typically low dimensional space to a higher (possibly infinite) dimensional one by means of a nonlinear function which is called a feature map. In many cases the higher dimensional space will be a Hilbert space, allowing a kernel to be defined on pairs of input points by taking the inner product of the images of these two points under the feature map. 

The advantage of this approach can be seen in typical non-linear, infinite-dimensional regression or classification tasks. These can sometimes be formulated as optimisation problem expressed only in terms of kernel evaluations at pairs of points in a training set, as we will illustrate in Section \ref{sec:representer_theorem}. In many situations, the kernel can be efficiently evaluated with no reference to the feature map, a property commonly referred to as kernel trick. This property allows one to benefit from the advantages of working in a higher dimensional feature space without the associated drawbacks.

The selection of an effective kernel will usually be task-dependent problem, and this challenge is exacerbated when the data are sequential. This chapter builds on the fundamentals presented in the previous chapter to lay out the mathematical foundations of signature kernels, a class of kernels tailored for tasks that involve sequential data, and  which have received attention in recent years  \cite{kiraly2019kernels, salvi2021signature, cass2021general, salvi2021higher, cochrane2021sk, lemercier2021siggpde, toth2020bayesian, horvath2023optimal, manten2024signature, pannier2024path}.

\section{From signatures to signature kernels}

\subsection{Weighted inner products on tensor algebras}

We recall that for any $k \in \N$, an inner product $\langle \cdot, \cdot \rangle_V$ on $V$ yields a canonical Hilbert-Schmidt inner product $\langle \cdot, \cdot \rangle_{V^{\otimes k}}$ on $V^{\otimes k}$ defined as
\begin{equation*}
	\langle v,w \rangle_{V^{\otimes k}} = \prod_{i=1}^k \langle v_i,w_i \rangle_V.
\end{equation*}
for any $v=(v_1,...,v_k), w = (w_1,...,w_k)$ in $V^{\otimes k}$. 

As a first step before introducing the definition of signature kernels, we observe how the linearity of the tensor algebra $T(V) = \bigoplus_{k=0}^\infty V^{\otimes k}$ allows us to construct a family of weighted inner products.  

\begin{definition}\label{defintion: weighted spaces}
    Given a weight function $\phi : \N \cup \{0\} \to \R_+$,  define for any the $\phi$-inner product $\langle \cdot, \cdot \rangle_\phi$ on the tensor algebra $T(V)$ as 
    \begin{equation}\label{eqn:weighted_inner_product}
        \langle v, w \rangle_\phi := \sum_{k=0}^\infty \phi(k) \langle v_k, w_k \rangle_{V^{\otimes k}}
    \end{equation}
    for any $v=(v_0,v_1,...), w = (w_0,w_1,...,)$ in $T(V)$.
    We will denote by $T_\phi((V))$ the Hilbert space obtained by completing $T(V)$ with respect to $\langle \cdot, \cdot \rangle_\phi$. 
\end{definition}

We will make use of a topology on these Hilbert spaces and, unless otherwise stated, we will assume that we work with the norm topology. 

Recall that $\mathcal{S}$ denotes the image of $\mathcal{C}$ by the ST. The following result present a condition on the weighting function $\phi$ which ensures that $\mathcal{S} \subset T_\phi((V))$.

\begin{lemma}\label{lemma:condition_phi}
If $\phi : \mathbb{N} \cup \{0\} \to \mathbb{R}_+$ is so that for any $C>0$ the series $\sum_{k\geq 0}\frac{C^k\phi(k)}{(k!)^2}$
converges, then $\mathcal{S} \subset T_\phi((V))$.
\end{lemma}

\begin{proof}
    For any path $x \in \wC$, the factorial decay of \Cref{prop:factorial_decay} yields 
    \begin{align*}
        \big|\big|S(x)_{a,b}\big|\big|^2_\phi = \sum_{k=0}^\infty \phi(k) \big|\big|S(x)_{a,b}^{(k)}\big|\big|^2_{V^{\otimes k}}  \leq \sum_{k=0}^\infty \phi(k) \frac{\norm{x}_{1,[a,b]}^{2k}}{(k!)^2}
    \end{align*}
    The summability condition guarantees that the series is finite.
\end{proof}

\subsection{Weighted signature kernels}

We first define a weighted signature kernels as $\phi$-inner products of a pair of STs.

\begin{definition}\label{def:phi_sigker}
    Let $\phi : \mathbb{N} \cup \{0\} \to \mathbb{R}_+$ be a weight function satisfying the condition of \Cref{lemma:condition_phi}. Then, the $\phi$-signature kernel $k_\phi : \mathcal{C} \times \mathcal{C} \to \R$ is defined for 
    \begin{equation*}
        k_\phi([x],[y]) = \left\langle S(x), S(y) \right\rangle_\phi
    \end{equation*}
    for any two unparameterised paths $[x],[y] \in \mathcal{C}$.
\end{definition}

\begin{remark}
    The previous definition of $\phi$-signature kernel can be easily extended to the setting where the weight function $\phi : \mathbb{N} \cup \{0\} \to \mathbb{R}$ is real-valued. This is done by noting that for such $\phi$, the bilinear form $\langle \cdot, \cdot \rangle_{\phi}$ of equation (\ref{eqn:weighted_inner_product}) extends to a continuous bilinear form on $T_{|\phi|}(V)$. If $\phi$ is such that $|\phi|$ satisfies the condition of \Cref{lemma:condition_phi}, then it is easy to see that the $\phi$-signature kernel 
    \begin{equation*}
        k^{x,y}_\phi(s,t) = \left\langle S(x)_{a,s}, S(y)_{a,t} \right\rangle_\phi
    \end{equation*}
    is well defined.
\end{remark} 

To simplify the notation, we will omit reference to equivalence classes and
write $k_\phi\left(  x,y \right)  $ instead of
 $k_\phi\left(  \left[  x \right]  ,\left[
y \right]  \right)$. Similarly, in computations it is usually desirable to define signature kernels directly on representative paths rather than equivalence classes. \Cref{def:phi_sigker} can be immediately extended to two continuous paths of bounded variation $x\in C_1([a,b],V)$ and $y\in C_1([c,d],V)$ 
\begin{equation*}
        k_\phi(x,y)_{s,t} = \left\langle S(x)_{a,s}, S(y)_{c,t} \right\rangle_\phi
\end{equation*}
for any $s \in [a,b]$ and $t \in [c,d]$.

The next result provides an integral equation relating the $\phi$-signature kernel to the $\phi_+$-signature kernel, obtained by shifting the weight function $\phi_+(k)=\phi(k+1)$ for all $k \in \N$. This result appeared first in \cite{cass2021general}. 

\begin{lemma}\label{lemma:sigkern_int}
Let $x \in C_1([a,b],V)$ and $y \in C_1([c,d],V)$ and assume that $\phi : \mathbb{N} \cup \{0\} \to \mathbb{R}$ is such that $|\phi|$ and $|\phi_+|$ satisfy the condition of \Cref{lemma:condition_phi}. Then the following two-parameter integral equation holds
\begin{equation}\label{eqn:kernel_shift_relation}
    k_\phi(x,y)_{s,t} = \phi(0) + \int_a^s\int_c^t k_{\phi_+}(x,y)_{u,v}\langle dx_u,dy_v \rangle_V.
\end{equation}
\end{lemma}

\begin{proof}
    The fact that the $\phi$- and $\phi_+$-signature kernels are well defined follows from the summability condition of \Cref{lemma:condition_phi}. To show the stated relation we observe that
    \begin{align*}
        k^{x,y}_\phi(s,t) &= \sum_{i=0}^\infty \phi(i) \left\langle S(x)_{a,s}^{(i)}, S(y)_{c,t}^{(i)}\right\rangle_{V^{\otimes i}} \\
        &= \phi(0) + \sum_{i=1}^\infty \phi(i) \left\langle \int_a^s S(x)_{a,u}^{(i-1)}dx_u, \int_c^t S(y)_{c,v}^{(i-1)}dy_v\right\rangle_{V^{\otimes i}}\\
        &= \phi(0) + \int_a^s \int_c^t \sum_{i=0}^\infty \phi(i+1) \left\langle S(x)_{a,u}^{(i)}dx_u,  S(y)_{c,v}^{(i)}dy_v\right\rangle_{V^{\otimes i}}\\
        &= \phi(0) + \int_a^s \int_c^t \sum_{i=0}^\infty \phi(i+1) \left\langle S(x)_{a,u}^{(i)},  S(y)_{c,v}^{(i)}\right\rangle_{V^{\otimes i}}\left\langle dx_u, dy_v\right\rangle_{V}\\
        &= \phi(0) + \int_a^s \int_c^t k_{\phi_+}(x,y)_{u,v} \left\langle dx_u, dy_v\right\rangle_{V}
    \end{align*}
    where the first and last equalities follow from the definitions of $k_\phi$ and $k_{|\phi|}$ respectively, the second equality follows from the same arguments as in the proof of \Cref{sigCDE}, in the third equality integral and sums can be exchanged thanks to the factorial decay of the terms in the ST, the fourth equality follows from a change of variable, and finally the fifth holds because for any $A,B \in V^{\otimes i-1}$ and $a,b \in V$ one has $$\langle Aa, Bb \rangle_{V^{\otimes i}} = \langle A, B \rangle_{V^{\otimes i-1}}\langle a, b \rangle_{V}.$$
\end{proof}

The special case where $\phi = \phi(0)$ is constant was treated in \cite{salvi2021signature}. In this case, one has that $\phi_+ = \phi$, hence equation (\ref{eqn:kernel_shift_relation}) reduces to the following integral equation
\begin{equation}\label{eqn:int_sigker_original}
    k_\phi(x,y)_{s,t} = \phi(0) + \int_a^s\int_c^t k_\phi(x,y)_{u,v}\langle dx_u,dy_v \rangle_V.
\end{equation}
Moreover, if $x,y$ are differentiable and $\phi \equiv 1$ one can differentiate both sides of (\ref{eqn:int_sigker_original})  with respect to $s$ and $t$ and obtain the following linear hyperbolic PDE
\begin{equation}\label{eqn:kernel_PDE}
    \frac{\partial^2 k_\phi(x,y)_{s,t}}{\partial s \partial t} = k_\phi(x,y)_{s,t} \langle \dot x_u,\dot y_v \rangle_V.
\end{equation}
with boundary conditions $k^{x,y}(0,t) = k^{x,y}(s,0) = 1$ for all $s,t \in [a,b]$.

\begin{remark}
    The PDE (\ref{eqn:kernel_PDE}) belongs to a class of hyperbolic PDEs introduced in \cite{goursat1916course} and known as Goursat problems. The existence and uniqueness of solutions of the PDE (\ref{eqn:kernel_PDE}) follow from \cite[Theorems 2 \& 4]{lees1960goursat}. The integral equation also extends beyond the case of differentiable paths to geometric rough paths, see again \cite{salvi2021signature} and \cite{cass2023fubini}.
\end{remark}

\begin{remark}
It is clear from \Cref{eqn:kernel_shift_relation} that for general weight functions $\phi$, the $\phi$-signature kernel $k_\phi(x,y)$ will not solve a PDE of the type (\ref{eqn:kernel_PDE}). Nonetheless, later in the chapter we will provide several examples of weight functions for which an efficient evaluation is still possible. 
\end{remark}

\subsection{Reproducing signature kernel Hilbert spaces}

In functional analysis, a reproducing kernel Hilbert space (RKHS) is abstractly defined as a Hilbert space of functions where all evaluation functionals are continuous. In practice, an RKHS is uniquely associated with a positive semidefinite kernel that "reproduces" every function in the RKHS, in the  sense of equation (\ref{eqn:reproducing_property}) below. 

\begin{definition}\label{def:rkhs}
    A Hilbert space $\mathcal{H}$ of functions defined on a set $\mathcal{X}$ is called reproducing kernel Hilbert space (RKHS) over $\mathcal{X}$ if, for each $x \in \mathcal{X}$, the point evaluation functional at $x$, $f \mapsto f(x)$,  is a continuous linear functional, i.e. there exists a constant $C_x \geq 0$ such that
    $$|f(x)| \leq C_x \norm{f}_\mathcal{H}, \quad \text{for all } f \in \mathcal{H}.$$
\end{definition}

Given a RKHS $\mathcal{H}$, the Riesz representation theorem (e.g. \cite[Thm. 5.25]{folland1999real}) implies that there exists a unique functional $k(x,\cdot)$ in $\mathcal{H}$ such that 
\begin{equation}\label{eqn:reproducing_property}
    \left\langle k(x,\cdot), f\right\rangle_\mathcal{H} = f(x), \quad \text{for all } f \in \mathcal{H} \text{ and } x \in \cX.
\end{equation}
This yields a symmetric kernel $k:\mathcal{X} \times \mathcal{X} \to \R$  defined for any $x,y \in \mathcal{X}$ as
\[
k(x,y)=\left\langle k(x,\cdot), k(y,\cdot)\right\rangle_\mathcal{H}.
\]
Note that $\left\langle k(x,\cdot), k(y,\cdot)\right\rangle_\mathcal{H}=k(x,\cdot)(y)=k(y,\cdot)(x)$ so that the element $k(x,\cdot)$ of $\mathcal{H}$ really is the same as the function $k(x,\cdot)$ obtained from the slices of the kernel function. 

By virtue of (\ref{eqn:reproducing_property}), the kernel $k$ is said to have the reproducing property.

We next define kernels with the additional property of being positive semidefinite.
\begin{definition}
    A kernel $k : \mathcal{X} \times \mathcal{X} \to \mathbb{R}$ is called positive semidefinite if for any $n \in \mathbb{N}$ and points $x_1,...,x_n \in \mathcal{X}$, the Gram matrix $K:=(k(x_i,x_j))_{i,j}$ is positive semidefinite, i.e. if for any $c_1,...,c_n \in \mathbb{R}$
    \begin{equation*}
        \sum_{i=1}^n\sum_{j=1}^nc_ic_jk(x_i,x_j) \geq 0.
    \end{equation*}
\end{definition}

In particular, this property holds for signature kernels as we show in the next lemma.
\begin{lemma}
    Let $\phi : \mathbb{N} \cup \{0\} \to \mathbb{R}_{+}$ be such that $\phi$ satisfies the condition of \Cref{lemma:condition_phi}. Then the $\phi$-signature kernel is symmetric and positive semidefinite.
\end{lemma}

\begin{proof}
    The $\phi$-signature kernel $k_\phi$ is clearly symmetric. Considering an arbitrary collection of unparameterised paths $x_1,...,x_n \in \wC$ and scalars $c_1,...,c_n \in \mathbb{R}$
    \begin{align*}
        \sum_{i,j=1}^n c_ic_k k_\phi(x_i,x_j) &= \sum_{i,j=1}^n c_ic_k \langle S(x_i),S(x_j)\rangle_\phi \\
        &= \Big\langle\sum_{i=1}^n c_i S(x_i), \sum_{j=1}^n c_j S(x_j) \Big\rangle_\phi \\
        &= \Big|\Big|\sum_{i=1}^n c_i S(x_i)\Big|\Big|^2_{\phi} \geq 0
    \end{align*}
    yields the claimed positive semidefiniteness.
\end{proof}

\begin{remark}
    Through a similar calculation one can see that any symmetric kernel that satisfies the reproducing property (\ref{eqn:reproducing_property}) is positive semidefinite. Indeed, consider $n$ points $x_1,...,x_n \in \cX$ and scalars $c_1,...,c_n \in \mathbb{R}$ and observe that
    \begin{align*}
        \sum_{i,j=1}^n c_ic_k k(x_i,x_j) &= \sum_{i,j=1}^n c_ic_k \langle k(x_i,\cdot),k(x_j,\cdot)\rangle_\mathcal{H} \\
        &= \Big\langle\sum_{i=1}^n c_i k(x_i,\cdot), \sum_{j=1}^n c_j k(x_j,\cdot) \Big\rangle_\mathcal{H} \\
        &= \Big|\Big|\sum_{i=1}^n c_i k(x_i,\cdot)\Big|\Big|^2_\mathcal{H} \geq 0
    \end{align*}
\end{remark}

A natural question to ask is whether, for a given kernel $k$, there exists a RKHS of functions such that the kernel $k$ has the reproducing property (\ref{eqn:reproducing_property}). A positive answer is provided by the Moore-Aronszajn theorem.
\begin{theorem}[\cite{aronszajn1950theory}]
    If $k : \mathcal{X} \times \mathcal{X} \to \mathbb R$ is a positive semidefinite kernel, then there exists a unique RKHS $\mathcal{H}$ such that $k$ has the reproducing property  (\ref{eqn:reproducing_property}).
\end{theorem}

This result is classical, therefore we do not provide a full proof here. However, it is still informative to outline how  the construction of a RKHS associated to a kernel can be specialised to the case of signature kernels. 
Let $\phi : \mathbb{N} \cup \{0\} \to \mathbb{R}$ be a weight function such that $|\phi|$ satisfies the condition of \Cref{lemma:condition_phi}. Consider the following linear space of real-valued functions on unparameterised paths 
\begin{equation}\label{eqn:pre_HS}
    \mathcal{H}^0_\phi := \text{Span}\{k_\phi(x,\cdot) : x \in \wC\}.
\end{equation} 
For functions $f = \sum_{i=1}^m\alpha_ik_\phi(x_i,\cdot)$ and $g = \sum_{j=1}^n\beta_jk_\phi(y_j,\cdot)$ in $\mathcal{H}^0_\phi$ the  expression
\begin{equation}\label{eqn:inner_product_H0}
    \langle f,g \rangle_{\mathcal{H}^0_\phi} = \sum_{i=1}^m\sum_{j=1}^n\alpha_i\beta_jk_\phi(x_i,y_j)
\end{equation} 
defines an inner product on $\mathcal{H}^0_\phi$. \Cref{eqn:inner_product_H0} immediately yields the reproducing property on $\mathcal{H}^0_\phi$, i.e.
$$\langle k_\phi(x,\cdot),f\rangle_{\mathcal{H}^0_\phi} = f(x), \quad \text{for all } f \in \mathcal{H}_\phi^0 \text{ and } x \in \wC.$$ 
By the Cauchy-Schwarz inequality, for any $x \in \wC$
\begin{equation}\label{eqn:boundness_eval_H0}
    |f(x)| \leq \sqrt{k_\phi(x,x)} \norm{f}_{\mathcal{H}^0_\phi}.
\end{equation}
Thus, every the evaluation functional $f \mapsto f(x)$ is continuous on $\mathcal{H}^0_\phi$. However, $\mathcal{H}^0_\phi$ is not yet a Hilbert space because it might fail to be complete. This can be remedied by completing $\mathcal{H}^0_\phi$. We recall that an abstract completion is constructed by considering equivalence classes of Cauchy sequences in $\mathcal{H}^0_\phi$. For any such sequence  $(f_{n})_{n \in \mathbb{N}}$  the limit $\lim_{n \to \infty} \norm{f_{n}}_{H^0_\phi}$ exists; two sequences $(f_{n})_{n \in \mathbb{N}}$ and $(g_{n})_{n \in \mathbb{N}}$ are then called equivalent if
$ \norm{f_{n}-g_{n}}_{H^0_\phi} \rightarrow 0$ as $n \rightarrow \infty $. Property (\ref{eqn:boundness_eval_H0}) however gives a concrete way of realising this space; the key observation is that if  $(f_{n})_{n \in \mathbb{N}}$ and $(g_{n})_{n \in \mathbb{N}}$ are equivalent Cauchy sequences in $\mathcal{H}_\phi^0$ then they have a common pointwise limit:
\[
f(x)=\lim_{n \rightarrow \infty}f_{n}(x)=\lim_{n \rightarrow \infty }g_{n}(x) = g(x),
\]
which follows from (\ref{eqn:boundness_eval_H0}). As such we can associate each equivalence class $\left[(f_{n})_{n \in \mathbb{N}}\right]$ with the pointwise limit function $f$. Conversely, every function $f$ that arises as a pointwise limit in this way can be associated with a unique equivalence class $\left[(f_{n})_{n \in \mathbb{N}}\right]$ of Cauchy sequences; this is a consequence of the following lemma.
\begin{lemma}
Suppose that $(f_{n})_{n \in \mathbb{N}}\ $ and $(g_{n})_{n \in \mathbb{N}}\ $ \ are two
Cauchy sequences in $\mathcal{H}_{\phi}^{0}.$ Then $f=g$ if and only if
$\lim_{n\rightarrow\infty}\left\vert \left\vert f_{n}-g_{n}\right\vert
\right\vert _{\mathcal{H}_{\phi}^{0}}=0.$
\end{lemma}
From this discussion we see that we can identify the completion $\mathcal{H}_\phi$ with the set of functions on $\wC$ which are pointwise limits of a Cauchy sequence in $\mathcal{H}^0_\phi$: 
\begin{equation*}
    \mathcal{H}_\phi := \{f \in \R^{\wC} : \text{ $\exists$ a Cauchy seq. } (f_n)_{n \in \N} \text{ in } \mathcal{H}^0_\phi \text{ so that } f = \lim_{n\to\infty}f_n \text{ pointwise}\} 
\end{equation*}
The inner product (\ref{eqn:inner_product_H0}) can then be extended to the following inner product on $\mathcal{H}_\phi$ 
\begin{equation*}
    \langle f, g\rangle_{\mathcal{H}_\phi} = \lim_{n\to\infty} \langle f_n, g_n\rangle_{\mathcal{H}_\phi^0}
\end{equation*}
where $(f_n)_{n \in \N}$ and $(g_n)_{n \in \N}$ are Cauchy sequences in $\mathcal{H}^0_\phi$ such that $f = \lim_{n\to\infty}f_n$ and $g = \lim_{n\to\infty}g_n$. We can  note that $\mathcal{H}_\phi^0$ is dense in $\mathcal{H}_\phi$ because
\begin{equation*}
    \norm{f-f_n}_{\mathcal{H}_\phi} = \lim_{k \to \infty}\norm{f_{n+k}-f_n}_{\mathcal{H}_\phi^0} \leq \sup_{k \in \N} \norm{f_{n+k}-f_n}_{\mathcal{H}_\phi^0} \underset{n \to \infty}{\longrightarrow} 0,
\end{equation*}
which allows to conclude that $\mathcal{H}_\phi$ is complete and is therefore a Hilbert space. It is easily checked that (\ref{eqn:boundness_eval_H0}) holds also for any  $f \in \mathcal{H}_\phi$ so that $\mathcal{H}_\phi$ is a RKHS. 

We study next in greater detail the RKHS associated to  signature kernels $k_{\phi}$. Where possible  we suppress reference to intervals and write e.g. $S(\gamma)$ instead of $S(\gamma)_{a,b}$. The general setting here will be that $E$ is a linear subspace of $T\left(  \left(
V\right)  \right)  $ which contains the range of the ST
$\mathcal{S}$ and on which an inner-product $\left\langle \cdot,\cdot
\right\rangle _{E}$ is defined. Then, we know from the general theory above that there exists a unique RKHS $\left(  \mathcal{H}%
_{E},\left\langle \cdot,\cdot\right\rangle _{\mathcal{H}_{E}}\right)$ associated to the kernel
$k_{E}:\mathcal{C}\times\mathcal{C}\rightarrow%
\mathbb{R}
$ given by
\[
k_{E}\left(x,y\right)
=\left\langle S\left(  x \right)  ,S\left(  y \right)  \right\rangle
_{E}.
\]\ This RKHS has the following properties:
\begin{enumerate}
\item If the function $k_E(x,\cdot) \in \mathbb{R}^{\mathcal{C}}$ are defined by
\[
k_E(x,\cdot): y  \mapsto k_E(x,y) = \left\langle
S\left(  x \right)  ,S\left(  y \right)  \right\rangle _{E},
\]
then the linear span of $\mathcal{H}_{E}^0 := \left\{  k_E(x,\cdot) : x  \in\mathcal{C}\right\}$ is a dense subspace of $\mathcal{H}_{E}$.

\item The reproducing property holds, i.e. for every $x$ and
$y$ in $\mathcal{C}$ 
$$\left\langle k_E(x,\cdot),k_E(y,\cdot) \right\rangle _{\mathcal{H}_{E}}=\left\langle S\left(  x \right)  ,S\left(
y \right)  \right\rangle _{E}.$$
\end{enumerate}

In the setting of weighted signature kernels we are interested in carrying out this construction when $E=T_{\phi}\left(  (V)\right)$ for weight functions $\phi$ which satisfy the conditions of Lemma \ref{lemma:condition_phi}.  We write the induced RKHS as \ $\mathcal{H}_{\phi
}$. The following lemma provides an explicit realisation of elements of this space as functions in $\mathbb{R}%
^{\mathcal{C}}$.

\begin{proposition}\label{prop:RKHS}
For every element $ f \in \mathcal{H}_{\phi}$ there is a unique $\ell_f$ in
$T_{\phi}\left(  \left(  V\right)  \right)  $ such that 
$f\left( x\right)  =\left\langle \ell_f, S\left(  x\right)
\right\rangle _{\phi}$. Furthermore, the map $\ell_f \mapsto f$ is an isomorphism between the
Hilbert spaces $T_{\phi}\left(  \left(  V\right)  \right)  $ and
$\mathcal{H}_{\phi}.$
\end{proposition}

\begin{remark}
If $\ell_f \mapsto \ell_f^*$
denotes the usual isomorphism between $T_{\phi}\left(  \left(  V\right)
\right)  $ and its dual $T_{\phi}\left(  \left(  V\right)  \right)  ^{\ast}$,
then $f(x) = \ell_f^* \left(  S\left(  x\right)
\right)$ for every $x$ in $\mathcal{C}$.
\end{remark}

\begin{proof}
Let $W=\overline{\text{span}\left(  \mathcal{S}\right)  }$ where  the
closure is taken in $T_{\phi}\left(  \left(  V\right)  \right)  $ and let
$W^{\perp}$ be  the orthogonal complement of $W$ in $T_{\phi}\left(  \left(
V\right)  \right)  $. We first show that $W^{\perp}=\left\{  0\right\}  $
is the trivial subspace so that $W=T_{\phi}((V))$. To do so, we take $u=\sum_{k=0}^{\infty}u_{k}$ in
$W^{\perp}$ and then for any $x$ in $\wC$ observe that the function
\begin{equation}
f_{x,u}\left(  \lambda\right)  =\left\langle u,S\left(  \lambda x\right)
\right\rangle _{\phi}=\sum_{k=0}^{\infty}\phi\left(  k\right)  \lambda
^{k}\left\langle u_{k},S^{\left(  k\right)  }\left(  x\right)  \right\rangle
_{V^{\otimes k}}\label{power series}%
\end{equation}
vanishes identically on $%
\mathbb{R}
.$ The conditions on $\phi$ together with the factorial decay of the ST
terms, Lemma \ref{fac decay} , ensures that the power series in (\ref{power series}) can by
differentiated term-by-term to give
\[
\left.  \frac{\partial^{n}f_{x,u}}{\partial\lambda^{n}}\left(  \lambda\right)
\right\vert _{\lambda=0}=\phi\left(  n\right)  n!\left\langle u_{n},S^{\left(
n\right)  }\left(  \gamma\right)  \right\rangle _{V^{\otimes n}}=0,
\]
for any $n$ in $\mathbb{N}$. It follows that $\sum_{k=0}^{n}u_{k}=u^{\left(  n\right)  }\perp$ span$\left(
\mathcal{S}_{N}\right)  $ w.r.t. the inner-product $\left\langle \cdot
,\cdot\right\rangle _{n}=\sum_{k=0}^{n}\left\langle \cdot,\cdot,\right\rangle
_{V^{\otimes k}}$ for every \thinspace$n$. From Lemma \ref{Driver} we conclude
that $u^{\left(  n\right)  }=0$ for every \thinspace$n$ and hence that $u=0.$

Finally, the map $\Psi: \ell_f \mapsto f$ is an isometry, i.e.
$$\left\vert \left\vert
\Psi\left(\ell_f\right)  \right\vert \right\vert _{\mathcal{H}_{\phi}}=\left\vert
\left\vert \ell_f \right\vert \right\vert _{T_{\phi}\left(  \left(  V\right)
\right)}.$$ 
therefore it extends uniquely to an isomorphism, from which the assertion follows.
\end{proof}

\section{Universal and characteristic signature kernels}

As we have learnt from the preceding discussion, the choice of kernel will dictate the resulting RKHS; in Example \ref{hyp} below we will describe an explicit function which belongs to $\mathcal{H}_{\phi}$ for certain choices of $\phi$ but not others. In other situations we will be interested to understand how well an RKHS can approximate a function in a given class. One important such class is the family of continuous functions on compact sets, and this lead to the following notion of a universal kernel. 

\subsection{Universality of a kernel}

\begin{definition} \label{ccuniv}(Universality) Let $k:\mathcal{C}\times\mathcal{C}\rightarrow\mathbb{R}$ is a continuous,
symmetric, positive definite kernel on $\mathcal{C}$ equipped with a topology. Then we say that $k$ is
universal if, for every compact subset $\mathcal{K\subset}$ $\mathcal{C}$,
the linear span of the set $\left\{  k(\gamma,\cdot) :\gamma\in\mathcal{K}\right\}
$ is dense in $C(\mathcal{K)}$ in the the topology of uniform convergence.
\end{definition}

This notion is sometimes called cc-universality, see \cite{sriperumbudur2010hilbert}. The following proposition gives an equivalent characterisations of this concept. We again specialise to the case of signature kernels, but the argument can be repurposed to a general setting by making appropriate modifications.

As a preliminary, we will adopt the notation $k_\phi^h := \langle h, S(\cdot) \rangle_\phi$, for any $h\in T_{\phi}\left((
V)\right)$.

\begin{proposition}\label{prop: universal equivalence}
Let $k_{\phi}$ denote a signature kernel determined by a weight function $\phi$ such that the conditions of Lemma \ref{lemma:condition_phi} hold, and let $\mathcal{H}_{\phi}$ be the associated RKHS.
Assume that $\mathcal{C}$ is equipped with a topology for which the ST $S:\mathcal{C} \rightarrow T_{\phi}\left(
(V)\right)$ is continuous. Then $k_{\phi}$ is
universal if and only if for any compact subset $\mathcal{K} \subset \mathcal{C}$, the set 
\[
\mathcal{H}_{\phi}|_{\mathcal{K}} := \left\{  k_\phi^h|_{\mathcal{K}}:h\in T_{\phi}\left((
V)\right)  \right\} 
\]
 is dense in $C\left(  \mathcal{K}\right)$ for the topology of uniform convergence, wherein
$k_\phi^h|_{\mathcal{K}}$ denotes the restriction to $\mathcal{K}$ of $k_\phi^h$.
\end{proposition}

\begin{proof}
That the condition is sufficient to ensure that $k_{\phi}$ is universal is clear from the
inclusion $\mathcal{S\subset}T_{\phi}\left((  V) \right)  .$ To show that it is
also necessary we will  prove that for any $h\in T_{\phi}\left(  (V)\right)  $
the function $k_\phi^h$ belongs to the closure of the span of $\left\{
k_\phi(\gamma, \cdot):\gamma\in\mathcal{K}\right\}  $ in the uniform topology on
$C\left(  \mathcal{K}\right)  .$ To this end, we fix $\epsilon>0$ and make use
of the compactness of $\mathcal{K}$ to first find a finite collection
$\left\{  \gamma_{i}\right\}  _{i=1}^{n}\subset\mathcal{K}$ such that
$\mathcal{K}$ is covered by the open sets
\begin{align*}
U_{i} &  =\left\{  \gamma\in\mathcal{K}:\left\vert \left\vert S\left(
\gamma_{i}\right)  -S\left(  \gamma\right)  \right\vert \right\vert _{\phi
}<\delta\left(  \epsilon\right)  \right\}  =S^{-1}\left(  B_{\left\Vert
\cdot\right\Vert _{\phi}}\left(  S\left(  \gamma_{i}\right)  ,\delta\left(
\epsilon\right)  \right)  \right)  \cap\mathcal{K},\text{ with}\\
\text{ }\delta\left(  \epsilon\right)   &  =\frac{\epsilon}{4\left\Vert
h\right\Vert _{\phi}}.
\end{align*}
We note that the openness of $U_{i}$ in $\mathcal{K}$ is guaranteed by the
continuity of the ST $S:$ $\mathcal{C}\rightarrow T_{\phi}\left(
\left(  V\right)  \right)  $ from which we also have
\[
K:=\sup_{\gamma\in\mathcal{K}}\left\Vert S\left(  \gamma\right)  \right\Vert
_{\phi}<\infty.
\]
The compactness of $\mathcal{K}$ also ensures the existence of a continuous
partition of unity $\left\{  \rho_{i}\right\}  _{i=1}^{n}$ which is
subordinate to the cover $\left\{  U\right\}  _{i=1}^{n}.$ This means that
each $\rho_{i}:\mathcal{K\rightarrow}\left[  0,1\right]  $ belongs to
$C\left(  \mathcal{K}\right)  ,$ is such that supp$\rho_{i}\subset U_{i},$ and
that  $\sum_{i=1}^{n}\rho_{i}\left(  \gamma\right)  =1$ for all $\gamma$ in
$\mathcal{K\,}$. The universality of $k_{\phi}$ guarantees that we can find
$g_{i}$ in $T_{\phi}\left((  V )\right)  $ for every $i=1,...,n$ so that
\begin{equation}
\sup_{\gamma\in\mathcal{K}}\left\vert \rho_{i}\left(  \gamma\right)
- k_\phi^{g_i}\left(  \gamma\right)  \right\vert \leq\frac{\epsilon}{2n\left\Vert
h\right\Vert _{\phi}K}.\label{est unity}%
\end{equation}
We can then define a function $f$ in the span of $\left\{  k_\phi(\gamma,\cdot)
:\gamma\in\mathcal{K}\right\}  $ by
\begin{equation}
f\left(  \gamma\right)  =\sum_{i=1}^{n}c_{i}k_\phi^{g_{i}}\left(  \gamma
_{i}\right)  k_\phi(\gamma_{i},\gamma)  \,\text{, where }c_{i}%
=\frac{k_\phi^h\left(  \gamma_{i}\right)  }{k_\phi(\gamma_i, \gamma_i)}=\frac{k_\phi^h\left(  \gamma_{i}\right)  }{\left\Vert S\left(
\gamma_{i}\right)  \right\Vert _{\phi}^{2}}.\label{k}%
\end{equation}
By simple estimates we obtain
\begin{equation}
\left\vert k_\phi^h\left(  \gamma\right)  -f\left(  \gamma\right)  \right\vert
\leq\sum_{i=1}^{n}\rho_{i}\left(  \gamma\right)  \left\vert k_\phi^h\left(
\gamma\right)  -c_{i}k_\phi(\gamma_{i}, \gamma)  \right\vert
+\sum_{i=1}^{n}\left\vert c_{i}\right\vert \left\vert \rho_{i}\left(
\gamma\right)  -k_\phi^{g_{i}}\left(  \gamma_{i}\right)  \right\vert \left\vert
k_\phi(\gamma_i, \gamma)  \right\vert .\label{bound part}%
\end{equation}
While for each $\gamma$ in $U_{i}$ we have the bound
\[
\left\vert k_\phi^h\left(  \gamma\right)  -c_{i}k_\phi(\gamma_i, \gamma)  \right\vert \leq\left\vert k_\phi^h\left(  \gamma\right)
-k_\phi^h\left(  \gamma_{i}\right)  \right\vert +\left\vert c_{i}\right\vert
\left\vert k_\phi(\gamma_i, \gamma_i)  -k_\phi(\gamma_i, \gamma)  \right\vert \leq\frac{\epsilon}{2},
\]
which we can use in (\ref{bound part}) with (\ref{est unity}), (\ref{k}) and
the estimate $\left\vert k_\phi(\gamma_i, \gamma)  \right\vert \leq
K\left\Vert S\left(  \gamma_{i}\right)  \right\Vert _{\phi}$ to obtain
\[
\left\vert k_\phi^h\left(  \gamma \right)  -f\left(  \gamma\right)
\right\vert \leq\frac{\epsilon}{2}+\sum_{i=1}^{n}\frac{\epsilon}{2n}=\epsilon.
\]
\end{proof}

We prove that the class of signature kernels we have considered are universal in the sense defined above.

\begin{theorem}\label{universal}(Universality of signature kernels)
Let $k_{\phi}$ denote a signature kernel determined by a weight function $\phi$ such that the conditions of Lemma \ref{lemma:condition_phi} hold, and let $\mathcal{H}_{\phi}$ denote the associated RKHS.
Assume that $\mathcal{C}$ is equipped with a topology with respect to which the ST $S:\mathcal{C} \rightarrow T_{\phi}\left(
(V)\right)$ is continuous. Then, $k_{\phi}$ is universal in the sense of Definition \ref{ccuniv}. 
\end{theorem}

\begin{proof}
The continuity and symmetry of $k_{\phi}:\mathcal{C}\times\mathcal{C}
\rightarrow\mathbb{R}$ follows from the second assumption and $k_{\phi}\left(
\gamma,\sigma\right)  =\left\langle S\left(  \gamma\right)  ,S\left(
\sigma\right)  \right\rangle _{\phi}.$ Taking advantage of Proposition
\ref{prop: universal equivalence} it suffices to prove that the set $\left\{ k_\phi^h|_\mathcal{K} : h \in
T_{\phi}\left((  V)\right) \right\}  $ is dense in $C\left(  \mathcal{K}%
\right).$ To do so we note that for any $h$ and $g$ in $T\left(  V\right)  $
and $\gamma$ in $\mathcal{K}$ it holds that
\[
\left\langle h,S\left(  \gamma\right)  \right\rangle _{\phi}\left\langle
g,S\left(  \gamma\right)  \right\rangle_{\phi}=\left\langle h\shuffle
_{\phi}g  ,S\left(  \gamma\right)  \right\rangle _{\phi},
\]
where $\shuffle_{\phi}:T\left(  V\right)  \times T\left(  V\right)  \rightarrow
T\left(  V\right)  $ is the bilinear map determined by%
\[
h\shuffle_{\phi}g  =\frac{\phi\left(  n\right)
\phi\left(  m\right)  }{\phi\left(  n+m\right)  }h\shuffle g, \text{ for } h \in V^{\otimes n}, g \in V^{\otimes m}.
\]
Note that when $\phi \equiv 1$
this is just the usual shuffle product. 

From this it follows that $\left\{ k^h_\phi|_\mathcal{K} : h \in
T_{\phi}\left((  V)\right) \right\}  $  contains the algebra 
\[
\mathcal{A} = \left\{k^h_\phi|_\mathcal{K} : h \in T\left(  V\right)
\right\}.
\]
We can then conclude by using the Stone-Weierstrass Theorem,
since $\mathcal{A}$ can be seen to contain the constant function
$k_\phi^{\mathbf{1}}|_{\mathcal{K}}$ and also to separate points in $\mathcal{K}$.

\end{proof}

\begin{remark}
There are different ways to choose a topology on $\mathcal{C}$. We have seen some discussion of this already in Section \ref{topology}; we refer to \cite{cass2022topologies} for a deeper discussion.
\end{remark}

Beyond consideration of its universality, an important factor in the selection of an appropriate kernel is the efficiency with which its RKHS can represent functions of interest. In principle functions which can be represented by a single element of  $\mathcal{H}%
_{\phi}$ for one choice of $\phi$ may have a much more complex (and only approximate) representation using RHKS elements for another choice. The following example, which is originally from \cite{cass2021general}, gives a concrete illustration of the principle.
\begin{example}[\cite{cass2021general}] \label{hyp}
(Hyperbolic development map) Let $V=\mathbb{R}^{d}$ with the Euclidean inner product and $N:=d+1$. Suppose that $F:\mathbb{R}^{d}\rightarrow \mathcal{M}_{N}$ is the linear map taking values in the space $\mathcal{M}_{N}$ of $N$ by $N$ real matrices and being defined by%
\[
F\left(  x\right)  =\left(
\begin{array}
[c]{cc}%
0 & x\\
x^{T} & 0
\end{array}
\right)  .
\]
There exists a unique solution $y=\left(  y^{i,j}\right)  _{i,j=1,...,N}$ in
$\mathcal{M}_{N}$ to the CDE
\[
dy_{t}=F\left(  d\gamma_{t}\right)  \cdot y_{t}\text{, started at }%
y_{0}=I_{N}\text{ with }t\in\left[  0,1\right]  ,
\]
where $\cdot$ denotes matrix multiplication and $I_{N}$ is the identity
matrix in $\mathcal{M}_{N}$. The real-valued function defined by
\[
\Psi :\gamma\mapsto y_{1}^{N,N}%
\]
is invariant on tree-like equivalence classes and it therefore induces a function on $\mathbb R^{\mathcal{C}}$. It can be shown that this function has the form%
\[
\Psi\left(  \left[  \gamma\right]  \right)  =\left\langle \mathbb{E}\left[
S\left(  \circ B\right)  _{0,1}\right]  ,S\left(  \gamma\right)  \right\rangle
_{\phi}\text{ with }\phi\left(  k\right)  :=2^{k/2}\left(  \frac{k}{2}\right)
!\text{ for }k\in%
\mathbb{N}
\cup\left\{  0\right\}  ,
\]
where $S\left(  \circ B\right)  _{0,1}$ denotes the Stratonovich ST of
$d$-dimensional standard Brownian motion. Using the Fawcett-Lyons-Victoir formula \cite{fawcett2002problems} (see Exercise \ref{ex:expected_sig}) we
have that
\[
A:=\mathbb{E}\left[  S\left(  \circ B\right)  _{0,1}\right]  =\exp\left(
\frac{1}{2}\sum_{i=1}^{d}e_{i}^2\right),
\]
which is easily verified to be in $T_{\phi}\left(  (V) \right)  $ and hence $\Psi$ is in $\mathcal{H}_{\phi}$. 

On the other hand, $\Psi$ is not in the RKHS $\mathcal{H}_{\mathbf 1}$, corresponding to $\phi \equiv 1$. We can prove this by contradiction: if $\Psi$ were in this RKHS then, by Proposition \ref{prop:RKHS}, there would need to exist $\tilde{A}$ in $T_{\mathbf 1} \left( ( V )\right)$ which allows $\Psi$ to be realised
as $\Psi \left(  \left[  \gamma\right]  \right)  = \langle \tilde
{A},S\left(  \gamma\right) \rangle_{\mathbf 1} $. Any such $\tilde{A}%
=\sum_{k\geq0}\tilde{A}_{k}$ would then need to satisfy (again, see Exercise \cref{ex:expected_sig})
\[
\tilde{A}_{2k}=2^{k}k!\pi_{2k}\exp\left(  \frac{1}{2}\sum_{i=1}^{d}e_{i}%
^{2}\right)  =\left(  \sum_{i=1}^{d}e_{i}^{2}\right)  ^{k},
\]
where $\pi_{n}:T((V)) \mapsto V^{\otimes n}$ is the canonical projection. This would imply that 
\[
\sum_{k=0}^{\infty} ||\tilde{A}_{2k}||^{2}=\infty
\]
which would contradict $\tilde{A}$ being an element of $T_{\mathbf{1}} \left( ( V )\right)$. 
\end{example}

\subsection{Characteristic kernels}

Another important property that a kernel can possess is that of being characteristic. Loosely speaking, a kernel $k$ is said characteristic to a topological space if for any (Borel) probability measure $\mu$ on that space, the $\mu$-expectation of the feature map $k(x,\cdot)$ uniquely characterises $\mu$, i.e it distinguishes the measure from any other measure. In what follows we make this statement rigorous for signature kernels.

Throughout this section we will assume that that $\phi$ is a weight function satisfying the conditions of Lemma \ref{lemma:condition_phi} and that $\mathcal{C}$ is equipped with a topology with respect to which the ST $S:\mathcal{C} \rightarrow T_{\phi}\left(
(V)\right)$ is continuous. We denote by $\mathcal{P}(\mathcal{K})$ the set of Borel probability measures on a compact set $\mathcal{K} \subset \mathcal{C}$, which by the Riesz representation theorem is the dual space of $C(\mathcal{K})$. Unless stated otherwise, we endow $\mathcal{P}(\mathcal{K})$ with the topology of weak convergence. Because any continuous function on a compact set is also bounded, for any measure $\mu \in \mathcal{P}(\mathcal{K})$ the following condition holds 
\begin{equation}\label{eqn:finite_exp_kernel}
    \mathbb{E}_{x \sim \mu}\left[\sqrt{k_\phi(x,x)}\right] < +\infty.
\end{equation}
To define what we mean by characteristic kernel, it is convenient to introduce the notion of \emph{kernel mean embedding} (KME).
\begin{definition}\label{def:KME}
    The $\phi$-signature kernel mean embedding (KME) is the function $M^\phi :\mathcal{P}(\mathcal{K}) \to \mathcal{H}_\phi$ that maps a Borel probability measure $\mu \in \mathcal{P}(\mathcal{K})$ to the element
    \begin{equation}\label{eqn:KME}
        M^\phi_\mu := \mathbb{E}_{x \sim \mu}[k_\phi(x,\cdot)] \in \mathcal{H}_\phi
    \end{equation}
    which is well-defined as a Bochner integral since the condition (\ref{eqn:finite_exp_kernel}) holds. 
    We say that the $\phi$-signature kernel $k_\phi$ is characteristic to the compact set $\mathcal{K}$ if the map $M^\phi$ is injective on $\mathcal{P}(\mathcal{K})$.
\end{definition}

An important result from the theory of kernels relates to the equivalence between the notions of universality, characteristicness and strict positive-definiteness of a kernel. In fact, these three notions can be shown to be equivalent for kernels defined over general locally convex topological vector spaces \cite[Theorem 6]{simon2018kernel}. In the following result we consider the simplified setting of compact sets of unparameterised paths. 

\begin{theorem}\cite[Simplified version of Theorem 6]{simon2018kernel}
    The following three statements are equivalent:
    \vspace{-0.3cm}
    \begin{enumerate}
        \item $k_\phi$ is universal.
        \item $k_\phi$ is characteristic to $\mathcal{P}(\mathcal{K})$.
        \item $k_\phi$ is strictly positive definite on $\mathcal{K}$.
    \end{enumerate}
\end{theorem}

\begin{remark}
    We note that the previous result holds under the assumption that the RKHS $\mathcal{H}_\phi|_{\mathcal{K}} \hookrightarrow C(\mathcal{K})$ is continuously embedded in $C(\mathcal{K})$, i.e. that $\mathcal{H}_\phi|_{\mathcal{K}}  \subset C(\mathcal{K})$ and that the topology of $\mathcal{H}_\phi|_{\mathcal{K}} $ is stronger than the topology induced by $C(\mathcal{K})$. The reader is invited to show this statement in \Cref{ex:cont_emb}.
\end{remark}

The RKHS distance between two KMEs yields a semi-metric on measures, often called the \emph{maximum mean discrepancy} (MMD), which has found numerous applications in machine learning and nonparametric testing \cite{gretton2012kernel} and that we introduce next. 

\begin{definition}\label{def:MMD}
    For any two Borel probability measures $\mu,\nu \in \mathcal{P}(\mathcal{K})$, the maximum mean discrepancy (MMD) is defined as 
    \begin{equation}\label{eqn:mmd_}
        d_\phi\left(\mu,\nu\right) = \sup_{f\in\mathcal{F}}(\mathbb{E}_{x \sim \mu}[f(x)]  - \mathbb{E}_{y \sim \nu}[f(y)] )
    \end{equation}
    where $\mathcal{F}$ is the unit ball in $\mathcal{H}_\phi$
\end{definition}

It is possible to show (see \Cref{ex:kme_mmd}) that, under appropriate assumptions, the MMD in equation (\ref{eqn:mmd_}) admits the following equivalent expression 
\begin{equation}\label{eqn:mmd}
    d_\phi(\mu, \nu) := \norm{M^\phi_\mu - M^\phi_\nu}_{\mathcal{H}_\phi}.
\end{equation}
Thus, the squared MMD admits the following expression
\begin{equation}\label{eqn:mmd_squared}
    \normalfont
    d_\phi(\mu,\nu)^2 = \mathbb{E}_{x,x' \sim \mu}[k_\phi(x,x')] - 2\mathbb{E}_{x,y \sim \mu \times \nu}[k_\phi(x,y)] + \mathbb{E}_{y,y' \sim \nu}[k_\phi(y,y')],
\end{equation}
where  $x',y'$ are independent copy of $x,y$ with  distribution $\mu, \nu$, respectively.

Consequently, given $m$ sample paths $\{x^i\}_{i=1}^m \sim \mu$ and $n$ sample paths $\{y^j\}_{j=1}^n \sim \nu$, an unbiased estimator of the MMD is given by the following expression
\begin{equation}\label{eqn:estimator_mmd}
    \normalfont
    \hat{d}_\phi(\mu,\nu)^2 = \frac{1}{m(m-1)} \sum_{j\neq i} k(x^i,x^j) - \frac{2}{mn}\sum_{i,j} k(x^i,y^j)  + \frac{1}{n(n-1)}\sum_{j\neq i} k(y^i,y^j).
\end{equation}

\begin{lemma}\label{lemma:char}
For any Borel probability measures $\mu,\nu \in \mathcal{P}(\mathcal{K})$ the following  holds
\normalfont
$$d_\phi(\mu,\nu)=0 \iff \mu = \nu.$$
\end{lemma}

\begin{proof} (Adapted from \cite{gretton2007kernel})
Clearly, if $\mu = \nu$ then $d_\phi(\mu,\nu)=0$. By \Cref{universal}, the $\phi$ signature kernel $k_\phi$ is universal, hence for any $\epsilon > 0$ and $f \in C(\mathcal{K})$ there exists $g \in \mathcal{H}_\phi$ such that 
$\sup_{x \in \mathcal{K}} |f(x) - g(x)| < \epsilon.$
By the triangle inequality
\begin{align*}
    |\mathbb{E}_{x \sim \mu}[f(x)] - \mathbb{E}_{y \sim \nu}[f(y)]| &\leq |\mathbb{E}_{x \sim \mu}[f(x)] - \mathbb{E}_{x \sim \mu}[g(x)]| \\
    &+ |\mathbb{E}_{x \sim \mu}[g(x)] - \mathbb{E}_{y \sim \nu}[g(y)]| \\
    &+ |\mathbb{E}_{y \sim \nu}[g(y)] - \mathbb{E}_{y \sim \nu}[f(y)]|
\end{align*}
The first and third term satisfy
\begin{align*}
     |\mathbb{E}_{x \sim \mu}[f(x)] - \mathbb{E}_{x \sim \mu}[g(x)]| &\leq \mathbb{E}_{x \sim \mu}|f(x)-g(x)| \leq \epsilon 
     \\
    |\mathbb{E}_{y \sim \nu}[g(y)] - \mathbb{E}_{y \sim \nu}[f(y)]| &\leq \mathbb{E}_{y \sim \nu}|g(y)-f(y)| \leq \epsilon 
\end{align*}
For the second term we have that
\begin{equation*}
    |\mathbb{E}_{x \sim \mu}[g(x)] - \mathbb{E}_{y \sim \nu}[g(y)]| = \langle g, M^\phi_\mu - M^\phi_\nu\rangle_{\mathcal{H}_\phi} = 0 
\end{equation*}
because by assumption $d_\phi(\mu,\nu) = 0$, therefore $M^\phi_\mu = M^\phi_\nu$. Hence
\begin{equation*}
    |\mathbb{E}_{x \sim \mu}[f(x)] - \mathbb{E}_{y \sim \nu}[f(y)]| \leq 2 \epsilon
\end{equation*}
which implies $\mu=\nu$ by \cite[Lemma 9.3.2]{dudley2018real}\label{lemma:hypo}, stating that for any Borel probability measures $\mu,\nu$ on $\wC$, the equality $\mu = \nu$ holds if and only if $\mathbb{E}_{x \sim \nu}[f(x)] = \mathbb{E}_{y \sim \mu}[f(y)]$ for all $f \in C_b(\wC)$, i.e. bounded continuous real-valued functions on $\wC$.
\end{proof}

\begin{remark}
    Without the compactness assumption on $\mathcal{K}$, \Cref{lemma:char} no longer holds, as shown by the following counter-example (\cite[Example B.5]{chevyrev2022signature}). Define $X,Y : [0,1] \to \R^2$ by $X_t = tM^\top$ and $Y_t = tN^\top$ where $M = (M_1,M_2)$ are two-independent lognormal random variables and $N=(N_1,N_2)$ has density 
    \begin{equation*}
        p_N(n_1,n_2) = p_M(n_1,n_2)\prod_{i=1}^2(1 + \sin(2\pi n_i))
    \end{equation*}
    where $p_M$ is the density of the lognormal $M$. It is easy to see that the expected STs of $X$ and $Y$ coincide and therefore that $d_{\phi}(\mu_X,\mu_Y) = 0$, although $X,Y$ have different laws $\mu_X \neq \mu_Y$. We note that it is nonetheless possible to restore characteristicness of the signature kernel in several ways; we refer the interested reader to \cite{chevyrev2022signature, cuchiero2023global}.
\end{remark}

It is natural to ask when convergence in MMD metric is equivalent to weak convergence. In that case, one says that the corresponding kernel metrizes the weak convergence of probability measures. This question has been deeply studied for general kernels in several articles. In particular, \cite{sriperumbudur2010hilbert} presented sufficient conditions under which the MMD metrizes the weak convergence when the underlying space is either $\mathbb R^d$ or a compact metric space. These results were then generalised in \cite{sriperumbudur2016optimal} where the authors showed that any continuous, bounded, integrally strictly positive
definite kernel over a locally compact Polish space such that RKHS-functions are continuous and vanish at infinity metrizes the weak convergence. These results were generalised even further in \cite{simon2020metrizing} where the Polish assumption is dropped in favour of more general Hausdorff spaces. In the following result, we place ourselves in the simplified setting of \cite{sriperumbudur2010hilbert} and show that $\phi$-signature kernels metrize the weak convergence of compactly supported Borel probability measures in $\mathcal{P}(\mathcal{K})$.

\begin{lemma}\label{lemma:kme_weakly_cont}
The kernel mean embedding $M^\phi$ is a  continuous map from $\mathcal{P}(\mathcal{K})$ endowed with the topology of weak convergence to $\mathcal{H}_\phi$ with its Hilbert norm topology.
\end{lemma}

\begin{proof}
Consider a sequence $(\mu_n)_{n \in \mathbb{N}}$ of Borel probability measures on $\mathcal{P}(\mathcal{K})$ converging weakly to a measure $\mu \in \mathcal{P}(\mathcal{K})$. Then
\begin{align*}
    \norm{M^\phi_{\mu_n} - M^\phi_{\mu}}_{\cH_\phi}^2 &= \bE_{(x,x') \sim \mu_n \times \mu_n}[k_\phi(x,x')] - 2\bE_{(x,y) \sim \mu_n \times \mu}[k_\phi(x,x')] \\ &+ \bE_{(y,y') \sim \mu \times \mu}[k_\phi(y,y')]
\end{align*}
where $\mu_n \times \mu_n$ denotes the product measure of $\mu_n$ itself; similarly for $\mu_n \times \mu$ and $\mu \times \mu$. Because $(\mu_n)_{n \in \mathbb{N}}$ converges weakly  in $\mathcal{P}(\mathcal{K})$ to $\mu$, the sequence of product measures $(\mu_n \times \mu_n)_{n \in \mathbb{N}}$ converges weakly  in $\mathcal{P}(\mathcal{K}) \times \mathcal{P}(\mathcal{K})$ to $\mu \times \mu$; similarly for $(\mu_n \times \mu)_{n \in \mathbb{N}}$. By continuity of $k_\phi$ and compactness of $\mathcal{K}$, $k_\phi$ is a bounded continuous function in both of its variables, thus by definition of weak convergence it follows that 
\begin{align*}
    &\bE_{(x,x') \sim \mu_n \times \mu_n}[k_\phi(x,x')] \to \bE_{(x,x') \sim \mu \times \mu}[k_\phi(x,x')], 
    \\
    &\bE_{(x,x') \sim \mu_n \times \mu}[k_\phi(x,x')] \to \bE_{(x,x') \sim \mu \times \mu}[k_\phi(x,x')] 
\end{align*}
as $n \to +\infty$. Thus, $\norm{M^\phi_{\mu_n} - M^\phi_{\mu}}_{\cH_\phi}^2 \to 0$ as $n \to +\infty$ concluding the proof.

\end{proof}

\begin{lemma}
    $\phi$-signature kernels metrize the weak topology on $\mathcal{P}(\mathcal{K})$, i.e. the $\phi$-signature kernel MMD is equivalent to the topology of weak convergence on $\mathcal{P}(\mathcal{K})$.
\end{lemma}

\begin{proof}(Adapted from \cite{sriperumbudur2010hilbert})
    Clearly, if a sequence $\{\mu_n\}_{n \in \mathbb N}$ of  measures converges weakly to a measure $\mu$ in $\mathcal{P}(\mathcal{K})$, then by the same argument as in \Cref{lemma:kme_weakly_cont}, the MMD $d_\phi(\mu_n, \mu) \to 0$ as $n \to \infty$. Conversely, since by \Cref{universal} $\mathcal{H}_\phi|_\mathcal{K}$ is dense in $C_b(\mathcal{K})$ in the topology of uniform convergence, one has that for any $f \in C_b(\mathcal{K})$ and any $\epsilon >  0$, there exists $g \in \mathcal{H}_\phi$ such that $\sup_{x \in \mathcal{K}} |f(x) - g(x)| < \epsilon.$ Hence
    \begin{align*}
        |\mu_n(f) - \mu(f)| &= |\mu_n(f-g) + \mu(g-f) + (\mu_n(g) - \mu(g))| \\
        &\leq \mu_n(|f-g|) + \mu(|f-g|) + |\mu_n(g) - \mu(g)|\\
        &\leq 2 \epsilon + |\mu_n(g) - \mu(g)| \\
        &\leq 2 \epsilon + \norm{g}_{\mathcal{H}_\phi} d_\phi(\mu_n, \mu).
    \end{align*}
    Since $d_\phi(\mu_n, \mu) \to 0$ as $n \to \infty$, the sequence  $\{\mu_n\}_{n \in \mathbb N}$ converges weakly to $\mu$.
\end{proof}

\subsection{Hypothesis testing}

This approach, and variants of it, have been used in the kernel learning to propose statistics for goodness-of-fit tests, see e.g. \cite{gretton2012kernel}. The goal in these
tasks is to use the MMD to design the rejection regions for hypothesis tests to determine whether an observed sample is drawn from a known target distribution. 

In the setting which is of interest to us, namely when the distributions are supported on (compact sets of) unparameterised paths, one is  given i.i.d. sample paths $\{x^1,...,x^m\} \sim \mu$ and $\{y^1,...,y^n\} \sim \nu$ we are interested in testing the \emph{null hypothesis} $H_0 : \mu = \nu$ against the \emph{alternative hypothesis} $H_1 : \mu \neq \nu$. This is achieved by comparing the test statistics $\hat{d}_\phi(\mu,\nu)$ with a particular threshold: if the threshold is exceeded, then the test rejects the null hypothesis. 

The \emph{acceptance region} of the test is the set of scalars below the threshold. A type I error is when $H_0$ is rejected despite being true. A type II error is when $H_0$ is accepted despite being false. The \emph{level} $\alpha$ of a test is an upper bound on the probability of a Type I error: this is a design parameter of the test which must be set in advance, and is used to determine the threshold to which we compare the test statistic. The \emph{power} of a statistical test is defined as $1$ minus the type II error. A test of level $\alpha$ is consistent if it achieves a zero type II error in the limit of infinite data samples. 

\cite[Theorem 10]{gretton2012kernel} states that the estimator (\ref{eqn:estimator_mmd}) is \emph{consistent}, which means that it converges in probability to the true value of the MMD when the number of sample paths goes to infinity. Assuming for simplicity that $m=n$ and that $k_\phi(x_i,y_j)<M$ for all $i=1,...,m$ and $j=1,...,n$, \cite[Corollary 11]{gretton2012kernel} states that the MMD-based hypothesis test of level $\alpha$ for the null hypothesis $H_0$ has the acceptance region 
$$\hat{d}_\phi(\mu,\nu)^2<\frac{4M}{\sqrt{m}}\sqrt{\log(\alpha^{-1})}.$$

We note that the above considerations do not immediately apply to the special case where $\mu$ is determined by the Wiener measure, because the latter is not compactly supported. Nonetheless, the result of \cite{chevyrev2016characteristic}
ensures that $d_{\phi}\left(  \mathcal{W},\nu\right)  =0$ if and only if $\mathcal{W=\nu}
$ still holds. The estimation $\hat{d}_\phi(\mu,\nu)$ from a sample $\left\{y^1,...,y^n\right\}$ drawn from $\nu$ involves especially handling the term
\[
\left\langle \mathbb{E}_{\Gamma\sim\mathcal{W}}[S\left(  \Gamma\right)
],S\left(  y^i \right)  \right\rangle _{\phi}.
\]
which is a statistic than can be used as a goodness-of-fit test to assess how likely are sample paths $\left\{y^1,...,y^n\right\}$ to be Brownian paths. Example \ref{hyp} can be extended to derive a closed-form formula for this expression across a selection of weighted signature kernels. See also \Cref{ex:expected_sig_continued} for an application to cubature on Wiener spaces.

\begin{remark}
Other types of hyptothesis tests can be constructed using the signature kernel to test independence \cite{gretton2007kernel} and conditional independence \cite{zhang2012kernel} for path-valued random variables. 
\end{remark}

\subsection{Distribution regression}

Distribution Regression (DR) refers to the supervised learning problem of regressing from probability measures to scalar targets. In practice, DR becomes useful when labels are only available for groups of inputs rather than for individual inputs. 

In this context, elements of a set are viewed as samples from an underlying probability measure \cite{szabo2016learning, muandet2012learning, flaxman2015machine, smola2007hilbert}. The DR framework can be intuitively summarized as a two-step procedure. Firstly, a probability measure $\mu$ supported on $\mathcal{X}$ is mapped to a point in an RKHS $\mathcal{H}_1$ by means of a kernel mean embedding $\mu \to M_\mu = \int_\mathcal{X} k_1(\cdot,x)\mu(dx)$, where $k_1: \mathcal{X} \times \mathcal{X} \to  \mathbb{R}$ is the associated reproducing kernel. Secondly, the regression is finalized by approximating a function $F:\cH_1 \to\mathbb{R}$ via a minimization of the form $F \approx \argmin_{g \in  \cH_2}\sum_{i=1}^M(y^i - g(M_{\mu^i}))^2$, resulting in a procedure involving a second kernel $k_2:\cH_1 \times \cH_1 \to \mathbb{R}$ with RKHS $\cH_2$.

\begin{figure*}[h]
\centering
\includegraphics[width=\textwidth,trim={0 4.1cm 0 0},clip]{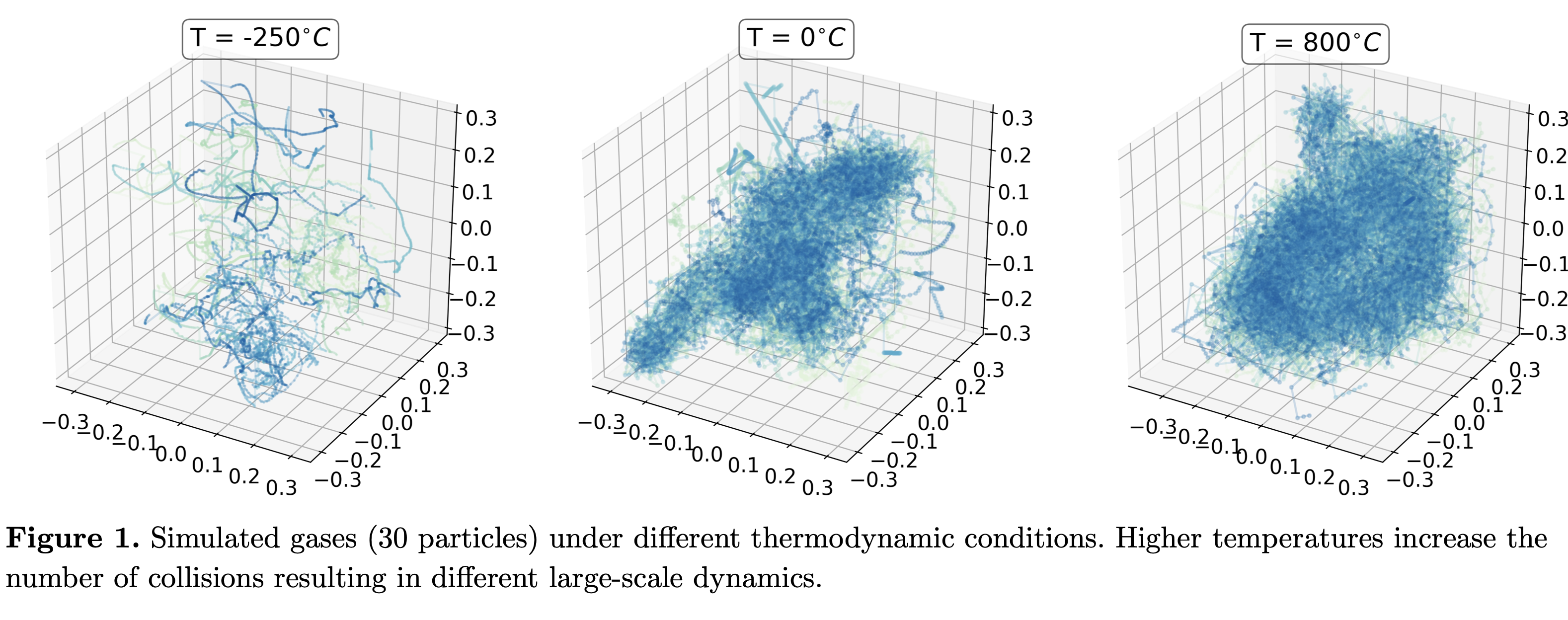}
\caption{\small Simulation of the trajectories traced by $20$ particles of an ideal gas in a $3$-d box under different thermodynamic conditions. Higher temperatures equate to a higher internal energy in the system which increases the number of collisions resulting in different large-scale dynamics of the gas. Figure taken from \cite{lemercier2021distribution}.}
\label{fig:ideal_gas}
\end{figure*}

Here, we are  interested in the case where inputs are Borel probability measures on $\wC$. For instance, in thermodynamics one may be interested in determining the temperature of a gas from the set of trajectories described by its particles \cite{reichl1999modern} as depicted in \Cref{fig:ideal_gas}. For other examples we refer the reader to \cite{lemercier2021distribution}. To address the challenge of DR on paths, we will make use of signature kernels to define a universal kernel on Borel probability measures on $\wC$. The following is \cite[Theorem 3.3]{lemercier2021distribution} adapted to the setting of $\phi$-signature kernels


\begin{theorem}\label{thm:kernel_DR} Let  $\mathcal{P}(\wC)$ denote the space of Borel probability measures on $\wC$ endowed with the topology of weak convergence. Let $k_{\phi}$ denote a signature kernel determined by a weight function $\phi$ such that the condition of Lemma \ref{lemma:condition_phi} holds. Let $f : \mathbb{R} \to \mathbb{R}$ be an entire function, i.e. analytic and with infinite radius of convergence. Then, the kernel $K_\phi:\mathcal{P}(\wC)\times\mathcal{P}(\wC)\to\mathbb{R}$  defined as 
\begin{equation}\label{eqn:univ_kernel}
\normalfont
K_\phi(\mu,\nu) =  f\big(\left\langle M^\phi_\mu, M^\phi_\nu\right\rangle_{\mathcal{H}_\phi}\big)
\end{equation}
is cc-universal on $C(\mathcal{P}(\wC))$ in the same sense as per \Cref{ccuniv}, i.e. for every compact subset $\mathcal{K} \subset \mathcal{C}$,
the associated RKHS is dense in the space of weakly continuous functions from $\mathcal{P}(\mathcal{K})$ to $\mathbb{R}$ in the topology of uniform convergence.
\end{theorem}

\begin{proof}[Proof of \Cref{thm:kernel_DR}]
By \cite[Theorem 2.2]{christmann2010universal} if $\widehat{\mathcal{K}}$ is a compact metric space and $\mathcal{H}$ is a separable Hilbert space such that there exists a continuous and injective map $\rho : \widehat{\mathcal{K}} \to \mathcal{H}$, then the RKHS of any kernel $k : \widehat{\mathcal{K}}\times \widehat{\mathcal{K}} \to \mathbb{R}$ of the form
\begin{equation*}
    k(x,y) = f(\langle \rho(x), \rho(y)\rangle_\mathcal{H})
\end{equation*}
is dense in $C(\widehat{\mathcal{K}})$ in the topology of uniform convergence. If $\mathcal{K} \subset \wC$ is a compact set of unparameterised paths, then the set $\mathcal{P}(\mathcal{K})$ is weakly-compact (e.g. \cite[Theorem 10.2]{walkden2014ergodic}). 
By \Cref{lemma:kme_weakly_cont} the signature kernel mean embedding map is injective and weakly continuous. Furthermore $\mathcal{H}_\phi$ is a Hilbert space with a countable basis, hence it is separable. Setting $\widehat{\mathcal{K}}=\mathcal{P}(\mathcal{K}), \mathcal{H} = \mathcal{H}_\phi$ and $\rho : \mu \to M^\phi_{\mu}$  concludes the proof. 
\end{proof}

\begin{example}
By \Cref{thm:kernel_DR}, the following Gaussian type kernel 
\begin{equation*}
    \normalfont
    K_\phi(\mu,\nu) =  \exp\left(-\frac{d_\phi(\mu,\nu)^2}{\sigma^2}\right), \quad \sigma \in \mathbb{R},
\end{equation*}
is a universal kernel on $\mathcal{P}(\wC)$, where $d_\phi$ is the $\phi$-signature MMD of equation (\ref{eqn:mmd_}). 
\end{example}

\begin{remark}
    In light of Theorem \ref{thm:kernel_DR}, DR on sequential data can be performed via any kernel method for regression such as SVM, kernel ridge etc. For examples of DR on streams from  thermodynamics, mathematical finance and agricultural science we refer the interested reader to \cite{lemercier2021distribution} and to \cite{cochrane2021sk} for an application to cybersecurity.
\end{remark}

\section{Computing signature kernels}

In this section we present various approaches for computing signature kernels.

\subsection{Truncated signature kernels}

A simple way for approximating a $\phi$-signature kernel of two paths $x,y \in C_1([a,b],V)$ amounts to compute the $\phi$-inner product between the two TSTs truncated at some level $n\in\N$, i.e.
\begin{equation} \label{trunestimate}
    k^{n,x,y}_\phi(s,t) = \sum_{i=1}^n \phi(i) \left\langle S(x)^i_{a,s}, S(y)^i_{a,t}\right\rangle_{V^{\otimes i}}
\end{equation}
The following lemma states that the truncated $\phi$-signature kernel converges to the $\phi$-signature kernel as the truncation $n$ goes to infinity.
\begin{lemma}
    Let $x,y \in C_1([a,b],V)$ and assume that $\phi : \mathbb{N} \cup \{0\} \to \mathbb{R}_{+}$ is such that $\phi$ satisfies the condition of \Cref{lemma:condition_phi}. Then the truncated $\phi$-signature kernel converges to the $\phi$-signature kernel with the following approximation error bound
    \begin{equation*}
        \left|k^{x,y}_\phi(s,t) - k^{n,x,y}_\phi(s,t)\right| \leq \sum_{i=n+1}^\infty \phi(i) \frac{(\norm{x}_{1,[a,s]}\norm{y}_{1,[a,t]})^i}{(i!)^2}
    \end{equation*}
\end{lemma}

\begin{proof}
    By Cauchy-Schwarz we get
    \begin{align*}
         \left|k^{x,y}_\phi(s,t) - k^{n,x,y}_\phi(s,t)\right| 
         &\leq \sum_{i=n+1}^\infty |\phi(i)| \left|\big\langle S(x)^i_{a,s}, S(y)^i_{a,t}\rangle_{V^{\otimes i}}\right| \\
         &\leq \sum_{i=n+1}^\infty |\phi(i)| \norm{S(x)^i_{a,s}}_{V^{\otimes i}}\norm{S(y)^i_{a,t}}_{V^{\otimes i}}\\
         &\leq \sum_{i=n+1}^\infty |\phi(i)| \frac{\norm{x}_{1,[a,s]}^i\norm{y}_{1,[a,t]}^i}{(i!)^2}
    \end{align*}
    where the last inequality follows from  \Cref{prop:factorial_decay}. The fact that this sum is finite is a consequence of the summability assumption on $\phi$.
\end{proof}

Computing signature kernels by taking inner products of truncated signatures is somewhat unsatisfying because it still relies on some means of computing iterated
integrals. This in general is both computationally expensive from a space complexity viewpoint but is also goes against the general philosophy of of kernel methods which is to avoid a direct evaluation of the underlying feature map. These drawback are particularly apparent in situations where the input streams take their values in a high dimensional space $V$ due to the exponential explosion of terms in the ST. It also affects considerations of accuracy. As can be see from the estimate in (\ref{trunestimate}) the $n$-truncated signature kernel will be good approximation to its untruncated counterpart only when $n!$ starts to dominates the $n^{\text{th}}$ power of the product of the length of the two paths $x$ and $y$. In particular, for paths having very large $1$-variation a significant number of signature levels may be needed. 

\subsection{Finite difference schemes for the signature kernel PDE}\label{sec:finite_difference}

In view of our prior discussion, it is natural to ask whether it is possible to obtain a kernel trick allowing to compute the signature kernel without referring back to the ST. In the case where the weight function is constant, we have already established that, for any two paths $x,y \in C_1([a,b],V)$, the $\phi$-signature kernel $k^{x,y}$ satisfies the integral equation (\ref{eqn:int_sigker_original}), which is equivalent to PDE (\ref{eqn:kernel_PDE}) if the paths are differentiable. 

Consider the practical situation where $x,y$ are obtained by piecewise linear interpolation of two time series $\mathbf x = ((s_0,\mathbf x_0), ..., (s_m,\mathbf x_m))$ and  $\mathbf y = ((t_0,\mathbf y_0), ..., (t_n,\mathbf t_n))$, with $s_0=t_0=a$ and $s_m=t_n=b$.  

Because for any $0 \leq i \leq m$ and $0 \leq j \leq n$ the restrictions $x|_{[s_i,s_{i+1}]}$ and $y|_{[s_i,s_{i+1}]}$ are linear paths in $V$, for $s_i\leq u \leq s \leq s_{i+1}$ and $t_j \leq v \leq t \leq t_{j+1}$, the signature kernel integral equation (\ref{eqn:int_sigker_original}) becomes
\begin{equation}\label{eqn:integral_form}
    k^{x,y}(s,t) =  k^{x,y}(s,v) +  k^{x,y}(u,t) -  k^{x,y}(u,v) + C \int_u^s\int_v^t  k^{x,y}(p,q) dpdq
\end{equation}
where $C = \left\langle \dot x_s, \dot y_t \right\rangle_V$ is a constant. The double integral in equation (\ref{eqn:integral_form}) can be approximated using numerical integration methods (see \cite{day1966finitediff, wazwaz1993finitediff}). For example, using a trapezoidal rule, (\ref{eqn:integral_form}) is approximated by the following explicit finite difference scheme
\begin{align}\label{eqn:finite_difference}
    k^{x,y}(s,t) \approx k^{x,y}(s, v) + k^{x,y}(u,t) - k^{x,y}(u,v) 
    + \frac{C}{2}\left(k^{x,y}(s, v) + k^{x,y}(u, t)\right)(s-u)(t-v).
\end{align}

\begin{remark}
    The finite difference scheme of equation (\ref{eqn:finite_difference}) can be applied, for some $\lambda\in\mathbb{N} \cup \{0\}$, on a dyadically refined computational grid $\mathcal{D}^\lambda$ of the form
    $$\mathcal{D}^\lambda = \bigcup_{\substack{0\leq i\leq m \\ 0\leq j\leq n}}  \left\{\left(s_i + p\,2^{-\lambda}(s_{i+1} - s_i), t_j + q\,2^{-\lambda}(t_{j+1} - t_j)\right)\right\}_{0\,\leq\, p, q\, \leq\, 2^\lambda}.$$ 
    It is shown in \cite[Theorem 3.5]{salvi2021signature} that be the resulting numerical solution $\hat{k}^{x,y}_\lambda$ converges globally on $[a,b]^2$ to the true signature kernel $k^{x,y}$ with the following uniform error rates
    \begin{equation}
    \sup_{\mathcal{D}}\big|k^{x,y}(s, t) - \hat{k}^{x,y}(s, t)\big| \lesssim \frac{1}{2^{2\lambda}},\ \ \ \text{for all $\lambda \geq 0$}
    \end{equation}
\end{remark}

\subsection{Weighted signature kernels as averaged PDE solutions}

In this section we show how $\phi$-signature kernels can be represented, under suitable integrability conditions, as the average of rescaled PDE solutions whenever the sequence $\{\phi(k) : k = 0,1,...\}$ coincides with the sequence of moments of a random variable. We have already seen that, in the special case where $\phi$ is constant, the computation collapses to solving a single PDE, making the resulting signature kernel particularly appealing for scaling computations to large datasets. 

Next we prove that the $\phi$-signature kernel  can be computed as the average of rescaled PDE solutions whenever the sequence $(\phi(n))_{n \in N}$ coincides with the sequence of moments of a random variable. This result is due to \cite{cass2021general}.

\begin{proposition}\label{prop:averaged_pde}
    Let $\pi$ be a real-valued random variable with finite moments of all orders. Consider the weight functions
    \begin{equation*}
        \phi(k) = \mathbb{E} [\pi^k] \quad \text{and} \quad \psi(k) =  \mathbb{E} [|\pi|^k]
    \end{equation*}
    such that $\psi$ satisfies the condition of \Cref{lemma:condition_phi}. Then, for any $x,y \in C_1([a,b],V)$ the $\phi$-signature kernel is well-defined and satisfies
    \begin{equation}\label{eqn:average_PDE}
        k_\phi^{x,y}(s,t) = \mathbb E_\pi [k^{\pi x, y}(s,t)] = \mathbb E_\pi [k^{x, \pi y}(s,t)]
    \end{equation}
\end{proposition}

\begin{proof}
    The function $|\phi|$ satisfies the condition of \Cref{lemma:condition_phi} because $\psi$ does, therefore the $\phi$-signature kernel is well-defined. In addition we have
    \begin{align*}
        k_\phi^{x,y}(s,t) &= \sum_{i=0}^\infty \mathbb{E} [\pi^i] \left\langle S(x)^i_{a,s}, S(y)^i_{a,t}\right\rangle_{V^{\otimes i}} \\
        &= \mathbb{E}_\pi \left[\sum_{i=0}^\infty \pi^i \left\langle S(x)^i_{a,s}, S(y)^i_{a,t}\right\rangle_{V^{\otimes i}}\right] \\
        &= \mathbb{E}_\pi \left[\sum_{i=0}^\infty  \left\langle S(\pi x)^i_{a,s}, S(y)^i_{a,t}\right\rangle_{V^{\otimes i}}\right]\\
        &= \mathbb{E}_\pi \left[k^{\pi x, y}(s,t)\right]
    \end{align*}
    where the first and last equalities follow from the definition of $\phi$- and ordinary signature kernels respectively, the second equality holds by Fubini's theorem which is applicable because $\psi$ satisfies the condition of \Cref{lemma:condition_phi} and the third follows from the same arguments as in the proof of \Cref{exercise:sigker_rescaling}. 
\end{proof}

If the random variable $\pi$ has a known probability density function to sample from, the expectation in equation (\ref{eqn:average_PDE}) can be computed by numerical methods such as Monte Carlo method or Gaussian quadrature procedure by averaging PDE numerical solutions presented in \Cref{sec:finite_difference}.

The integral in \Cref{prop:averaged_pde} can be numerically approximated by making repetitive calls to the finite difference scheme of \Cref{sec:finite_difference} and then averaging PDE solutions using Gaussian Quadrature rules (see for example \cite{suli2003introduction}). For a general weight function $\psi \in L^1((\alpha,\beta))$, consider a family of polynomials $\{P_n : n\geq 0\}$ that are orthogonal in $L^2(\psi(\omega)d\omega, (\alpha,\beta))$. Then, classical Gaussian Quadrature produces a set of points $\omega_0, ..., \omega_n$ (the roots of $P_{n+1}$) and a set of weights $\psi_1,...,\psi_n$ such that 
\begin{equation*}
    \int_\alpha^\beta k^{\omega x, y}(s,t) \psi(\omega)d\omega \approx \sum_{k=0}^n\psi_k k^{\omega_k x, y}(s,t)
\end{equation*}
For explicit error rates of such approximation see \cite[Section 4]{cass2021general}.


\section{Learning with signature kernels}

The celebrated representer theorem states that a large class of optimization problems with RKHS regularizers have solutions
that can be expressed as finite linear combinations of kernel evaluations centered at training points.

\subsection{Representer theorems}\label{sec:representer_theorem}

We present this result in the context of unparameterised $\phi$-signature kernels.

\begin{theorem}\label{thm:representer}
    Let $\phi : \mathbb{N} \cup \{0\} \to \mathbb{R}$ be such that $|\phi|$ satisfies the condition of \Cref{lemma:condition_phi}. Consider a set $\{(x^i, y^i) \in \wC \times \R \}_{i=1}^N$ of input-output pairs, where the inputs are unparameterised paths and outputs are scalars. Let  $\ell : (\wC \times \R^2)^N \to \R$ be an arbitrary function and let $\mathcal{L} : \R^\wC \to \R$ be the following functional 
    \begin{equation}\label{eqn:repr_functional_}
        \mathcal{L}(f) = \ell\left((x^1,y^1,f(x^1)), ..., (x^N,y^N,f(x^N))\right)
    \end{equation}
    Then, any minimiser
    \begin{equation}\label{eqn:repr_functional}
        f^* = \argmin_{f \in \mathcal{H}_\phi} \mathcal{L}(f) + \lambda \norm{f}_{\mathcal{H}_\phi}^2, \quad \lambda > 0,
    \end{equation}
    if it exists, is such that $f^* \in \mathcal{A}$ where
    \begin{equation*}
        \mathcal{A} := \text{Span}\{ k_\phi(\cdot,x^i) : i=1,...,N\}.
    \end{equation*}
\end{theorem}

The proof of a representer theorem of this form is classical, see for example \cite{scholkopf2001generalized}. The overall idea goes as follows. First one notes that any $f \in \mathcal{H}_\phi = \mathcal{A} \oplus \mathcal{A}^\top$ can be decomposed as $f = g + f^\perp$ where $g \in \mathcal{A}$ and $f^\perp \in \mathcal{A}^\top$. Hence, $f^{\perp}(x^j)=0$, and so
$f(x^j) = g (x^j)$, for all $j=1,...,N$, from which it follows that $\mathcal{L}\left(  g\right)  =\mathcal{L}\left(
f\right) $. In addition, any $f$ in $\mathcal{H}_{\phi}$ satisfies $\left\vert \left\vert
f\right\vert \right\vert _{\mathcal{H}_{\phi}}\geq\left\vert \left\vert
g\right\vert \right\vert _{\mathcal{H}_{\phi}}$, which leads to the inequality 
$$\mathcal{L}\left(  g\right)
+\lambda\left\vert \left\vert g\right\vert \right\vert _{\mathcal{H}_{\phi}%
}^{2}\leq\mathcal{L}\left(  f\right)  +\lambda\left\vert \left\vert
f\right\vert \right\vert _{\mathcal{H}_{\phi}}^{2}.$$

    
    

 
\begin{remark}\label{rq:existence_uniqueness}
    A solution to the optimisation problem (\ref{eqn:repr_functional}) in \Cref{thm:representer} might not exist in general, and if it does it might not be unique. A sufficient condition to ensure existence and uniqueness is to require that the functional $f \mapsto \mathcal{L}(f) + \lambda \norm{f}_{H_\phi}^2$ is continuous and convex from $\mathcal{H}_\phi$ to $\R$, see \Cref{ex:existence_uniqueness_RT}.
\end{remark}

Several popular techniques can be cast in such regularization framework.

\begin{example}
    Regularised least squares, also known as (kernel) ridge regression, corresponds to the following choice of functional in equation (\ref{eqn:repr_functional_})
    \begin{equation*}
        \ell\left((x^1,y^1,f(x^1)), ..., (x^N,y^N,f(x^N))\right) =  \frac{1}{N}\sum_{i=1}^N (y^i-f(x^i))^2.
    \end{equation*}
    By \Cref{thm:representer} and \Cref{rq:existence_uniqueness}, we know a unique minimiser of (\ref{eqn:repr_functional}) exists and is of the form $f^* = \sum_{i=1}^N\alpha_i k_\phi(x^i,\cdot)$ for some vector  $\boldsymbol \alpha = (\alpha_1,...,\alpha_N) \in \R^N$ of scalars. Substituting this representation of $f^*$ in (\ref{eqn:repr_functional}) yields the finite-dimensional optimisation problem
    \begin{equation}\label{eqn:kernel_ridge}
        \min_{\boldsymbol \alpha \in \R^N} \frac{1}{N}\norm{\boldsymbol y - \boldsymbol K \boldsymbol \alpha}_{\R^N}^2 + \lambda \boldsymbol \alpha^\top \boldsymbol K\boldsymbol \alpha,
    \end{equation}
    where $\boldsymbol y = (y^1,...,y^N)$ and $\boldsymbol K \in \R^{N \times N}$ is the Gram matrix with $i,j$ entry equal to $k_\phi(x^i,x^j)$. Equation (\ref{eqn:kernel_ridge})  is a convex optimisation problem and its solution can found for example by differentiating with respect to $\boldsymbol \alpha$  equating to zero, leading to the following linear system of $N$ equations
    \begin{equation*}
        (\boldsymbol K + N\lambda\boldsymbol I) \boldsymbol \alpha =  \boldsymbol y.
    \end{equation*}
\end{example}

\begin{example}
    Given binary labels $y^i = \pm 1$, the support vector machine (SVM) classifier can be interpreted as a regularization method corresponding to the following choice of functional in equation (\ref{eqn:repr_functional_})
     \begin{equation*}
        \ell\left((x^1,y^1,f(x^1)), ..., (x^N,y^N,f(x^N))\right) = \sum_{i=1}^N \max\{0, 1-y^if(x^i)\}.
    \end{equation*}
\end{example}

\newpage

%% file: ex_chapter2.tex
\section{Exercises}

\begin{exercise}\label{ex:kernel_linear}
Let $x,y \in C_1([a,b],V)$ be two linear paths with (constant) derivatives $\dot x \equiv d_x$ and $\dot y \equiv d_y$. 

\begin{enumerate}
\item Show that the $\mathbf{1}$-signature kernel satisfies the following PDE
\begin{equation}\label{eqn:linear_sigker}
    \frac{\partial^2 k^{x,y}}{\partial s \partial t} = C_3 k^{x,y}, \quad \text{where } C_3 = \left\langle d_x, d_y \right\rangle_V,
\end{equation}
and with boundary conditions $k^{x,y}(0,t) = k^{x,y}(s,0) = 1$ for all $s,t \in [a,b]$.
\item  Show that 
\begin{equation*}
k^{x,y}(s,t) =\sum_{k=0}^\infty\frac{
\left(st\left\langle d_x,d_y\right\rangle_V \right)^k}{\left( k!\right) ^{2}}%
=I_{0}\left( 2\sqrt{st\left\langle d_x,d_y\right\rangle_V }\right) ,
\end{equation*}
where $I_{0}$ is the modified Bessel function of the first kind of order $0$.

\item Recall that $I_{0}$ satisfies $z^{2}I_{0}^{\prime \prime }\left(
z\right) +zI_{0}^{\prime }\left( z\right) -z^{2}I_{0}\left( z\right) =0.$
Taking $x = y$ and $s=t$, and  $h\left( t\right) = k^{x,x}\left( t,t\right)$, find a second-order differential equation satisfied by $h$.

\item Show that the solution to the following PDE
\begin{equation*}
\frac{\partial^2 z}{\partial s\partial t} = \left\langle
d_x,d_y\right\rangle_V z,
\end{equation*}
with $(s,t) \in [a,b] \times [c,d]$ and nonlinear boundary conditions 
\begin{equation*}
z(s,c) = f(s), \quad z(a,t) = g(t)
\end{equation*}
such that $f,g$ are differentiable functions, is given by 
\begin{align*}
    z(s,t) = z(a,c) k^{x,y}(s-a,t-c) &+ \int_{a}^{s} k^{x,y}(s-r,t-c) f'(r)dr \\
    &+ \int_{c}^{t} k^{x,y}(s-a,t-r) g'(r)dr.
\end{align*}
\item How does the analysis change for $\phi$-signature kernel $k_\phi^{x,y}$ with $\phi(k) = k!$
\end{enumerate}
\end{exercise}

\begin{exercise}\label{exercise:sigker_rescaling}
    Let $\phi : \mathbb{N} \cup \{0\} \to \mathbb{R}$ be such that $|\phi|$ satisfies the condition of \Cref{lemma:condition_phi}. Let $\theta \in \R$ and let  $\theta\phi : \mathbb{N} \cup \{0\} \to \mathbb{R}$ be a weight function such that $(\theta\phi)(k) = \theta^k\phi(k)$. Show that for any $x,y \in C_1([a,b],V)$ and for every $(s,t) \in [a,b]^2$ the following relations hold 
    \begin{equation*}
        k^{x,y}_{\theta \phi}(s,t) = k^{\theta x,y}_{\phi}(s,t) = k^{x,\theta y}_{\phi}(s,t)
    \end{equation*}
    where the path $\theta x \in C_1([a,b],V)$ is  obtained by pointwise multiplication of the path $x$ by the scalar $\theta$. 
\end{exercise}

\begin{exercise}\label{ex:exp}
    Show that if $\phi(k) = (k/2)!$ then the $\phi$-signature kernel satisfies  equation (\ref{eqn:average_PDE}) where $\pi \sim \text{Exp}(1)$  is an exponential random variable.
\end{exercise}

\begin{exercise}\label{ex:2d_pde}
    Suppose that $x,y \in C_1(V)$ are differentiable and let $$f_\omega(s,t):=\text{Re}[k^{\exp(i\omega)x,y}(s,t)] \quad \text{and} \quad g_\omega(s,t):=\text{Im}[k^{\exp(i\omega)x,y}(s,t)].$$
    Show that $(f_\omega, g_\omega)$ solves a two-dimensional PDE.
\end{exercise}

\begin{exercise}\label{ex:cont_emb}
    Let $\phi$ be a weight function satisfying the conditions of Lemma \ref{lemma:condition_phi}. Assume that $\mathcal{C}$ is equipped with a topology with respect to which the ST $S:\mathcal{C} \rightarrow T_{\phi}\left( (V)\right)$ is continuous. Let $\mathcal{K} \subset \mathcal{C}$ be a compact set. Show that the RKHS $\mathcal{H}_\phi|_{\mathcal{K}} \hookrightarrow C(\mathcal{K})$ is continuously embedded in $C(\mathcal{K})$, i.e. that $\mathcal{H}_\phi|_{\mathcal{K}}  \subset C(\mathcal{K})$ and that the topology of $\mathcal{H}_\phi|_{\mathcal{K}} $ is stronger than the topology induced by $C(\mathcal{K})$.
\end{exercise}

\begin{exercise}\label{ex:expected_sig_continued}
    This is a continuation of \Cref{ex:expected_sig}. Let $\phi : \mathbb{N} \cup \{0\} \to \mathbb{R}_+$ be such that $\phi(k) = (k/2)!$. Let $(T_\phi(\mathbb{R}^d), \left\langle \cdot, \cdot \right\rangle_\phi)$ be the completion of the tensor algebra $T(\mathbb{R}^d)$ with respect to the inner product $\left\langle \cdot, \cdot \right\rangle_\phi$.
    \begin{enumerate}
        \item Prove that $\mathbb{E}\left[
    S\left( B\right)_{0,t}\right] \in T_\phi(\mathbb{R}^d)$.
    \item For any path $x \in C_1([0,1],\mathbb{R}^d)$ derive an expression for $\left\langle \mathbb{E}\left[
    S\left( B\right)_{0,1}\right], S(x)_{0,1} \right\rangle_\phi$. 
    \item Let $\{x^1,...,x^n\} \subset (C_1([0,1],\mathbb{R}^d))^n$ be a collection of paths of with distinct STs. Let $C_n$ be the set of probability measures supported on the set $\{x^1,...,x^n\}$. Show that the following optimisation has a unique minimiser
    \begin{equation*}
        \mu^* = \argmin_{\mu \in C_n} \norm{\mathbb{E}\left[
           S(B)_{0,1}\right] - \mathbb{E}_{x \sim \mu}\left[ S(x)_{0,1}\right]}_\phi^2.
    \end{equation*}
    \end{enumerate}
\end{exercise}

\vspace{0.5cm}
\begin{exercise}\label{ex:kme_mmd}
    Let $\mu, \nu$ be two Borel probability measures on $\wC$ and let $\phi$ be a weight function satisfying the condition of \Cref{lemma:condition_phi}. Assume that 
    $$\mathbb{E}_{x \sim \mu}[\sqrt{k_\phi(x,x)}] < \infty \quad \text{and} \quad \mathbb{E}_{x \sim \nu}[\sqrt{k_\phi(x,x)}] < \infty.$$
    \begin{enumerate}
        \item Show that the $\phi$-MMD of equation (\ref{eqn:mmd}) admits the following  expression
        $$d_\phi(\mu, \nu) = \norm{M^\phi_\mu - M^\phi_\nu}_{\mathcal{H}_\phi},$$
        where $M^\phi_\mu, M^\phi_\nu \in \mathcal{H}_\phi$ are the $\phi$-kernel mean embeddings of $\mu, \nu$ respectively.
        \item Show that the squared MMD has the following expression
        \begin{equation*}
            \normalfont
            d_\phi(\mu,\nu)^2 = \mathbb{E}_{x,x' \sim \mu}[k_\phi(x,x')] - 2\mathbb{E}_{x,y \sim \mu \times \nu}[k_\phi(x,y)] + \mathbb{E}_{y,y' \sim \nu}[k_\phi(y,y')],
        \end{equation*}
        where where $x'$ (resp. $y'$) is an independent copy of $x$ (resp. $y$) with the same distribution $\mu$ (resp. $\nu$).
        \item Given $m$ sample paths $\{x^i\}_{i=1}^m \sim \mu$ and $n$ samples $\{y^j\}_{j=1}^n \sim \nu$, show that the following expression provides an unbiased estimator of the MMD         
        \begin{equation*}
            \hat{d}_\phi(\mu,\nu)^2 = \frac{1}{m(m-1)} \sum_{j\neq i} k_\phi(x^i,x^j) - \frac{2}{mn}\sum_{i,j} k_\phi(x^i,y^j)  + \frac{1}{n(n-1)}\sum_{j\neq i} k_\phi(y^i,y^j).
        \end{equation*}
    \item Show that the estimator obtained in 3. is asymptotically consistent, i.e. that it converges in probability to the squared MMD $d_\phi(\mu,\nu)^2$ as $m,n \to +\infty$.
    \end{enumerate}
\end{exercise}

\begin{exercise}\label{ex:complex_ker}
    Let $\phi : \mathbb{Z} \to \mathbb{C}$ be a complex-valued weight function and define the following kernel function
\begin{equation}\label{eqn:complex_ker}
    k^{x,y}_\phi(s,t) = \left\langle S(x)_{a,s}, S(y)_{a,t}\right\rangle_\phi := \sum_{k=-\infty}^\infty \phi(k)\big\langle S(x)_{a,s}^{|k|}, S(y)_{a,t}^{|k|} \big\rangle_{V^{\otimes |k|}}
\end{equation}
Let $x,y \in C_1([a,b],V)$ be two continous paths of bounded variation. Let $\phi : \mathbb{Z} \to \mathbb{C}$ be a complex-valued weight function such that $\{\phi(k) : k \in \mathbb{Z}\}$ are the Fourier coefficients of some bounded integrable function $f : [-\pi, \pi] \to \mathbb{C}$, i.e. 
    $$f(\omega) = \sum_{k=-\infty}^\infty \phi(k) e^{ik\omega}$$
    Show that the kernel in (\ref{eqn:complex_ker}) satisfies the following relation
    $$k^{x,y}_\phi(s,t) = \frac{1}{2 \pi} \int_{-\pi}^\pi \left( k^{\exp(-i\omega)x,y}(s,t) + k^{\exp(i\omega)x,y}(s,t) \right)f(\omega) d\omega - \phi(0)$$
\end{exercise}

\begin{exercise}
Let $E:=T_{\phi }\left( \left( V\right) \right) $ be the Hilbert space in Definition \ref{defintion: weighted spaces}
in which $\phi \equiv 1$. Consider another Hilbert space containing $T\left(
E\right) $, the tensor algebra over $E.$ To distinguish the tensor product
spaces on $E$ from those involving the underlying space $V$, we use $%
\boxtimes $ instead of $\otimes $ so that e.g. $T\left( E\right) =\oplus
_{k=0}^{\infty }E^{\boxtimes k}.$ Let $\mathcal{A}$ denote the set of all
finite indices in $\cup _{k=0}^{\infty }%
\mathbb{N}
_{0}^{k}$ and, for a given $A=\left( a_{1},...a_{k}\right) $ in $\mathcal{A}$%
, let $V^{\boxtimes A}:=\boxtimes _{i=1}^{k}V^{\otimes a_{i}}.$ Define an
inner-product on $\oplus _{A\in \mathcal{A}}V^{\boxtimes A}$ between $%
v=\sum_{A\in \mathcal{A}_{F}}v_{A}$ and $u=\sum_{A\in \mathcal{A}_{F}}u_{A}$%
, for any finite subset $\mathcal{A}_{F}\subset \mathcal{A}$, by  
\[
\left\langle v,u\right\rangle _{\mathcal{H}_{\psi }}:=\sum_{k=0}^{\infty
}\psi \left( k\right) \sum_{A\in \mathcal{A}_{F}:\left\vert A\right\vert
=k}\left\langle v_{A},u_{A}\right\rangle _{V^{\boxtimes A}},
\]%
where $\left\langle \cdot ,\cdot \right\rangle _{V^{\boxtimes A}}$ denotes
the canonical inner product on $V^{\boxtimes A}:$ for $A=\left(
a_{1},...a_{k}\right) \in 
\mathbb{N}
_{0}^{k}$%
\[
\left\langle v_{a_{i}}\boxtimes ...\boxtimes v_{a_{k}},u_{a_{i}}\boxtimes
...\boxtimes u_{a_{k}}\right\rangle =\prod_{i=1}^{k}\left\langle
v_{a_{i}},u_{a_{i}}\right\rangle _{E}\text{.}
\]%
Let $\mathcal{H}_{\psi }$ be the completion of this space. 

\begin{enumerate}[(i)]
\item Suppose  $\psi :%
\mathbb{N}
_{0}^{k}\rightarrow \mathbb{R}_{+}.$ Show that the linear map $F:T\left(
E\right) \longrightarrow \mathcal{H}_{\psi }$determined on $E^{\boxtimes k}$
by    
\[
F\left( v_{1}\boxtimes ...\boxtimes v_{k}\right)
=\sum_{r_{1},...,r_{k}=0}^{\infty }v_{1}^{r_{1}}\boxtimes ...\boxtimes
v_{k}^{r_{k}},\text{ for }v_{i}=\sum_{r=0}^{\infty }v_{i}^{r}
\]%
is well defined and satisfies  
\[
\left\langle F\left( v\right) ,F\left( w\right) \right\rangle _{\mathcal{H}%
_{\psi }}=\left\langle v,w\right\rangle _{T_{\psi }\left( \left( E\right)
\right) }.
\]%
Hence deduce that $F$ is an isomorphism between the Hilbert spaces $T_{\psi
}\left( \left( V\right) \right) $ and $\mathcal{H}_{\psi }$ 

\item Suppose that $\exp _{\boxtimes }:T\left( \left( E\right) \right)
\rightarrow T\left( \left( E\right) \right) $ denotes the exponential map
defined using $\boxtimes$, cf. (\ref{eqn:tensor_exp_}). Let $\phi :$ $\mathcal{C\rightarrow }$ $%
T\left( \left( E\right) \right) $ be defined by $\phi =F\circ \exp
_{\boxtimes }\circ S$ . Prove that $\phi \left( \mathcal{C}\right) \subset 
\mathcal{H}_{\psi }$ and, for the choice of weight function $\psi \left(
k\right) =\frac{1}{k!}$ for $k$ in $%
\mathbb{N}
_{0},$ show that 
\[
\left\langle \phi \left( \gamma \right) ,\phi \left( \sigma \right)
\right\rangle _{\psi }=\exp \left( \left\langle S\left( \gamma \right)
,S\left( \sigma \right) \right\rangle _{T_{\phi }\left( \left( V\right)
\right) }\right) 
\]

\item Let $k:\mathcal{C\times }\mathcal{C\rightarrow 
\mathbb{R}
}_{+}$ denote the kernel,%
\[
k\left( \gamma ,\sigma \right) =\exp \left( -\frac{1}{2\sigma ^{2}}%
\left\vert \left\vert S\left( \gamma \right) -S\left( \sigma \right)
\right\vert \right\vert _{T_{\phi }\left( \left( V\right) \right)
}^{2}\right) .
\]%
Using the same choice of $\psi $ as previously, show that $k$ is associated
to the feature map $\phi :$ $\mathcal{C\rightarrow H}_{\psi }$ defined by $%
\phi \left( \gamma \right) =\sum_{A\in \mathcal{A}}\phi \left( \gamma
\right) _{A}$where 
\[
\phi \left( \gamma \right) _{A}=\frac{\exp \left( -\frac{1}{2\sigma ^{2}}%
\left\vert \left\vert S\left( \gamma \right) \right\vert \right\vert
_{T_{\phi }\left( \left( V\right) \right) }^{2}\right) }{\sigma ^{\left\vert
A\right\vert }\sqrt{\left\vert A\right\vert !}}\boxtimes _{i\in
A}S\left( \gamma \right)^{(i)} 
\]
\end{enumerate}
\end{exercise}

\begin{exercise}\label{ex:existence_uniqueness_RT}
Show that if the functional $f \mapsto \mathcal{L}(f) + \lambda \norm{f}_{H_\phi}^2$ in \Cref{thm:representer} is continuous and convex from $H_\phi$ to $\R$, then the optimisation problem (\ref{eqn:repr_functional}) has a unique solution. 
\end{exercise}

\begin{exercise}[* \cite{cirone2023neural}]
    Let $x : [0,1] \to \R^d$ be continuously differentiable. For any $N \in \mathbb N$ consider the randomised signature process $S^N(x) : [0,1] \to \R^N$ defined as the solution of the following linear CDE
    \[
       S^{N}_t(x) = S^N_0 + \frac{1}{\sqrt{N}}\sum_{k=1}^d \int_0^t  A_k S^{N}_{\tau}(x) dx_{\tau}^k
    \]
    where $A_1,...,A_d \in \R^{N \times N}$ are random matrices with iid entries sampled from a standard normal distribution $[A_{k}]_{i,j} \sim \mathcal{N}(0,\frac{1}{N})$ for $k=1,...,d$ and initial condition $[S_0^N]_{i}\sim \mathcal{N}(0,1)$, for  $i,j=1,...,N$.
    \begin{enumerate}
        \item Show that for any $s,t \in [0,1]$ 
        \begin{equation*}
            \lim_{N \to \infty} \mathbb E\Big[
        \frac{1}{N}\left\langle S^N_s(x), S^N_t(y) \right\rangle_{\R^N}\Big] = 
        k^{x,y}_1(s,t)
        \end{equation*}
        where $k^{x,y}_1$ is the original signature kernel.
        \item Fix $m \in \mathbb N$ continuously differentiable paths $x_1, \dots, x_m : [0,1] \to \R^d$. Let $\phi \in \R^N$ be a random vector with iid entries sampled from $\mathcal{N}(0,\frac{1}{N})$ and define the random $\Phi^N \in \R^N$ with entries
        \[
            \Phi^N_i := \left\langle \phi, S^N_1(x_i)\right\rangle_{\R^N}.
        \]
        Show that $\Phi^N$ converge in distribution to the mulivariate normal $\mathcal{N}(0,\mathbf K)$ with covariance matrix $\mathbf K \in \R^{m \times m}$ with entries
        \begin{equation*}
            \mathbf K_{i,j} = k^{x_i,x_j}_1(1,1).
        \end{equation*}
        for any $i,j=1,...,m$
    \end{enumerate}
For an extension to the non-Gaussian case and other classes of random matrix ensembles, the reader can consult \cite{cass2024free}.
\end{exercise}

\newpage

%% file: chapter3.tex
\chapter{Rough paths and deep learning}

The symbiosis of differential equations and deep learning has become a topic of great interest in recent years. In particular, \textit{neural differential equations} (NDEs)
demonstrate that neural networks and differential equation are two sides
of the same coin \cite{chen2018neural, kidger2020neural, morrill2021neural, salvi2022neural, cirone2023neural, hoglund2023neural}. Traditional parameterised differential equations are a
special case. Many popular neural networks, such as residual and recurrent networks, are discretisations of differential equations. There are many types of NDEs. Rough path theory offers a way of analysing them under the same theoretical frameworks. For a brief account of recent applications of rough path theory to machine learning we refer the reader to \cite{fermanian2023new}.

\section{Rough differential equations (RDEs)}

We begin this section by recalling what we mean by solution of a controlled differential equation (CDE) driven by a continuous path of bounded variation. Several definitions are required. Recall that a collection of continuous vector fields  $f=\left(f_1, \ldots, f_d\right)$ on $\mathbb{R}^e$ can be viewed as a map
$$
f: b \in \mathbb{R}^e \mapsto\left\{a=\left(a_1, \ldots, a_d\right) \in \mathbb R^d \mapsto \sum_{i=1}^d f_i(b) a_i\right\} \in \mathcal L\left(\mathbb{R}^d, \mathbb{R}^e\right),
$$
where we equip $\mathcal L\left(\mathbb{R}^d, \mathbb{R}^e\right)$ with the operator norm, defined as follows
$$
\|f(b)\|:=\sup_{a \in \mathbb{R}^d:\|a\|=1}\left\|\sum_{i=1}^d f_i(b) a_i\right\| .
$$
Let $x:[0,T] \to \mathbb R^d$ be a continuous path of bounded variation, and consider the CDE
\begin{equation}\label{eqn:CDE}
    dy_t = f(y_t)dx_t \equiv \sum_{i=1}^df_i(y_t)dx^i_t, \quad y_0 \in \mathbb R^e,
\end{equation}
understood as a Riemann-Stieltjes integral equation
$$
y_t = y_0 + \int_0^T f(y_s)dx_s.
$$
A collection of vector fields $f: \mathbb{R}^e \rightarrow \mathcal L\left(\mathbb{R}^d, \mathbb{R}^e\right)$ is called \emph{bounded} if
$$
\|f\|_{\infty}:=\sup_{b \in \mathbb{R}^e}\|f(b)\|<+\infty.
$$
For any $U \subset \mathbb{R}^e$ we define the \emph{1-Lipschitz norm} (in the sense of Stein) by
$$
\|f\|_{\operatorname{Lip}^1(U)}:=\max \left\{\sup _{b, c \in U: b \neq c} \frac{\|f(b)-f(c)\|}{\|b-c\|}, \sup_{b \in U}\|f(b)\|\right\}.
$$
We say that $f \in \operatorname{Lip}^1\left(\mathbb{R}^e\right)$ if $\|f\|_{\operatorname{Lip}^1} := \|f\|_{\operatorname{Lip}^1\left(\mathbb{R}^e\right)}<+\infty$ and $f \in \operatorname{Lip}^1_{loc}\left(\mathbb R^e\right)$ if $\|f\|_{\operatorname{Lip}^1(U)}<\infty$ for all bounded subsets $U \subset \mathbb{R}^e$. The concept of Lip$^1$ regularity will later be generalized to $\operatorname{Lip}^\gamma$ in the sense of Stein, for any $\gamma \geq 1$. Note that 1-Lipschitz functions are Lipschitz continuous functions that are also bounded.

Existence and uniqueness results for solutions of the CDE (\ref{eqn:CDE}) are classical and not the main objective of this book. For completeness, we provide only a summary here, and refer the interested reader to \cite[Chapter 3]{friz2010multidimensional} for details about the proofs.

\begin{itemize}
    \item Global existence is guaranteed if the vector fields $f$ are continuous and bounded \cite[Theorem 3.4]{friz2010multidimensional}.
    \item If one only assumes continuity of the vector fields then existence holds up to an explosion time, i.e. there exists either a global solution or there exists $\tau \in [0,T]$ and a solution $y:[0,\tau) \to \mathbb R^e$ such that $y$ is a solution on $[0,t]$ for any $t 
    \in (0,\tau)$ and $\lim_{t \to \tau} |y_t| = +\infty$ \cite[Theorem 3.6]{friz2010multidimensional}. 
    \item If in addition one assumes a linear growth condition of the form $\|f_i(b)\| \leq A(1 + \norm{b})$ for some $A \geq 0$ and for all $y \in \mathbb R^e$, then it can be shown that explosion cannot happen \cite[Theorem 3.7]{friz2010multidimensional}.
    \item If the vector fields are  $1$-Lipschitz, then a unique global solution exists \cite[Theorem 3.8]{friz2010multidimensional}, while if they are only locally $1$-Lipschitz then a unique solution exists up to a possible explosion time \cite[Corollary 3.9]{friz2010multidimensional}. Explosion can be ruled out also under an additional linear-growth condition similar to the above.
\end{itemize}

In addition, by \cite[Theorem 3.19]{friz2010multidimensional} the map $(y_0, f, x) \mapsto y$ is Lipschitz continuous from $\mathbb R^e \times (\text{Lip}^1(\mathbb R^e, \mathcal{L}(\mathbb R^d, \mathbb R^e)), \norm{\cdot}_\infty) \times (C_1([0,T], \mathbb R^d), \norm{\cdot}_{1\text{-var}})$ to $(C_1([0,T], \mathbb R^e), \norm{\cdot}_{1\text{-var}})$
with explicit bound
\begin{equation*}
    \norm{y^1 - y^2}_{1\text{-var}} \leq 2\left(v\ell\norm{y^1_0 - y^2_0} + v\norm{x^1 - x^2}_{1\text{-var}} + \ell \norm{f^1 - f^2}_\infty \right) e^{3v\ell}
\end{equation*}
where 
\begin{equation*}
    \max_{i=1,2}\left\{\norm{f^i}_{\text{Lip}^1}\right\} \leq v, \quad \max_{i=1,2}\left\{\norm{x^i}_{1\text{-var}}\right\} \leq \ell.
\end{equation*}
and 
\begin{equation*}
    dy_t^i = f^i(y^i_t)dx^i_t, \quad \text{started at }  \ y^i_0 \in \mathbb R^e, \quad \text{for } i=1,2.
\end{equation*}

\subsection{Rough paths}

We denote by $\lfloor p \rfloor$ the largest integer which is less than or equal to p, and by $\Delta_T$  the following set
\begin{equation}
    \Delta_T = \{(s,t) \in [0,T]^2 : s \leq t \}. 
\end{equation}

\begin{definition}[\textbf{Control}]\label{def:control}
A continuous control is a continuous non-negative function $\omega : \Delta_T \to [0, + \infty)$ which is super-additive in the sense that
\begin{equation}
    \omega(s,t) + \omega(t,u) \leq \omega(s,u), \ \ \ \forall 0 \leq s \leq t \leq u \leq T
\end{equation}
and for which $\omega(t,t)=0$ for all $t \in [0,T]$.
\end{definition}

\begin{definition}[\textbf{Rough path}]\label{def:rough_path}
Let $p \geq 1$ and let $\omega$ be a control. A $p$-rough path over $\mathbb{R}^d$ controlled by $\omega$ is a continuous map $\mathbf x : \Delta_T \to T^{\lfloor p \rfloor}(\mathbb R^d)$ satisfying the following two conditions:
\begin{enumerate}
    \item $\mathbf x_{s,u}\otimes \mathbf x_{u,t} = \mathbf x_{s,t}$ for all $0 \leq s \leq u \leq t \leq T$;
    \item it has finite $p$-variation on $\Delta_T$ controlled by $\omega$, in the sense
    \begin{equation}\label{eqn:control_bound}
    \|\pi_i\textbf{x}_{s,t}\|_{(\mathbb{R}^d)^{\otimes i}} \leq C\omega(s,t)^{i/p}, \ \ \forall (s,t) \in \Delta_T, \ \forall i=1,...,\lfloor p \rfloor
    \end{equation}
    where $C>0$ is a  constant that does not depends on $\mathbf{x}$.
\end{enumerate}
 We denote the space of $p$-rough paths over $\mathbb R^d$ by $\Omega_p(\mathbb R^d)$.
\end{definition}

Hence, a $p$-rough path is a continuous map from $\Delta_T$ to $T^{\lfloor p \rfloor}(\mathbb R^d)$ that satisfies the algebraic condition 1. called \emph{(multiplicativity)} and the analytic condition 2.

\begin{definition}[\textbf{$p$-variation metric}]\label{def:p_var_metric}
    The homogeneous $p$-variation metric of two $p$-rough paths $\mathbf{x}, \mathbf{y} : \Delta_T \to T^{\lfloor p \rfloor}(\mathbb R^d)$ is defined as follows
    \begin{equation}
    d_p(\mathbf{x}, \mathbf{y}) = \max_{1 \leq i \leq \lfloor p \rfloor}\sup_{\mathcal{D} \subset 
    [0,T]}\left(\sum_{t_k \in \mathcal{D}}\norm{\pi_i\mathbf{x}_{t_k, t_{k+1}} - \pi_i\mathbf{y}_{t_k, t_{k+1}}}^{p/i}_{(\mathbb R^d)^{\otimes i}}\right)^{1/p}
    \end{equation}
    where the supremum is taken over all partitions $\mathcal{D}$ of the interval $[0,T]$.
\end{definition}

It can be observed from Definition \ref{def:rough_path} that the less regular the rough path $\mathbf{x}$ is the more terms are needed to describe it in addition to the underlying path increments $\mathbf{x}^{(1)}_{s,t}$. The next theorem is an important result in rough path theory. It states that a $p$-rough path can be extended uniquely to a $q$-rough path if $p \leq q$. 

\begin{theorem}[\textbf{Extension Theorem}]\cite[Theorems 3.7]{lyons2007differential}\label{thm:extension_theorem}
Let $\mathbf{x}$ be a $p$-rough path controlled by a control $\omega$. Let $q \geq p$. Then, there exists a unique extension of $\mathbf{x}$ denoted by
$$S^{\lfloor q \rfloor}(\mathbf x): \Delta_T \to T^{\lfloor q \rfloor}(\mathbb R^d)$$ 
such that
\begin{enumerate}
    \item $S^{\lfloor q \rfloor}(\mathbf x)_{s,u} \otimes S^{\lfloor q \rfloor}(\mathbf x)_{u,t} = S^{\lfloor q \rfloor}(\mathbf x)_{s,t}$ for all $0 \leq s \leq u \leq t \leq T$;
    \item $\| \pi_iS^{\lfloor q \rfloor}(\mathbf x)_{s,t} \|^{i}_{V^{\otimes i}} \leq C \omega(s,t)^{i/p}$, for some $C>0$, and all $(s,t) \in \Delta_T$, $i \leq \lfloor q \rfloor$;
     \item $\pi_{\leq \lfloor p \rfloor} \mathbf x_{s,t} = \pi_{\leq \lfloor p \rfloor} S^{\lfloor q \rfloor}(\mathbf x)_{s,t}$ for all $(s,t) \in \Delta_T$.
\end{enumerate}
\end{theorem}

The solution theory for rough differential equations originally developed in \cite{lyons1998differential} makes use of the following subset of rough paths.

\begin{definition}[\textbf{Geometric $p$-rough path}]\label{def:geom_rough_path}
A geometric $p$-rough path is a $p$-rough path  expressed as the limit in $p$-variation of a sequence $(\pi_{\leq \lfloor p \rfloor} \circ S(x^n))$ of truncated signatures of bounded variation paths $(x^n)$. We denote by $G \Omega_p = G \Omega_p (\mathbb R^d)$ the space of geometric $p$-rough paths on $\mathbb R^d$. In other words, $G\Omega_p$ is the closure of $\{\pi_{\leq \lfloor p \rfloor} \circ S(x) \mid x \in \Omega_1\}$ in $p$-variation. 
\end{definition}


The extension map in \Cref{thm:extension_theorem} induces a map $(G \Omega_p(\mathbb R^d), d_p) \to (G \Omega_q(\mathbb R^d), d_q)$ sending a geometric $p$-rough path $\mathbf{x}$ to the geometric $q$-rough path $S^{\lfloor q \rfloor}(\mathbf{x})$ which is continuous in $p$-variation; see \cite[Theorems 3.10]{lyons2007differential} for a proof. In the special case where $p \in [1,2)$, the extension map is just the ST.

Let $f : \mathbb R^e \to \mathcal{L}(\mathbb R^d,\mathbb R^e)$ be smooth and recall from Chapter $1$ that this map can be also seen as a linear map from $\mathbb R^d$ into the space $C^{\infty}(\mathbb R^e)$ of smooth vector fields over $\mathbb R^e$. In what follows we will give meaning to the \emph{rough differential equation} (RDE)
\begin{equation}\label{eqn:RDE}
    dy_t = f(y_t)d\mathbf{x}_t, \quad y_0 = a \in \mathbb R^e
\end{equation}
driven by a geometric rough path $\mathbf x$. In addition to Lyons’ original work \cite{lyons1998differential}, various approaches to solve RDEs have been introduced in the literature. They can be split into methods that define solutions as either
\begin{enumerate}
    \item fixed points of Picard iterations \cite{lyons1998differential, lyons2007differential, gubinelli2004controlling, friz2020course}, or
    \item limit of solutions of ODEs \cite{friz2010multidimensional, bailleul2015flows} or discrete approximations \cite{davie2007differential}.
\end{enumerate}

The precise treatment of these different appraches to RDEs is beyond the scope of this book. We will discuss the first approach informally, while the second approach will be presented in more details. Our choice is justified by the fact that it will be more convenient to use the second approach to develop efficient backpropagation procedures for neural RDEs discussed in the next section.

\paragraph{RDE solutions as fixed points of Picard iterations} The basic idea of \cite{lyons1998differential, lyons2007differential} is to transform the RDE (\ref{eqn:RDE}) into an integral equation of the form
\begin{equation}\label{eqn:rough_integral}
    \mathbf{z} = \int h(\mathbf{z})d\mathbf{x}
\end{equation}
where $\mathbf{z}$ is a rough path over $\mathbb R^d \oplus \mathbb R^e$ and $h$ is given by
\begin{equation*}
    h(x_1,y_1)(x_2,y_2) = (x_2, f(y_1 + a)x_2)
\end{equation*}
One can then use the theory of rough integration to give meaning to the \emph{rough integral} in (\ref{eqn:rough_integral}). A solution of the RDE (\ref{eqn:RDE}) is then defined to be a rough path $\mathbf{z}$ which satisfies (\ref{eqn:rough_integral}) and whose projection is a rough path over $\mathbb R^d$ that coincides with $\mathbf{x}$. In particular, the solution is itself a rough path. Thus, one can consider the projection of $\mathbf{z}$ to $\mathbb R^e$ in
order to obtain an ordinary solution path $y$ in $\mathbb R^e$. The existence of solutions is shown using a suitably adapted Picard iteration argument.

More recently, \cite{gubinelli2004controlling} introduced the concept of \emph{controlled rough paths}.  Controlled rough paths are defined with respect to a reference rough path. In the context of RDEs this reference path is given by the driving rough path $\mathbf{x}$. Let us sketch the main idea for the case $p \in[2,3)$; this case was treated fully in \cite{friz2020course}. Thus, assume that $\mathbf{x}$ is a $p$-rough path for $p \in[2,3)$ controlled by $\omega$. A continuous path $y:[0, T] \rightarrow \mathbb R^e$ of finite $p$-variation is said to be controlled by $\mathbf{x}$ if there exists a continuous path $y^{\prime}:[0, T] \rightarrow \mathcal{L}(\mathbb R^d, \mathbb R^e)$ of finite $p$-variation also controlled by $\omega$, a remainder $R^y: \Delta_T \to \mathbb R^e$ and a constant $C>0$ such that
\begin{equation}\label{eqn:controlled_rough_path}
    y_t-y_s=y_s^{\prime} \pi_1 \mathbf{x}_{s, t} + R_{s, t}^y \quad \text { and } \quad\left\|R_{s, t}^y\right\| \leq C\omega(s,t)^{2 / p}
\end{equation}
hold for all $(s, t) \in \Delta_T$. One can interpret (\ref{eqn:controlled_rough_path}) as a generalisation of Taylor's theorem where the path $y$ is not approximated by polynomials but by the homogeneous terms of the rough path $\mathbf{x}$. One can then define integrals of controlled rough paths (and functions thereof) with respect to the reference path $\mathbf{x}$. Therefore, we can rewrite the RDE (\ref{eqn:RDE}) in integral form
\begin{equation}\label{eqn:integral_rde}
    y_t=a + \int_0^T f\left(y_s\right) \mathrm{d} \mathbf{x}_s
\end{equation}
and treat this equation as a fixed point equation in the space of controlled rough paths and apply a Picard iteration to obtain a solution. It turns out \cite[Theorem 8.4]{friz2010multidimensional} that if $y$ solves (\ref{eqn:integral_rde}), then $y^{\prime}$ is given by $y_s^{\prime}=f\left(y_s\right)$. The concept of controlled rough paths has been further extended to signals evolving in time and space in \emph{regularity structures} \cite{hairer2014theory}, which provides a solution theory to a large class of singular partial differential equations.

\subsection{RDE solutions as limits of approximations}

The main idea here is to approximate the rough path $\mathbf{x}$ by limit points of sequences of smooth paths in $p$-variation. The solution of the original RDE can then be defined as the uniform limit of the smooth CDE solutions in the sense of \Cref{thm:ULT} below. Before stating it we recall the definition of $\gamma$-Lipschitz function (in the sense of Stein).

\begin{definition}[\textbf{Lipschitz function}]
Let $V,W$ be two normed space and let $\gamma>0$. A function $g : V \to W$ is called $\gamma$-Lipschitz if $g$ is $\lfloor \gamma \rfloor$ times continuously differentiable and such that there exists a constant $M \geq 0$ such that the supremum norm of its $k^{th}$ derivative, $k=0,...,\lfloor \gamma \rfloor$, and the $(\gamma - \lfloor \gamma \rfloor)$-H\"older norm of its $\lfloor \gamma \rfloor^{th}$ derivative are bounded by $M$. The smallest $M$ satisfying these conditions is the $\gamma$-Lipschitz norm of $g$, denoted by $\norm{g}_{Lip^\gamma}:=\norm{g}_{Lip^\gamma(V, W)}$. We denote by $\text{Lip}^\gamma(V,W)$ the space of $\gamma$-Lipschitz functions from $V$ to $W$. 
\end{definition}

This definition applies in particular to the context of RDEs of the form (\ref{eqn:RDE}), where $V = \mathbb R^e$, $W = \mathcal{L}(\mathbb R^d, \mathbb R^e)$ with operator norm. Saying that the collections of vector fields $f = (f_1,...,f_d) \in \text{Lip}^\gamma(\mathbb R^e, \mathcal L(\mathbb R^d, \mathbb R^e))$ is equivalent to $f_1,...,f_d \in \text{Lip}^\gamma(\mathbb R^e,\mathbb R^e)$.

We now make precise the meaning of the RDE (\ref{eqn:RDE}) where $f : \mathbb R^e \to \mathcal{L}(\mathbb R^d, \mathbb R^e)$ is a family of sufficiently nice vector fields and $\mathbf x$ is a geometric $p$-rough path. The following is essentially
an existence and uniqueness result for such RDEs. Our precise definition of RDE solution will follow as it is based entirely on this result.

\begin{theorem}[\textbf{Universal Limit Theorem}]\cite[Theorems 10.29 and 10.57]{friz2010multidimensional}\label{thm:ULT}
Let $p \geq 1$, let $\mathbf{x} \in G\Omega_p(V)$ be a geometric $p$-rough path, and let $f : \mathbb R^e \to \mathcal L(\mathbb R^d, \mathbb R^e)$ be either linear or $\gamma$-Lipschitz with $\gamma > p$. Let $(x^N)$ be a sequence of continuous bounded variation paths such that the sequence $(\pi_{\leq \lfloor p \rfloor} \circ S(x^N))$ converges in $p$-variation to $\mathbf{x}$. Then, there exists $y \in C^{p-var}([0,T],\mathbb R^e)$ such that the sequence $(y^N)$ defined by
\begin{equation}\label{eqn:seq_cdes}
    dy^N_t = f(y^N_t)dx^N_t, \quad y_0^N = y_0 \in \mathbb R^e.
\end{equation}
converges uniformly to $y$ on $[0,T]$. Furthermore, the  map $(y_0, f, \mathbf x) \mapsto y$ is (at least) locally Lipschitz continuous in all its parameters.
\end{theorem}


\begin{remark}
    We note that to only ensure existence of solutions the regularity assumption on $f$ can be relaxed to $f \in \text{Lip}^{\gamma-1}(\mathbb R^e, \mathcal L(\mathbb R^d, \mathbb R^e))$ as proved in \cite[Theorem 10.16]{friz2010multidimensional}. Also note that as for CDEs driven by continuous bounded variation paths, if the vector fields are only locally $\gamma$-Lipschitz, i.e. $f \in \text{Lip}^{\gamma-1}_{loc}(\mathbb R^e, \mathcal L(\mathbb R^d, \mathbb R^e))$, then a (local) solution might exist only up to an explosion time.
\end{remark}

\Cref{thm:ULT} yields the following notion of solution to a RDE.

 \begin{definition}[\textbf{RDE solution}]\label{def:RDE_sol}
     Let $\mathbf x \in G\Omega_p(\mathbb R^d)$ be a geometric $p$-rough path. We say that $y \in C^{p-var}([0,T],\mathbb R^e)$ is a solution to the RDE
     \begin{equation*}
        dy_t = f(y_t)d\mathbf{x}_t, \quad y_0 = a \in \mathbb R^e
    \end{equation*}
     if $y$ belongs to the set of (uniform) limit points constructed in \Cref{thm:ULT}. In particular, if $f : \mathbb R^e \to \mathcal L(\mathbb R^d, \mathbb R^e)$ is linear or $\gamma$-Lipschitz with $\gamma > p$, then $y$ is unique.
 \end{definition}

 \begin{remark}\label{rk:full_rde}
     The notion of RDE solution presented in \Cref{def:RDE_sol} maps a geometric $p$-rough path to a $\mathbb R^e$-valued continuous path of finite $p$-variation. However, it might be desirable to construct a "full" solution also as a geometric rough path. This is the case, for example, if one is interested in using a solution to a first RDE to be the driving signal for a second RDE. We will need to consider such scenario in the proof of \Cref{thm:RDE_adjoint}. More precisely, we will say that $\mathbf y \in G\Omega(\mathbb R^e)$ is the (full) solution to the RDE
     \begin{equation}\label{eqn:RDE_full}
         d \mathbf y_t = f(\mathbf y_t)d\mathbf x_t, \quad \text{started at } \ \mathbf y_0 \in \pi_{\leq \lfloor p \rfloor} \circ S(\Omega_1(\mathbb R^e))
     \end{equation}
     if there exists a sequence $\{x^N\}$ of continuous bounded variation paths such that the sequence $(\pi_{\leq \lfloor p \rfloor} \circ S(x^N))$ converges in $p$-variation to $\mathbf{x}$ and such that the sequence $(\mathbf y_0 \cdot \pi_{\leq \lfloor p \rfloor} \circ S (y^N))$ converges uniformly on $[0,T]$ to $\mathbf y$ as $N \to \infty$, where $\{y^N\}$ are the solutions to the CDEs (\ref{eqn:seq_cdes}), with $y_0 = \pi_1 \mathbf y_0$.
     
     As remarked in \cite[Theorem 10.38]{friz2010multidimensional}, (full) RDE solutions are simply RDE solutions in the sense of \Cref{def:RDE_sol} but along modified vector fields as
     \begin{align*}
         d \mathbf y_t = d\left( \mathbf y_0 \cdot \pi_{\leq \lfloor p \rfloor} \circ S(y)_{0,t} \right) = \mathbf y_0 \cdot \pi_{\leq \lfloor p \rfloor} \circ S(y)_{0,t} \cdot dy_u = \mathbf y_t \cdot f(\pi_1 \mathbf y_t) d\mathbf x_t.
     \end{align*}
 \end{remark}


\subsection{Stratonovich SDEs as RDEs}

In this section we consider an important instance of a rough path, namely the one associated with Brownian motion. We let $(\Omega, \mathcal{F}, \mathbb{P}; \{\mathcal{F}\}_{t \geq 0})$ be a filtered probability space containing a $d$-dimensional Brownian motion. 

\begin{theorem}[\textbf{Brownian motion as a rough path}]\cite[Corollary 13.22]{friz2010multidimensional}\label{thm:BM_RP}
Let $W$ be a standard $d$-dimensional
Brownian motion and $B^N$ be the piecewise linear path with $N$ pieces that coincides with $W$ on the uniform partition $\mathcal{D}_N:=\{0=t_0<t_1<....<t_N=1\}$, with $t_k=kh$ and step-size $h = 1/N$. Then, there exists a random geometric $p$-rough path $\mathbf{W}$ with $p\in[2,3)$ such that for almost all $\omega \in \Omega$
\begin{equation}
    d_p(S^2(W^N)(\omega),\mathbf{W}(\omega)) \to 0 \quad \text{ as } \quad N \to \infty
\end{equation}
and
\begin{equation}
    \mathbf{W}(\omega)_{s,t} = \left(1,(W_t - W_s)(\omega), \left(\int_s^T(W_u - W_s)\otimes \circ \ dW_u\right)(\omega)\right)
\end{equation}
where the stochastic integral is in the Stratonovich sense. The geometric rough path $\mathbf{W}$ is often referred to as Stratonovich enhanced Brownian motion. 
\end{theorem}

Note that the piecewise linear interpolation in \Cref{thm:BM_RP} can be replaced by other kinds of approximation and the same convergence holds. Also, it is not necessary for the partition to be uniform; other examples include dyadic \cite[Proposition 3.6]{friz2020course}. 

To put simply, this theorem states that Brownian motion can be approximated (in a rough path sense) by a sequence of bounded variation paths. This is particularly helpful within stochastic analysis as it allows one to construct pathwise solutions for SDEs governed by sufficiently regular vector fields. Importantly for us, the above theory applies directly to (Stratonovich) SDEs as Brownian motion can be viewed as a geometric $p$-rough path with $p \in (2, 3)$ by \Cref{thm:BM_RP}.

The following is a consequence of the universal limit theorem \ref{thm:ULT} for SDEs.

\begin{theorem}
    Suppose that $\mu \in Lip^1(\mathbb{R}^{e+1},\mathbb{R}^e)$ and $\sigma \in Lip^{2 + \epsilon}(\mathbb{R}^{e+1}, \mathcal L(\mathbb{R}^d,\mathbb{R}^e))$. Let $\{W^N\}_{N \geq 1}$ be a sequence of piecewise linear paths converging to the Stratonovich enhanced Brownian motion $\mathbf{W}$. Let $\{y^N\}_{N \geq 1}$ be the sequence of solutions to the following CDE
    \begin{equation}
        dy^N_t = \mu(t,y^N_t)dt + \sigma(t,y^N_t)dW^N_t, \quad y^N_0 = a \in \mathbb{R}^e.
    \end{equation}
    Then, $y^N$ converges uniformly to a process $y:[0,T]\to \mathbb{R}^e$ which coincides almost surely with the strong solution of the Stratonovich SDE
    \begin{equation}
        dy_t = \mu(t,y_t)dt + \sum_{i=1}^d\sigma_i(t,y_t)\circ dW^i_t, \quad y_0 = a.
    \end{equation}
\end{theorem}


\section{Numerical RDE solvers}

Recall the identification of a vector field $g = (g^1,...,g^e)^\top : \mathbb R^e \to  \mathbb R^e$ with the first order differential operator
\begin{equation*}
    g(y) = \sum_{j=1}^e g^j(y)\frac{\partial}{\partial y^j}.
\end{equation*}
Given a family of smooth vector fields $f = (f_1,...,f_d)$ on $\mathbb R^e$, it will be convenient in the sequel to define the algebra homomorphism $\Phi_f$ from $T((\mathbb R^d))$ into the space of differential operators generated by
\begin{equation*}
    \Phi_f(\mathbf 1) = I_e \quad \text{and} \quad \Phi_f(e_i) = f_i
\end{equation*}
where $e_1,...,e_d$ are orthogonal basis elements of $\mathbb R^d$. Thus, with this notation, for any multi-index $(i_1,...,i_k) \in \{1,...,d\}^k$ we have that
\begin{equation*}
    \Phi_f(e_1 ... e_k) = f_{i_1}...f_{i_k}.
\end{equation*}
where the right-hand-side is understood as the composition of smooth differential operators.

\subsection{Euler scheme}


A natural step-$N$ approximation for the solution of the RDE
\begin{equation*}
    dy_t = f(y_t)d\mathbf x_t, \quad y_s \in \mathbb R^e
\end{equation*}
is given by the \emph{step-$N$ Euler scheme}
$y_t \approx \mathcal{E}_{s,t}^N(y_s; f, \mathbf x)$ where
\begin{align}
    \mathcal{E}_{s,t}^N(y_s; f, \mathbf x) := \Phi_f\Big(S^N(\mathbf x)_{s,t} \Big) (y_s).
    \label{eqn:Euler_scheme}
\end{align}
The reader is invited to derive a first error estimate for the Euler scheme in \Cref{ex:Euler_estimate} for paths of bounded variation. It turns out that a similar error estimate also holds for RDEs driven by geometric rough paths as stated in the following lemma.

\begin{lemma}[\textbf{Local error of Euler scheme}]\cite[Corollary 10.17]{friz2010multidimensional}\label{lemma:local_error_Euler}
    Let $\mathbf{x} \in G\Omega(\mathbb R^d)$ be a geometric $p$-rough path and $f = (f_1,...,f_d)$ a family of $Lip^{\gamma-1}$ vector fields over $\mathbb R^e$ with $\gamma > p$. Let $y$ be the (not necessarily unique) solution to the RDE (\ref{eqn:RDE}). Then there exists a constant $C = C(p,\gamma)$ such that
    \begin{equation*}
        \| y_t - \mathcal{E}_{s,t}^{\lfloor \gamma \rfloor}(y_s; f, \mathbf x) \| \leq C\left(\|f\|_{\operatorname{Lip}^{\gamma-1}} \|\mathbf x\|_{p-var;[s,t]}\right)^\gamma.
    \end{equation*}
\end{lemma}

Iterations of the step-$N$ Euler scheme over a partition $\mathcal{D} = \{0=t_0<t_1<...<t_n=T\}$ leads to an approximate solution $y^{\text{Euler};N,\mathcal{D}}$ over the entire interval $[0,T]$ defined for any $N \geq p$ and at any $t_k \in \mathcal{D}$ as
\begin{equation*}
     y^{\text{Euler};N,\mathcal{D}}_T := \mathcal{E}^N_{t_0,t_1}(...\mathcal{E}^N_{t_1,t_2}(\mathcal{E}^N_{t_{k-1},t_k}(y_0;f,\mathbf x); f,\mathbf x);...;f,\mathbf x).
\end{equation*}

\begin{lemma}[\textbf{Global error of Euler scheme}]\cite[Theorem 10.33]{friz2010multidimensional}
    Under the assumptions as in \Cref{lemma:local_error_Euler}, there exists a constant $C = C(p,\gamma)$ such that
    \begin{equation}\label{eqn:global_error_Euler}
        \| y_T -y^{\text{Euler};\lfloor \gamma \rfloor,\mathcal{D}}_T \| \leq Ce^{C\omega(0,T)}\sum_{i=1}^n\omega(t_{i-1},t_i)^{(\lfloor \gamma \rfloor +1)/p}
    \end{equation}
    with the control $\omega(s,t) = \|f\|^p_{\operatorname{Lip}^{\gamma}} \|\mathbf x\|^p_{p-var;[s,t]}$.
\end{lemma}

\begin{remark}
    Runge-Kutta schemes for RDEs have been studied in \cite{redmann2022runge}.
\end{remark}

\subsection{The log-ODE method}



The log-ODE method is an effective method for approximating the solution of a CDE by reducing it to an ODE \cite{gaines1997variable, gyurko2008numerical, gyurko2010rough, janssen2012order, morrill2021neural}. This method is based on another path-transform, called \emph{log-signature transform} (logST) which efficiently encodes the same integral information as the ST but it removes certain polynomial redundancies, such as
\begin{equation*}
    \int_0^T\int_0^s d x_u^{i} d x_s^{j} + \int_0^T\int_0^s d x_u^{j} d x_s^{i} = x_t^{i}x_t^{j} - x_0^{i}x_0^{j}, \quad i,j\in\{1,\cdots, d\}.
\end{equation*}
To this end, recall that the extended tensor algebra $T((\mathbb R^d))$ has a well-defined notion of logarithm defined in equation (\ref{eqn:tensor_log}) as well as its truncated version. 


It is well-known that the logarithm is the inverse of the exponential map when restricted to the set of grouplike elements. The image of the latter by the logarithm map is a linear subspace of $T((\mathbb{R}^d))$ corresponding to the \emph{free Lie algebra} $\mathfrak L(\mathbb R^d)$. It is the smallest Lie subalgebra of $T((\mathbb{R}^d))$ under the Lie bracket. We denote by $\mathfrak L^N(\mathbb R^d)$ the truncated Lie algebra. A clear presentation of the background to this can be found in  \cite{reutenauer2003free} and \cite{salvi2023structure}. 

The vector space $\mathfrak L^N(\mathbb R^d)$ has dimension $\beta(m, N)$ with
\begin{equation*}
    \beta(m, N) = \sum_{k = 1}^N \frac{1}{k} \sum_{i | k} \mu\left(\frac{k}{i}\right) m^i
\end{equation*}
where $\mu$ the M{\"o}bius function. Bases of this space are known as \emph{Hall bases} \cite{reutenauer2003free, reizenstein2017calculation}. One of the most well-known bases is the \emph{Lyndon basis}, denoted by $\mathcal{B}^N$, and indexed by \emph{Lyndon words}. A Lyndon word is any word which occurs lexicographically strictly earlier than any word obtained by cyclically rotating its elements.

Note that, by freeness of $\mathfrak L^N(\mathbb R^d)$ as a Lie algebra, $\Phi_f$ maps elements of $\mathfrak L^N(\mathbb R^d)$ to first-order differential operators which can be associated with vector fields.

The image of a signature under the logarithm (or its representation in a Hall basis) is called the \emph{log-signature transform} (logST). Note that $\beta(m,N)<N$, i.e. that the truncated log-signature can store the same amount of information as the truncated signature but it requires less storage memory.

The \emph{step-$N$ log-ODE method} of the RDE (\ref{eqn:RDE}) over an interval $[s,t]$ is given by 
\begin{equation}
    y_t \approx \mathcal{L}^N_{s,t}(y_s; f, \mathbf x) := z_1
\end{equation}
where $z : [0,1] \to \mathbb R^e$ is the solution of the following ODE on $[0,1]$
\begin{align}\label{eq:standardlogode}
dz_u = F^N_{s,t}\left(z_u; f, \mathbf x\right) du, \quad z_0 = y_s
\end{align}
and where the vector field $F^N_{s,t}(\cdot \ ; f, \mathbf x) : \mathbb{R}^e\rightarrow\mathbb{R}^e$ is defined from the log-signature as
\begin{equation}\label{eq:logodevectfield}
F^N_{s,t}\left(z; f, \mathbf x\right) := \Phi_f\Big(\log^N (S(\mathbf x)_{s,t}) \Big) (z).
\end{equation}
The next theorem quantifies the approximation error of the \emph{log-ODE method} in terms of the regularity of the systems vector field $f$ and control path $\mathbf x$.

\begin{theorem}[\textbf{Local error of the log-ODE method}]\label{thm:logODEthm}
Let $p \geq 1$. Let $\mathbf{x} \in G\Omega_p(\mathbb{R}^d)$ be a geometric $p$-rough path. Let $f \in Lip^\gamma(\mathbb{R}^e, L(\mathbb{R}^d,\mathbb{R}^e))$, with $\gamma > p$. Let $y$ be the solution to the RDE (\ref{eqn:RDE}). Then, there is a constant $C = C(p,\gamma)$ such that 
\begin{equation}
\norm{y_t - \mathcal{L}^{\floor \gamma}_{s,t}(y_s; f, \mathbf x)} \leq C\left(\norm{f}_{\mathrm{Lip}^\gamma}\norm{\mathbf{x}}_{p\text{-var};[s,t]}\right)^\gamma.
\label{eq:local_logodeestimate}
\end{equation}
\end{theorem}

\begin{proof}
    
    Firstly note that the vector field $F^{\lfloor \gamma \rfloor}_{s,t}(\cdot \ ; f, \mathbf x) \in Lip^\gamma(\mathbb R^e, \mathbb R^e)$ with 
    $$\|F^{\lfloor \gamma \rfloor}_{s,t}(\cdot \ ; f, \mathbf x) \|_{\mathrm{Lip}^\gamma}  \leq \norm{f}_{\mathrm{Lip}^\gamma} \|\mathbf x\|_{p-var;[s,t]}.$$
    Secondly, the step-$\lfloor \gamma \rfloor$ Euler scheme
     defined in (\ref{eqn:Euler_scheme}) and applied to the log-ODE (\ref{eq:standardlogode}) reads
    \begin{align*}
        \mathcal{E}^{\lfloor \gamma \rfloor}_{0,1}(y_s; F^{\lfloor \gamma \rfloor}_{s,t}(\cdot \ ; f, \mathbf x), \mathbf t) = y_s + \sum_{i=1}^{\lfloor \gamma \rfloor} \frac{1}{k!} \underbrace{F^{\lfloor \gamma \rfloor}_{s,t}(\cdot 
        \ ; f, \mathbf x) ... F^{\lfloor \gamma \rfloor}_{s,t}(\cdot 
        \ ; f, \mathbf x)}_{k \text{ times}} (y_s).
    \end{align*}
    By \Cref{ex:Euler_estimate}, there exists a constant $C_1 = C_1(\gamma)$ such that
    \begin{align*}
        \|\mathcal{L}^{\lfloor \gamma \rfloor}_{s,t}(y_s; f, \mathbf x) - \mathcal{E}^{\lfloor \gamma \rfloor}_{0,1}(y_s; F^{\lfloor \gamma \rfloor}_{s,t}(\cdot \ ; f, \mathbf x), \mathbf t)\|
        &\leq C_1 \|F^{\lfloor \gamma \rfloor}_{s,t}(\cdot \ ; f, \mathbf x) \|_{Lip^\gamma}^\gamma \\
        &\leq C_1(\norm{f}_{\mathrm{Lip}^\gamma} \|\mathbf x\|_{p-var;[s,t]})^{\gamma}.
    \end{align*}
    Furthermore, by \Cref{lemma:local_error_Euler} there exists a constant $C_2 = C_2(\gamma,p)$ such that
    \begin{equation*}
        \| y_t - \mathcal{E}_{s,t}^{\lfloor \gamma \rfloor}(y_s; f, \mathbf x) \| \leq C_2\left(\|f\|_{\operatorname{Lip}^{\gamma}} \|\mathbf x\|_{p-var;[s,t]}\right)^\gamma 
    \end{equation*}
    where $\mathcal{E}_{s,t}^{\lfloor \gamma \rfloor}(y_s; f, \mathbf x)$ is the step-$\lfloor \gamma \rfloor$ Euler scheme
    applied to the RDE (\ref{eqn:RDE}). 
    
    Now note that
    \begin{align*}
        \mathcal{E}_{s,t}^{\lfloor \gamma \rfloor}(y_s; f, \mathbf x) &= \Phi_f\Big(S^{\lfloor \gamma \rfloor}(\mathbf x)_{s,t} \Big) (y_s) \\
        &= y_s + \Phi_f\Big(\sum_{k=1}^{\lfloor \gamma \rfloor} \frac{1}{k!} (\log^{\lfloor \gamma \rfloor}(S(\mathbf x)_{s,t}))^{\otimes k} \Big) (y_s) \\
        &= y_s + \sum_{k=1}^{\lfloor \gamma \rfloor} \frac{1}{k!} \underbrace{\Phi_f\big(\log^{\lfloor \gamma \rfloor}(S(\mathbf x)_{s,t})\big)...\Phi_f\big(\log^{\lfloor \gamma \rfloor}(S(\mathbf x)_{s,t})}_{k \text{ times}} \big) (y_s) \\
        &= y_s + \sum_{k=1}^{\lfloor \gamma \rfloor} \frac{1}{k!} \underbrace{F^{\lfloor \gamma \rfloor}_{s,t}\left(\cdot \ ; f, \mathbf x\right) ...F^{\lfloor \gamma \rfloor}_{s,t}\left(\cdot \ ; f, \mathbf x\right)}_{k \text{ times}} \big) (y_s) \\
        &= \mathcal{E}^{\lfloor \gamma \rfloor}_{0,1}(y_s; F^{\lfloor \gamma \rfloor}_{s,t}(\cdot \ ; f, \mathbf x), \mathbf t)
    \end{align*}
    Hence
    \begin{align*}
        \| y_t - \mathcal{L}^{\lfloor \gamma \rfloor}_{s,t}(y_s; f, \mathbf x) \| &\leq \| y_t - \mathcal{E}_{s,t}^{\lfloor \gamma \rfloor}(y_s; f, \mathbf x) \| + \|\mathcal{L}^{\lfloor \gamma \rfloor}_{s,t}(y_s; f, \mathbf x) - \mathcal{E}_{s,t}^{\lfloor \gamma \rfloor}(y_s; f, \mathbf x) \| \\
        &= (C_1 + C_2)\left(\|f\|_{\operatorname{Lip}^{\gamma}} \|\mathbf x\|_{p-var;[s,t]}\right)^\gamma.
    \end{align*}
    
\end{proof}

Note that if the vector fields $f=(f_1, \cdots, f_d)$ are linear, then $F$ is also linear, and so the log-ODE (\ref{eq:standardlogode}) also becomes linear. Therefore, the log-ODE solution exists and is explicitly given as the exponential of the matrix $F^N$, i.e. $z_u = \exp(uF^N)z_s$. 


That said, the estimate (\ref{eq:local_logodeestimate}) does not directly apply when the vector fields $\{f_i\}$ are linear as they would be unbounded. Fortunately, it is well known that linear RDEs are well posed and the growth of their solutions can be estimated. Hence, the arguments that established Theorem \ref{thm:logODEthm} will hold in the linear setting as $\|f\|_{\mathrm{Lip}^\gamma}$ would be finite when defined on the domains that the solutions $y$ and $z^{s,t}$ lie in.

Iterations of the log-ODE method over a partition $\mathcal{D} = \{0=t_0<t_1<...<t_n=T\}$ leads to an approximate solution $y^{\text{log-ODE};N,\mathcal{D}}$ over the entire interval $[0,T]$ defined for any $N \geq p$ and at any $t_k \in \mathcal{D}$ as
\begin{equation*}
     y^{\text{log-ODE};N,\mathcal{D}}_T := \mathcal{L}^N_{t_0,t_1}(...\mathcal{L}^N_{t_1,t_2}(\mathcal{L}^N_{t_{k-1},t_k}(y_0;f,\mathbf x); f,\mathbf x);...;f,\mathbf x).
\end{equation*}
Given the local error estimate (\ref{eq:local_logodeestimate}) for the log-ODE method, global error estimates similar to the one for the Euler scheme in (\ref{eqn:global_error_Euler}) were developed in \cite[Theorem 3.2.1]{gyurko2008numerical} and \cite[Theorem 5.8]{janssen2012order}. We refer to these two PhD theses for a proof of the following statement. Thus, we see that higher convergence rates can be achieved through using more terms in each log-signature. Also, as for the Euler scheme, it is unsurprising that the error estimate (\ref{eqn:global_error_Euler}) increases with the ``roughness'' of the driving rough path. Hence, the accuracy performance of the log-ODE method can be improved by choosing an appropriate step size and depth of log-signature.

\section{Differential equations and deep learning}

\subsection{Review of Deep Learning}

We begin this section by recalling some basic facts about classical neural networks. The emphasis of the first section will be on brevity over completeness and some familiarity with deep learning concepts is assumed from the reader.

\subsubsection{Universal approximation}

A \emph{feedforward neural network} $f : \mathbb{R}^m \to \mathbb{R}^n$ with $m$ neurons in the input layer, $n$ neurons in the output layer, $\ell-2$ hidden layers of $h$ neurons, and activation function $\sigma : \mathbb{R} \to \mathbb{R}$ is of the form
\begin{align*}
    z^{(0)} = x \in \mathbb{R}^m, \quad \underbrace{c^{(k)} = A^{(k)}z^{(k-1)} + b^{(k)}, \quad z^{(k)} = \sigma\left(c^{(k)}\right)}_{\text{for } k=1,...,\ell}, \quad f(x) = z^{(L)} \in \mathbb{R}^n,
\end{align*}
where $\sigma$ acts pointwise and
\begin{equation*}
    A^{(1)} \in \mathbb{R}^{h\times m}, \quad A^{(2)},...,A^{(L-1)} \in \mathbb{R}^{h \times h}, \quad A^{(L)} \in \mathbb{R}^{n \times h},
\end{equation*}
\begin{equation*}
    b^{(1)},...,b^{(L-1)} \in \mathbb{R}^h, \quad b^{(L)} \in \mathbb{R}^n.
\end{equation*}

In what follows we will consider the activation function $\sigma$ to be among the following functions: identity, logistic: $\frac{1}{1+e^{-x}}$, $\tanh$, ReLU: $\max\{0,x\}$.

\begin{theorem}[\textbf{Universal approximation of arbitrarily-wide NNs}]\cite{pinkus1999approximation}\label{thm:univ_nn_width}~\\
Let $\sigma:\mathbb{R}\to\mathbb{R}$ be a continuous, non-polynomial activation function. Let $\mathcal{N}_{m,n}^\sigma$ represent the class of feedforward neural networks with activation function $\sigma$, $m$ neurons in the input layer, $n$ neurons in the output layer, and one hidden layer with an arbitrary number of neurons. Let $K \subset \mathbb{R}^m$ be a compact set. Then $\mathcal{N}_{m,n}^\sigma$ is dense in $C(K,\mathbb{R}^n)$ with respect to the uniform norm.
\end{theorem}

\begin{theorem}[\textbf{Universal approximation of arbitrarily-deep NNs}]\cite{kidger2020universal}\label{thm:univ_nn_depth}~\\
Let $\sigma:\mathbb{R}\to\mathbb{R}$ be a non-affine, continuous activation function which is continuously differentiable at least at one point and with non-zero derivative at that point. Let $\mathcal{N}_{m,n,h}^\sigma$ represent the class of feedforward neural networks with activation function $\sigma$, $m$ neurons in the input layer, $n$ neurons in the output layer, and an arbitrary number of hidden layers with $h$ neurons each. Let $K \subset \mathbb{R}^m$ be a compact set. Then $\mathcal{N}_{m,n,m+n+2}^\sigma$ is dense in $C(K,\mathbb{R}^n)$ with respect to the uniform norm.
\end{theorem}

\subsubsection{Backpropagation} 

Suppose we are given a dataset of $N$ input-output pairs $\{x_i,y_i\}_{i=1}^N$, where $x_i \in \mathbb{R}^m$ ad $y_i \in \mathbb{R}^n$. Consider a feedforward neural network $f_\theta$, where $\theta$ denotes the set of learnable weights in all the layers of the net. Backpropagation is the procedure that attempts to minimize a loss function $\mathcal{L}$ with respect to the neural network's weights $\theta$. This can be achieved by training $f_\theta$ via gradient descent and learning rate $\eta \in (0,1)$ and updating the weights according to the following equation
\begin{align*}
    \theta_{k+1} = \theta_k - \eta\nabla_\theta \mathcal{L},
\end{align*}
where $\mathcal{L}$ is a loss function, e.g. mean-squared error $\mathcal{L} = \frac{1}{N} \sum_{i =1}^N\norm{f_\theta(x_i) - y_i}^2$, cross-entropy etc. The quantities we are interested in computing for backpropagation are the gradients of the loss $\mathcal{L}$ with respect to the model weights $\theta$. To incorporate the bias terms $b^{(k)}$ into the weights we will denote $a_{0,i}^{(k)}=b^{(k)}_i$, where $A^{(k)} = \{a^{(k)}_{i,j}\}$, and $z^{(k)}_0=1$. By definition we have
\begin{align*}
    c^{(k)}_i = \sum_{j=1}^{r_{k-1}}a^{(k)}_{j,i}z^{(k-1)}_j + b_i^{(k)} = \sum_{j=0}^{r_{k-1}}a^{(k)}_{j,i}z^{(k-1)}_j
\end{align*}
where $r_k=m$ if $k=1$, $r_k=h$ if $k=2,...,\ell-1$ and $r_k=n$ if $k=\ell$. By chain rule
\begin{equation*}
    \frac{\partial \mathcal{L}}{\partial a^{(k)}_{i,j}} = \underbrace{\frac{\partial \mathcal{L}}{\partial c^{(k)}_j}}_{\delta^{(k)}_j}\underbrace{\frac{\partial c^{(k)}_j}{\partial a^{(k)}_{i,j}}}_{z^{(k-1)}_i}
\end{equation*}
Using again the chain rule we have
\begin{align*}
    \delta^{(k)}_j = \sum_{p=0}^{r_{k+1}}\frac{\partial \mathcal{L}}{\partial c^{(k+1)}_p}\frac{\partial c^{(k+1)}_p}{\partial c^{(k)}_j} = \sum_{p=0}^{r_{k+1}}\delta_p^{(k+1)}\frac{\partial c^{(k+1)}_p}{\partial c^{(k)}_j}
\end{align*}
Remembering the definition of $c^{(k+1)}_p = \sum_{q=1}^{r_k}a^{(k+1)}_{q,p}\sigma(c^{(k)}_q)$, we have that
\begin{equation*}
    \frac{\partial c^{(k+1)}_p}{\partial c^{(k)}_j} = a_{j,p}^{(k+1)}\sigma'(c^{(k)}_j)
\end{equation*}
Hence, 
\begin{equation*}
    \delta^{(k)}_j = \sigma'(c^{(k)}_j)\sum_{p=0}^{r_{k+1}}\delta_p^{(k+1)}a_{j,p}^{(k+1)}
\end{equation*}
which implies
\begin{align*}
    \frac{\partial \mathcal{L}}{\partial a^{(k)}_{i,j}} = \delta_j^{(k)}z_i^{(k-1)} = \sigma'(c^{(k)}_j)z_i^{(k-1)}\sum_{p=0}^{r_{k+1}}\delta_p^{(k+1)}a_{j,p}^{(k+1)}
\end{align*}
This equation is where backpropagation gets its name. Namely, the "error" term $\delta^{(k)}_j$ at layer $k$ is dependent on the terms $\delta^{(k+1)}_p$ at the next layer $k+1$. Thus, errors flow backward, from the last layer to the first layer. Note that all the values of $c^{(k)}_j$ and $z^{(k-1)}_i$ need to be computed before the backward pass can begin and need to be saved in memory during the forward pass in order to be reused during the backward pass.

\subsection{Neural ODEs}

A \emph{neural ordinary differential equation} (Neural ODE)  is an ODE using a neural network to parameterise the vector field
\begin{equation}\label{eqn:node}
    \mathrm{d} y_t = f_\theta(y_t) dt, \quad y_0 \in \mathbb R^e
\end{equation}
Here $\theta$ represents some vector of learnt parameters, $f_\theta: \mathbb{R}^e \to \mathbb{R}^e$ is any standard neural architecture, and $y:[0, T] \rightarrow \mathbb{R}^e$ is the solution. For many applications $f_\theta$ will just be a simple feedforward network. Note that (\ref{eqn:node}) is a time-homogeneous equation, but the setup can be easily generalised to time-inhomogeneous equations where the vector field also depends on time. The central idea now is to use a differential equation solver as part of a learnt differentiable computation graph.


As noted in \cite{chen2018neural}, the Euler discretization of the Neural ODE (\ref{eqn:node}) matches the definition of a \emph{residual neural network} 
\begin{equation}
    \mathrm{d} y_t = f_\theta(y_t) dt \implies  \text{Euler scheme } \implies y_{k+1} = y_k + \Delta f_{\theta}(y_k) 
\end{equation}
where $k \in \{0,...,L\}$ indexes layers, $\Delta$ is the time step (e.g. $1/T$) and $y_k \in \mathbb{R}^e$ is the output of layer $k$. Euler’s method is perhaps the simplest method for solving ODEs. There have since been more than 120 years of development of efficient and accurate ODE solvers which can be leveraged (Runge-Kutta, adaptive schemes, implicit solvers ...). A different choice of solver yields a different discretization and so a different neural architecture.

Training a Neural ODE means backpropagating through the differential equation solver. There are several ways to do this. Consider optimizing a loss function\footnote{Implicitly, $L$ also depends on some true output data we are trying to match, but we obviously do not differentiate with respect to that so we assume that $L$ is only a function of the model's output.} $L:\mathbb{R}^e \to \mathbb{R}$, whose input is the result of an ODE solver:
\begin{equation}
    L(y_T) = L \left(y_0 + \int_0^Tf_\theta(y_t)dt\right) = L\left( \text{ODESolve}(y_{t_0},f_\theta,[0,T]) \right).
\end{equation}
To optimize $L$ we require the gradients $\partial_\theta L(y_T)$.

The first option is simply to backpropagate through the internal operations of the
differential equation solver. A differential equation solver internally performs the usual arithmetic operations of addition, multiplication, and so on, each of which is differentiable. This method is known as \emph{discretise-then-optimise}. Generally speaking, this approach is fast to evaluate and produces accurate gradients, but it is memory-inefficient, as every internal operation of the solver must be recorded. We refer to \cite{kidger2022neural} for further details on computational efficiency.

It is possible to compute gradients of a scalar-valued loss with respect to all inputs of any ODE solver, without backpropagating through the operations of the solver. Not storing any intermediate quantities of the forward pass allows to train models with constant
memory cost as a function of depth, a major bottleneck for training deep models. Indeed, one can treat the ODE solver as a black box, and compute gradients by deriving a backwards-in-time differential equation, known as \emph{adjoint equation}, which is then solved numerically. This method is known as \emph{optimise-then-discretise}. In the sequel, we will state and prove analogous theorems for neural differential equation driven by rough paths. 



The main objective to backpropagate through a Neural ODE consists of computing the quantity $\partial_{y_0} L(y_T)$, for some continuously differentiable function $L \in C^1\left(\mathbb{R}^e\right)$. 
It can be easily shown that $a_t$ satisfies the following linear ODE run backwards-in-time
\begin{equation}\label{eqn:adjoint_node}
    \frac{d}{dt} a_t = -a_t^\top \nabla f_\theta(y_t), \quad \text{started at } a_T = \nabla L(y_T)
\end{equation} 
where $a_t^\top \nabla f_\theta(y_t)$ is a vector-Jacobian product. We do not provide a proof here as we will prove an anlogous result for the more general setting of CDEs in the next section.

Equation (\ref{eqn:adjoint_node}) can be solved numerically by another call to the ODE solver\footnote{Here, we assumed that the loss function $L$ depends only on the last time point $t_1$. If $L$ depends also on intermediate time points $s_0=t_0 < s_1 <... < s_N=t_1$, we can repeat the adjoint step for each of the intervals $[s_{k-1}, s_k]$ in the backward order and sum up the obtained gradients.}, run backwards-in-time starting from $a_T:=\nabla L(y_T)$, in full analogy with the concept of backpropagation for training standard neural networks:
\begin{equation*}
    a_0 = a_T + \int_T^0da_t = a_T - \int_T^0 a_t^\top \nabla f_\theta(y_t)dt.
\end{equation*} 

The final step is to compute the gradients of $L$ with respect to the parameters $\theta$. For this, we define $a^\theta_t := \partial_\theta L(y_t)$. Then, equation (\ref{eqn:adjoint_node}) applied to the augmented variable $(a_t, a_t^\theta)^\top$ reads
\begin{align*}
    \frac{d}{dt}
    \begin{pmatrix}
        a_t \\ a^\theta_t
    \end{pmatrix} &= - (a_t^\top, a^\theta_t) \frac{\partial}{\partial (y,\theta)}F(y,\theta) 
\end{align*}
where $F(y,\theta) = (f_\theta(y),0)$ so that
\begin{align*}
    \frac{d}{dt}
    \begin{pmatrix}
        a_t \\ a^\theta_t
    \end{pmatrix} &= - (a_t^\top,a^\theta_t)
    \begin{pmatrix}
        \partial_y f_\theta(y_t) & \partial_\theta f_\theta(y_t)\\
        0 & 0 
    \end{pmatrix}\\
    &= -\left(a_t^\top\partial_y f_\theta(y_t), a_t \partial_\theta f_\theta(y_t)\right)
\end{align*}
The first term is the adjoint we have already mentioned, while the second term reads
\begin{equation*}
    \frac{d}{dt} a^\theta_t = -  a_t\partial_\theta f_\theta(y_t) \implies a^\theta_T:=\partial_\theta L(y_T) = - \int_T^0a_t\partial_\theta f_\theta(y_t) dt
\end{equation*}
Hence, the gradient of the loss with respect to the parameters can be computed by another call to the ODE solver.

Adjoint-based methods are usually slightly slower to evaluate because one needs to recalculate $y_t$ in the backward pass. Also, they might introduce small errors compared to direct backpropagation through the ODE solver, so they should generally be used when memory is a concern. We refer again the interested reader to \cite{kidger2022neural} for further details on these computational aspects.

\begin{remark}
    There is a third way of backpropagating through Neural ODEs given by reversible ODE solvers, i.e. numerical solvers where $(y_k,a_k)$ can bealgebraically recovered exactly from $(y_{k+1},a_{k+1})$. These solvers offer both memory efficiency and accuracy of the computed gradients. Like optimise-then-discretise, they do require a small amount of extra computational work, to recompute the forward solution during backpropagation. This is a topic still in its infancy and beyond the scope of this book. We refer to \cite[Section 5.3.2]{kidger2022neural} for further details on this topic.
\end{remark}

For examples and python implementation, we refer the reader to the documentation of the PyTorch library \texttt{torchdiffeq}\footnote{\url{https://github.com/rtqichen/torchdiffeq}} or the JAX library \texttt{diffrax}\footnote{\url{https://docs.kidger.site/diffrax/}}.

\subsection{Neural CDEs}\label{sec:Neural_CDEs}

Let $f_\theta : \mathbb{R}^{e} \to \mathcal{L}(\mathbb{R}^d, \mathbb{R}^e)$ be a locally 1-Lipschitz neural network satisfying the usual linear growth conditions. A \emph{neural controlled differential equation} (Neural CDE) driven by a continuous path $x: [0,T] \to \mathbb{R}^d$ of bounded variaton is of the form
\begin{equation}\label{eqn:ncde}
 dy_t = f_\theta(y_t)dx_t = \sum_{i=1}^d f^i_\theta(y_t)dx^{(i)}_t, \quad \text{started at } y_0 \in \mathbb R^e
\end{equation}
or, in integral form,
$$
y_t = y_0 + \int_0^t f_\theta\left(y_u\right) d x_s.
$$

\begin{remark}
    A Neural CDE as defined in \cite{kidger2020neural} is in fact composed of two additional parametric map: an initial "lifting" neural network $\xi_\theta: \mathbb R^d \to \mathbb R^e$ such that $y_0 = \xi_\theta(x_0)$ and a final (linear) readout map $\ell_\theta(y_T)$ providing the output of the model. 
\end{remark}

Thus, Neural CDEs are the continuous-time analogue of recurrent neural networks (RNNs), as Neural ODEs are to ResNets. The output $y_T$ is then fed into a loss function and trained via stochastic gradient descent in the usual way. The difference of a Neural CDE compared with a Neural ODE is simply that there is a $\dd x_t$ rather than a $\dd t$ term. This modification is what enables the model to change continuously based on incoming data.


\subsubsection{Universality}

Consider the subspace $C_{1,0,t} \subset C_1$ of paths started at $0$ with one coordinate-path (without loss of generality, the first) corresponds to time, i.e. $x^{(1)}_t = t$.

The next result states the universality of linear functionals on the solution of Linear Neural CDEs. Thanks to the density results of neural networks (\Cref{thm:univ_nn_width}  and \Cref{thm:univ_nn_depth}), the same result holds if the linear vector fields are replaced by non-linear neural networks. This result was first proved by \cite{kidger2021neural}.

\begin{theorem}
    Let $M \in \mathbb N \cup \{0\}$. For any $g \in \mathcal{L}(\mathbb R^M, \mathbb R)$, any $f \in \mathcal{L}(\mathbb R^M, \mathcal{L}(\mathbb R^d, \mathbb R^M))$, and any $y_0 \in \mathbb R^M$, define the map $\Psi_{f,g,y_0} \in \mathbb{R}^{C_{1,0,t}}$ as
    \begin{equation*}
    \Psi_{f,g,y_0}: x \mapsto g(y_T) \quad \text{where} \quad y_T = y_0 + \int_0^T f(y_t)dx_t.
    \end{equation*}
    Let $K$ be a compact subset of $\mathbb{R}^{C_{1,0,t}}$ with respect to the $1$-variation topology. Then $$\mathcal{B}_{K}:=\{\Psi_{f,g,y_0}|_{K} : M \in \mathbb N \cup \{0\}, g \in (\mathbb R^M)^*,f \in \mathcal{L}(\mathbb R^M, \mathcal{L}(\mathbb R^d, \mathbb R^M)), y_0 \in \mathbb R^M\}$$ 
    is a dense subset of $C(K)$ with the topology of uniform convergence.
\end{theorem}

 \begin{proof}
     For any $N \in \mathbb N \cup \{0\}$, and any $g \in T^N(\mathbb R^d)^{*}$, define the function $\Phi_g \in \mathbb{R}^{C_{1,0,t}}$ as $\Phi_g: x \mapsto \left(g, \pi_{\leq N} \circ S\left( x\right) \right)$. By \Cref{ex:truncated_sig_cde}, the truncated signature $\pi_{\leq N} \circ S(x)_{0,t}$ is the unique solution to the linear CDE
    \begin{equation*}
        dy_t = f(y_t)dx_t, \quad \text{started at } y_0 = (1,0,...,0) \in T^N(\mathbb R^d)
    \end{equation*}
    where the linear map $f \in \mathcal{L}(T^N(\mathbb R^d), \mathcal{L}\left( \mathbb R^d, T^N(\mathbb R^d) \right))$ is given by
    \begin{equation*}
        f \left( y \right)a = \pi_{\leq N}(\mathfrak{i}_{\leq N}(y) \cdot \mathfrak{i}_1(a)).
    \end{equation*}
    Thus, setting $M = \frac{d^{N+1} - 1}{d - 1}$ and identifying $T^N(\mathbb R^d)$ with $\mathbb R^M$, we have $\Phi_g = \Psi_{f,g,y_0}$.

    Let $K$ be a compact subset of $\mathbb{R}^{C_{1,0,t}}$ with respect to the $1$-variation topology. By \Cref{ex:univ_time}, $\mathcal{A}_{K}:=\{\Phi_g|_{K}: g\in T^N(\mathbb R^d)^{*}, N \in \mathbb N \cup \{0\}\}$ is a dense subset of $C(K)$ with the topology of uniform convergence. We have just proved that $\mathcal{A}_{K} \subset \mathcal{B}_{K}$, from which the result follows.
 \end{proof}

\subsubsection{Backpropagation}

As for Neural ODEs, there are two main options to backpropagate through a Neural CDE. The first option is simply to backpropagate through the internal operations of the differential equation solver. The second method involves deriving a backwards-in-time adjoint equation.


We generalise our notation for Neural CDEs by writing $y^{s, a, x}$ to mean the solution to
$$
y_t^{s, a, x} = a + \int_s^t f_\theta\left(y_u^{s, a, x}\right) d x_u .
$$
It is well-known that the flow map $\Phi_t(s,a,x) \equiv y_t^{s, a, x}$ is differentiable in $a$ (e.g. by \cite[Theorem 4.4]{friz2010multidimensional}) and the Jacobian
$$J_t^{s, a, x}:= \partial_a \Phi_t(s, a, x) \in \mathcal{L}\left(\mathbb{R}^e, \mathbb{R}^e\right)$$
itself satisfies the linear CDE 
\begin{equation}\label{CDE:Jacobian}
    d J_t^{s,a,x} = \sum_{i=1}^d \nabla f_\theta^i \left(y_t\right) \cdot J_t^{s,a,x} d x_t^{(i)}, \quad J_s^{s,a,x}=I_e
\end{equation}
where $J$ is regarded as an $e$ by $e$ real matrix and $\cdot$ denotes the matrix product and $I$ is the identity matrix in $\mathbb R^m$. 

To simplify notation, we remove the dependence on $a$ and $x$ from the Jacobian which simply denote by $J_t^s$ from now on. If $s=0$ we will simply write $J_t$. By letting 
$$z_t:=\int_0^t \sum_{i=1}^d \nabla f_\theta^i\left(y_u\right) d x_u^{(i)}\in \mathcal{L}\left(\mathbb{R}^e, \mathbb{R}^e\right)$$
we can re-write equation (\ref{CDE:Jacobian}) as a linear CDE driven by $z$
$$
d J_t=d z_t \cdot J_t, \quad J_0=I_e.
$$
The following lemma shows that the Jacobian admits an inverse and that this inverse also solves a linear CDE driven by $z$.
\begin{lemma}\label{lemma:inv_jac}
    For every $t \in[0, T]$ the matrix $J_t$ is non-singular. Furthermore, its inverse $M_t=J_t^{-1}$ is the solution to the equation
    \begin{equation}\label{CDE:inverse}
        d M_t=-M_t \cdot d z_t, \quad M_0=I_e.
    \end{equation}
\end{lemma}

\begin{proof}
    Define $M$ to be the unique solution to (\ref{CDE:inverse}) then $A_t:=J_t \cdot M_t$ satisifes
    \begin{equation}\label{eqn:ad}
        d A_t=d z_t \cdot J_t \cdot M_t-J_t \cdot M_t \cdot d z_t=\operatorname{ad}_{d z_t}\left(A_t\right), \quad A_0=I_e
    \end{equation}
    where the adjoint map is defined by $\operatorname{ad}_A: B \longmapsto[A, B]:=A \cdot B-B \cdot A$, for $A$ and $B$ in $\mathbb{R}^{e \times e}$. The linear equation (\ref{eqn:ad}) has a unique solution; it is easily seen that this solution is $A_t \equiv I_e$. By the same argument, $B_t:=M_t \cdot J_t$ solves
    $$
    d B_t=-M_t \cdot d z_t \cdot J_t+M_t \cdot d z_t \cdot J_t=0,  \quad B_0=I_e
    $$
    so that $I_e = J_t \cdot M_t=M_t \cdot J_t$ for all $t$ in $[0, T]$.
\end{proof}

We use Lemma \ref{lemma:inv_jac} to prove the following theorem, which provides a differential equation for the ajoint process of a Neural CDE. The proof is slightly different from the one  proposed in \cite[Appedix C.3]{kidger2022neural}.

\begin{theorem}[\textbf{Adjoint equation for a Neural CDE}]\label{thm:adj_CDE}
    Consider the Neural CDE (\ref{eqn:ncde}) driven by a continuous path $x:[0,T]\to\mathbb{R}^d$ of bounded variation and with continuously differentiable neural vector fields $f_\theta$. Let $L:\mathbb{R}^e \to \mathbb{R}$ be a continuously differentiable loss function. Then, the adjoint process $a_t := \partial_{y_t} L\left(y_T\right)$ satisfies the following backwards-in-time linear CDE
    \begin{equation}\label{eqn:adjoint_ncde}
        da_t = -\sum_{i=1}^d a_t^\top \nabla f_\theta^{i}(y_t)dx^i_t, \quad a_T = \nabla L(y_T).
    \end{equation}
\end{theorem}

\begin{proof}
    By definition of the adjoint process and by the chain rule we have
    $$
    a_t = \nabla L\left(y_T\right)^\top \cdot J^{t}_T=\nabla L\left(y_T\right)^\top \cdot J_T^{0} \cdot\left(J_t^{0}\right)^{-1}=\nabla L\left(y_T\right)^\top \cdot J_T^0 \cdot M_t^0 .
    $$
    As $M$ solves 
    (\ref{CDE:inverse}), for any $u \in \mathbb{R}^e$ the process $Y_t=u^\top M_t$ satisfies
    $$
    d Y_t=-Y_t d z_t, \quad  Y_0=u .
    $$
    Applying this with the choice $u=J_T^\top \nabla L\left(y_T\right)$ we obtain that $a^\top=Y$ so that 
    $$
    d a_t = -a^\top_t d z_t=\sum_{i=1}^d a_t^\top \nabla f_\theta^i\left(y_t\right) d x_t^{(i)} \quad \text{and} \quad a_T=\nabla L\left(y_T\right).
    $$
\end{proof}

For examples of python implementation we refer the reader to the documentation of the PyTorch library \texttt{torchcde}\footnote{\url{https://github.com/patrick-kidger/torchcde}} or the JAX library \texttt{diffrax}\footnote{\url{https://docs.kidger.site/diffrax/}}.






The adjoint equation (\ref{eqn:adjoint_ncde}) for Neural CDEs driven by bounded variation paths can be generalised to neural differential equations driven by (geometric) rough paths, dubbed Neural RDEs, which we discuss next.

\subsection{Neural RDEs}

Let $p \geq 1$, let $\mathbf x$ be a geometric $p$-rough path on $\mathbb R^d$ and let $f_\theta : \mathbb{R}^{e} \to \mathcal{L}(\mathbb{R}^d, \mathbb{R}^e)$ be a $\text{Lip}^\gamma$ neural network, with $\gamma > p$. A \emph{neural rough differential equation} (Neural RDE) driven by $\mathbf x$ is of the form
\begin{equation}\label{eqn:nrde}
 dy_t = f_\theta(y_t)d\mathbf x_t, \quad y_0 \in \mathbb R^e
\end{equation}
where the solution is in the sense of \Cref{def:RDE_sol}. Next, we derive the adjoint equation for the Neural RDE (\ref{eqn:nrde}). A variant of the proof for Neural SDEs appear in \cite[Appendix C.3]{kidger2022neural}.

\begin{theorem}[\textbf{Adjoint equation for a Neural RDE}]\label{thm:RDE_adjoint} Let $p \geq 2$, let $\mathbf x$ be a geometric $p$-rough path on $\mathbb R^d$, let $f_\theta \in \text{Lip}^\gamma(\mathbb{R}^{e}, \mathcal{L}(\mathbb{R}^d, \mathbb{R}^e))$ be a neural network, with $\gamma > p$, and let $L:\mathbb{R}^e \to \mathbb{R}$ denote a continuously differentiable  loss function. Then the adjoint process $a_t := \partial_{y_t}L(y_T)$ of the Neural RDE (\ref{eqn:nrde}) coincides with the solution of the following backwards-in-time linear RDE
\begin{equation}\label{eqn:adj_rough}
    da_t = - a_t^\top \nabla f_\theta(y_t) d \mathbf{x}_t, \quad a_T = \nabla L(y_T).
\end{equation}
\end{theorem}

\begin{proof}
Let $\{y^N\}_{N \geq 1}$ be the sequence of solutions to the following CDEs
\begin{equation*}
    dy_t^N = f_\theta(y_t^N)d x^N_t, \quad y_0^N = y_0 \in \mathbb{R}^n
\end{equation*}
where $(x^N)$ is a sequence of continuous bounded variation paths such that the sequence $(\pi_{\leq \lfloor p \rfloor} \circ S(x^N))$ converges to the rough path $\mathbf{x}$ in $p$-variation. Then, by \Cref{thm:ULT}, we have that the corresponding sequence of CDE solutions $\{y^N\}_{N\geq1}$ converges in $p$-variation to the solution $y$ of the RDE (\ref{eqn:nrde}).

By \Cref{thm:adj_CDE}, each adjoint process $a_t^N:=\partial_{y^N_t}L(y^N_T)$ satisfies the backwards-in-time linear CDE
\begin{equation*}
    da_t^N = - (a^N)^\top_t \nabla f_\theta(y_t^N) dx^N_t = -\sum_{i=1}^d (a^N)^\top_t \nabla f^i_\theta(y_t^N) d(x^N)^{(i)}_t, \quad a_T^N = \nabla L(y^N_T).
\end{equation*}
As in the proof of \Cref{thm:adj_CDE}, by letting 
$$z^N_t:=\int_0^t \sum_{i=1}^d \nabla f_\theta^i\left(y^N_s\right) d (x^N)_s^{(i)} \in \mathcal{L}\left(\mathbb{R}^e, \mathbb{R}^e\right)$$
we can rewrite the adjoint equation for $a^N$ as follows
\begin{equation}\label{eqn:m4}
    da^N_t = -(a^N)^\top_t dz_t^N.
\end{equation}
Note that since $f_\theta$ is bounded, the path $t \mapsto z^N_t$ is of bounded variation.


Consider the solutions $(z^N, \tilde z^N)$ of the following CDEs driven by $(x^N, y^N)$
\begin{align*}
    d z^N_t &= \sum_{i=1}^d \nabla f_\theta^i\left(\tilde z^N_t\right) d (x^N)_t^{(i)}\\
    d \tilde z^N_t &= dy^N_t
\end{align*}

Since $f_\theta \in \text{Lip}^\gamma(\mathbb{R}^{e}, \mathcal{L}(\mathbb{R}^d, \mathbb{R}^e))$ it follows that  $\nabla f_\theta \in \text{Lip}^{\gamma-1}(\mathbb{R}^{e}, \mathcal{L}(\mathbb{R}^d, \mathcal{L}(\mathbb{R}^e, \mathbb{R}^e)))$. Thus, by \Cref{thm:ULT} and \Cref{rk:full_rde}, the sequence $(z^N)$ is such that the sequence $(\pi_{\leq \lfloor p \rfloor} \circ S(z^N))$ converges in $p$-variation to a geometric rough path $\mathbf{z}$. 

Therefore, by \Cref{thm:ULT}, there exists $y \in C^{p\text{-var}}([0,T],\mathbb R^e)$ such that the sequence $(a^N)$ converges uniformly to the solution of the linear RDE run backwards-in-time
\begin{equation*}
    da_t = - (a)^\top_t \nabla f_\theta(y_t)d\mathbf x_t, \quad a_T = \nabla L(y_T).
\end{equation*}
It remains to show that $a_t = \partial_{y_t}L(y_T)$.

Recall that by \cite[Theorem 4.4]{friz2010multidimensional} the Jacobians $J_t^{s, y_s^N, x, N} = J_t^{s,N}$ satisfy the CDEs
\begin{equation*}
    d J_t^{s,N} = \nabla f_\theta\left(y_t^N\right) \cdot J_t^{s,N} d x_t^N, \quad J_s^N = I_e.
\end{equation*}
Considering the augmented path $(J^{s,N}, y^N)$, we can invoke again \Cref{thm:ULT} to show that the sequence $(J^{s,N})$  converges uniformly to the solution in $C^{p\text{-var}}([0,T], \mathcal{L}(\mathbb R^e,\mathbb R^e))$ of the following linear RDE
\begin{equation*}
    d J^s_t =  \nabla f_\theta\left(y_t\right) \cdot J_t^s d \mathbf x_t, \quad J_s^s = I_e.
\end{equation*}
Similarly, by \Cref{lemma:inv_jac}, the inverses $M^{s,N} := (J^{s,N})^{-1}$ exist and satisfy the CDEs
\begin{equation*}
    d M_t^{s,N}=-M_t^{s,N} \cdot d z^N_t, \quad M_s^{s,N}=I_e.
\end{equation*}
Therefore, by \Cref{thm:ULT}, there exists $M^s \in C^{p\text{-var}}([0,T], \mathcal{L}(\mathbb R^e,\mathbb R^e))$ such that the sequence $(M^{s,N})$ converges uniformly to the solution of the linear RDE 
\begin{equation*}
    dM^s_t = - M^s_t \cdot d\mathbf z_t, \quad M_s^s = I_e.
\end{equation*}
Because $L$ is continuously differentiable, we have that the sequence
$$
a_t^N = \nabla L\left(y^N_T\right)^\top \cdot J^{t,N}_T=\nabla L\left(y_T^N\right)^\top \cdot J_T^{0,N} \cdot\left(J_t^{0,N}\right)^{-1}=\nabla L\left(y_T^N\right)^\top \cdot J_T^{0,N} \cdot M^{0,N}_t
$$
converges uniformly to 
$$\nabla L\left(y_T\right)^\top \cdot J^0_T \cdot M^0_t = \partial_{y_t}L(y_T)$$ 
as $N \to \infty$, which concludes the proof.
\end{proof}

\subsection{Neural SDEs as generative models for time series}\label{sec:Neural_SDEs}

A Neural SDE is a model of the form
\begin{equation}\label{eqn:Neural_SDE}
    dy_t = \mu_\theta(y_t)dt + \sigma_\theta(y_t)\circ dW_t, \quad y_0 \in \mathbb R^e
\end{equation}
where
\begin{equation*}
    \mu_\theta : \mathbb{R}^e\to \mathbb{R}^e \quad \text{and} \quad \sigma_\theta : \mathbb{R}^e \to \mathcal{L}(\mathbb{R}^d, \mathbb{R}^e)
\end{equation*}
are $\text{Lip}^1$ and $\text{Lip}^{2+\epsilon}$ neural networks respectively, and where $W$ is a $d$-dimensional Brownian motion. We note that the initial condition $y_0$ can be either deterministic or random. In the latter case, to ensure existence and uniqueness of solutions we require in addition that $\mathbb{E}[\|y_0\|^2]<\infty$. As for Neural CDEs, a Neural SDE as defined in the literature is in fact composed of two additional parametric map: an initial "lifting" neural network and a final (linear) readout map providing the output of the model. For the sake of simplicity we omit these add-ins in our presentation.

Informally, given a target distribution $y^{\text{true}}$ on pathspace, the Neural SDE model (\ref{eqn:Neural_SDE}) can be trained so that $y$ should have approximately the same distribution as the target $y^{\text{true}}$, for some notion of closeness in the space of distributions on path space. So a Neural SDE can be seen as a natural generative model for time series data. We only summarise the main approaches that have been proposed in the literature and refer the interested reader to the relevant papers for further details about the implementations of these models.

In \cite{issa2024non} the authors propose to train a Neural SDE by minimizing the  objective
\begin{equation*}
     \min_\theta \tilde d_\phi\left(y,y^{\text{true}}\right)
\end{equation*}
where $\tilde d_\phi$ is a variant of the $\phi$-signature kernel MMD from Chapter 2, and where the dependence of $y$ on the parameters $\theta$ is implicit through (\ref{eqn:Neural_SDE}).

The option followed in \cite{kidger2021neural} is to train the Neural SDE as a Generative Adversarial Network (GAN) \cite{goodfellow2014generative} using the Wasserstein distance. The discriminator of a Wasserstein-GAN requires a class of test functions that maps a path to $\mathbb{R}^e$. The choice made in \cite{kidger2021neural} is to use Neural CDEs as such test functions. The objective function to train a Neural SDE-GAN is of the form 
\begin{equation*}
    \min_\theta\max_\phi \left(\mathbb{E}_y[F_\phi(y)] - \mathbb{E}_{y^{\text{true}}}[F_\phi(y^{\text{true}})]\right)
\end{equation*}
where $F_\phi$ is the solution map of a Neural CDE parameterised by some parameters $\phi$.

In \cite{arribas2020sigsdes} the vector fields of a Neural SDEs are replaced by a linear functional on the signature of the driving Brownian motion. The resulting Sig-SDE model is fit to financial data by matching option prices observed in the market.

%% file: ex_chapter3.tex
\section{Exercises}

\begin{exercise}
    Let $x : [0,T] \to V$ be a smooth path. Show that $\mathbf X : \Delta_T \to T^1(V)$ defined by $\mathbf X_{s,t} = (1, x_t - x_s)$ is a geometric $1$-rough path controlled by a control $\omega : \Delta_T \to \mathbb R_+$ defined by $\omega(s,t) = |t-s|$.
\end{exercise}

\begin{exercise}
    Show that the map $\mathbf X : \Delta_T \to T^2(V)$ defined as
    \begin{equation*}
        \mathbf X_{s,t} = \left(1, \begin{pmatrix}
            0 \\ 0
        \end{pmatrix}, (t-s)\begin{pmatrix}
            0 & -1 \\
            1 & 0
        \end{pmatrix}\right)
    \end{equation*}
    is a geometric $2$-rough path controlled by $\omega : \Delta_T \to \mathbb R_+$ defined by $\omega(s,t) = |t-s|$.
\end{exercise}

\begin{exercise}
    Consider the two sequence of paths of bounded variation
    \begin{align*}
        x^n_t = \frac{1}{n}\left(\cos(2\pi n^2 t), \sin(2\pi n^2 t) \right), \quad y^n_t = \frac{1}{n}\left(\cos(2\pi n^2 t), -\sin(2\pi n^2 t) \right)
    \end{align*}
    Prove that the path $(x^n, y^n) \in \Omega_1(\mathbb R^2 \oplus \mathbb R^2)$ converges in $(p+\epsilon)$-variation to a rough path that is not the zero path.
\end{exercise}

\begin{exercise}\label{ex:Euler_estimate}
    Let $\gamma>1$ and let $f=\left(f_1, ..., f_d\right)$ be a collection of $\operatorname{Lip}^{\gamma-1}$ vector fields over $\mathbb R^e$. Let $x : [s,t] \to \mathbb R^d$ be a continuous paths of bounded variation. Show that there exists a constant $C=C(\gamma)$ such that,
    $$
    \left\|y_t -  \mathcal{E}^{\lfloor \gamma \rfloor}_{s,t}(y_s;f,x) \right\| \leq C\left(\|f\|_{\operatorname{Lip}^{\gamma-1}} \int_s^t\left|d x_r\right|\right)^\gamma.
    $$
    where $\mathcal{E}^{\lfloor \gamma \rfloor}_{s,t}(y_s;f,x)$ is the Euler scheme defined in equation (\ref{eqn:Euler_scheme}).
\end{exercise}


\begin{exercise}
    Let $\mu_\theta \in Lip^1(\mathbb{R}^{e+1},\mathbb{R}^e)$ and $\sigma_\theta \in Lip^{2 + \epsilon}(\mathbb{R}^{e+1}, \mathcal L(\mathbb{R}^d,\mathbb{R}^e))$ be two neural networks. Let $L:\mathbb{R}^e \to \mathbb{R}$ be a continuously differentiable  loss function. Show that the adjoint process $a_t := \partial_{y_t}L(y_T)$ of the Neural Stratonovich SDE
    \begin{equation*}
        dy_t = \mu_\theta(t,y_t)dt + \sum_{i=1}^d\sigma^i_\theta(t,y_t)\circ dW^i_t, \quad y_0 = a
    \end{equation*}
     coincides with the solution of the following backwards-in-time linear Stratonovich SDE
    \begin{equation*}
        da_t = - a_t^\top \nabla \mu_\theta(t,y_t) dt - \sum_{i=1}^d a_t^\top \nabla \sigma_\theta^i(t,y_t) \circ d W_t, \quad a_T = \nabla L(y_T).
    \end{equation*}
\end{exercise}

\begin{exercise}[Stochastic normalising flows]
Let $p \in [2,3)$, and let $\boldsymbol{x}$ be a geometric $p$-rough path on $\mathbb R^d$, and assume there exists a sequence $(x^n)$ of $C^1(\mathbb R^d)$  paths such that $(\pi_{\leq 2} \circ S(x^n))$ converges in $p$-variation to $\mathbf x$. Let $\gamma > p$ and $f_\theta \in Lip^{\gamma}(\mathbb R^e, \mathcal{L}(\mathbb R^d, \mathbb R^e))$ be a neural network. Consider the following Neural RDE
$$ \mathrm{d} z_t =f_\theta\left(z_t\right) \mathrm{d} \boldsymbol{x}_t, \quad z_0 \sim p_0$$
and the sequence of CDEs
$$\mathrm{d} z^n_t =f_\theta \left(z_t^n\right) \mathrm{d} x^n_t \quad z^n_0 = z_0$$
where $p_0$ is a density on $\mathbb{R}^d$ such that $\log p_0$ is continuous. Prove the following.

\begin{enumerate}
    \item The law of $z_t^n$ admits a differentiable density $p_t^n$, given by 
    $$\frac{d}{dt}\log p_t(z^n_t) = \nabla_{z^n_t} \cdot \left(f_\theta(z^n_t)\frac{d}{dt}x^n_t\right).$$
    \item The law $z_t$ admits a density $p_t$ satisfying for any $x \in \mathbb{R}^d$
    $$\sup _{t \in[0, T]}\left|\log p_t^n(x)-\log p_t(x)\right| \rightarrow 0, \quad \text{as } n \rightarrow \infty.$$
    \item The path $t \mapsto \log p_t\left(z_t\right)$ is the unique solution to the rough differential equation
    $$
    \mathrm{d} \log p_t\left(z_t\right)=-\nabla_{z_t} \cdot\left(f_\theta\left(z_t\right) \mathrm{d} \boldsymbol{x}_t\right).
    $$
    \item For any continuously differentiable loss function $L: \mathbb{R}^{d+1} \rightarrow \mathbb{R}$ 
    $$
    \partial_\theta L\left(z_t^n, \log p_t^n\left(z_t^n\right)\right) \rightarrow \nabla_\theta L\left(z_t, \log p_t\left(z_t\right)\right), \quad \text{as } n \rightarrow \infty.
    $$
\end{enumerate}

\end{exercise}


%% file: book.bbl
\begin{thebibliography}{10}

\bibitem{aronszajn1950theory}
Nachman Aronszajn.
\newblock Theory of reproducing kernels.
\newblock {\em Transactions of the American mathematical society}, 68(3):337--404, 1950.

\bibitem{arribas2020sigsdes}
Imanol~Perez Arribas, Cristopher Salvi, and Lukasz Szpruch.
\newblock Sig-sdes model for quantitative finance.
\newblock In {\em ACM International Conference on AI in Finance}, 2020.

\bibitem{bailleul2015flows}
Isma{\"e}l Bailleul.
\newblock Flows driven by rough paths.
\newblock {\em Revista matem{\'a}tica iberoamericana}, 31(3):901--934, 2015.

\bibitem{boedihardjo2016signature}
Horatio Boedihardjo, Xi~Geng, Terry Lyons, and Danyu Yang.
\newblock The signature of a rough path: uniqueness.
\newblock {\em Advances in Mathematics}, 293:720--737, 2016.

\bibitem{bogachev2007}
Vladimir~Igorevich Bogachev and Maria Aparecida~Soares Ruas.
\newblock {\em Measure theory}, volume~1.
\newblock Springer, 2007.

\bibitem{cass2024wiener}
Thomas Cass and Emilio Ferrucci.
\newblock On the wiener chaos expansion of the signature of a gaussian process.
\newblock {\em Probability Theory and Related Fields}, pages 1--39, 2024.

\bibitem{cass2021general}
Thomas Cass, Terry Lyons, and Xingcheng Xu.
\newblock Weighted signature kernels.
\newblock {\em Annals of Applied Probability 34(1A): 585-626}, 2024.

\bibitem{cass2023signature}
Thomas Cass, Remy Messadene, and William~F Turner.
\newblock Signature asymptotics, empirical processes, and optimal transport.
\newblock {\em Electronic Journal of Probability}, 28:1--19, 2023.

\bibitem{cass2023fubini}
Thomas Cass and Jeffrey Pei.
\newblock A fubini type theorem for rough integration.
\newblock {\em Revista Mathematica Iberoamericana}, 39(2), 2023.

\bibitem{cass2024free}
Thomas Cass and William~F Turner.
\newblock Free probability, path developments and signature kernels as universal scaling limits.
\newblock {\em arXiv preprint arXiv:2402.12311}, 2024.

\bibitem{cass2022topologies}
Thomas Cass and William~F Turner.
\newblock Topologies on unparameterised path space.
\newblock {\em Journal of Functional Analysis, 286(4)}, 2024.

\bibitem{chen1954iterated}
Kuo-Tsai Chen.
\newblock Iterated integrals and exponential homomorphisms.
\newblock {\em Proceedings of the London Mathematical Society}, 3(1):502--512, 1954.

\bibitem{chen1957integration}
Kuo-Tsai Chen.
\newblock Integration of paths, geometric invariants and a generalized baker-hausdorff formula.
\newblock {\em Annals of Mathematics}, 65(1):163--178, 1957.

\bibitem{chen1961formal}
Kuo-tsai Chen.
\newblock Formal differential equations.
\newblock {\em Annals of Mathematics}, pages 110--133, 1961.

\bibitem{chen2001collected}
Kuo-Tsai Chen and Philippe Tondeur.
\newblock {\em Collected Papers of KT Chen}.
\newblock Springer Science \& Business Media, 2001.

\bibitem{chen2018neural}
Ricky~TQ Chen, Yulia Rubanova, Jesse Bettencourt, and David~K Duvenaud.
\newblock Neural ordinary differential equations.
\newblock {\em Advances in neural information processing systems}, 31, 2018.

\bibitem{chevyrev2016characteristic}
Ilya Chevyrev, Terry Lyons, et~al.
\newblock Characteristic functions of measures on geometric rough paths.
\newblock {\em The Annals of Probability}, 44(6):4049--4082, 2016.

\bibitem{chevyrev2022signature}
Ilya Chevyrev and Harald Oberhauser.
\newblock Signature moments to characterize laws of stochastic processes.
\newblock {\em Journal of Machine Learning Research}, 23(176):1--42, 2022.

\bibitem{christmann2010universal}
Andreas Christmann and Ingo Steinwart.
\newblock Universal kernels on non-standard input spaces.
\newblock In {\em Advances in neural information processing systems}, pages 406--414, 2010.

\bibitem{cirone2023neural}
Nicola~Muca Cirone, Maud Lemercier, and Cristopher Salvi.
\newblock Neural signature kernels as infinite-width-depth-limits of controlled resnets.
\newblock 2023.

\bibitem{cochrane2021sk}
Thomas Cochrane, Peter Foster, Varun Chhabra, Maud Lemercier, Cristopher Salvi, and Terry Lyons.
\newblock Sk-tree: a systematic malware detection algorithm on streaming trees via the signature kernel.
\newblock {\em arXiv preprint arXiv:2102.07904}, 2021.

\bibitem{cuchiero2023global}
Christa Cuchiero, Philipp Schmocker, and Josef Teichmann.
\newblock Global universal approximation of functional input maps on weighted spaces.
\newblock {\em arXiv preprint arXiv:2306.03303}, 2023.

\bibitem{cuturi2007kernel}
Marco Cuturi, Jean-Philippe Vert, Oystein Birkenes, and Tomoko Matsui.
\newblock A kernel for time series based on global alignments.
\newblock In {\em 2007 IEEE International Conference on Acoustics, Speech and Signal Processing-ICASSP'07}, volume~2, pages II--413. IEEE, 2007.

\bibitem{davie2007differential}
Alexander~M Davie.
\newblock Differential equations driven by rough paths: an approach via discrete approximation.
\newblock {\em arXiv preprint arXiv:0710.0772}, 2007.

\bibitem{day1966finitediff}
J.~T. Day.
\newblock A runge-kutta method for the numerical solution of the goursat problem in hyperbolic partial differential equations.
\newblock {\em The Computer Journal}, 9:81--83, 1966.

\bibitem{diehl2019invariants}
Joscha Diehl and Jeremy Reizenstein.
\newblock Invariants of multidimensional time series based on their iterated-integral signature.
\newblock {\em Acta Applicandae Mathematicae}, 164(1):83--122, 2019.

\bibitem{Driver}
Bruce Driver.
\newblock Personal communication.

\bibitem{dudley2018real}
Richard~M Dudley.
\newblock {\em Real analysis and probability}.
\newblock CRC Press, 2018.

\bibitem{dyson1949radiation}
Freeman~J Dyson.
\newblock The radiation theories of tomonaga, schwinger, and feynman.
\newblock {\em Physical Review}, 75(3):486, 1949.

\bibitem{esig}
Terry~Lyons et~al.
\newblock Coropa computational rough paths (software library).
\newblock 2010.

\bibitem{fawcett2002problems}
Thomas Fawcett.
\newblock {\em Problems in stochastic analysis: Connections between rough paths and non-commutative harmonic analysis}.
\newblock PhD thesis, University of Oxford, 2002.

\bibitem{fermanian2023new}
Adeline Fermanian, Terry Lyons, James Morrill, and Cristopher Salvi.
\newblock New directions in the applications of rough path theory.
\newblock {\em IEEE BITS the Information Theory Magazine}, 2023.

\bibitem{flaxman2015machine}
Seth~R Flaxman.
\newblock {\em Machine learning in space and time}.
\newblock PhD thesis, Ph. D. thesis, Carnegie Mellon University, 2015.

\bibitem{fliess1983algebraic}
Michel Fliess, Moustanir Lamnabhi, and Fran{\c{c}}oise Lamnabhi-Lagarrigue.
\newblock An algebraic approach to nonlinear functional expansions.
\newblock {\em IEEE transactions on circuits and systems}, 30(8):554--570, 1983.

\bibitem{folland1999real}
Gerald~B Folland.
\newblock {\em Real analysis: modern techniques and their applications}, volume~40.
\newblock John Wiley \& Sons, 1999.

\bibitem{friz2020course}
Peter~K Friz and Martin Hairer.
\newblock {\em A course on rough paths}.
\newblock Springer, 2020.

\bibitem{friz2010multidimensional}
Peter~K Friz and Nicolas~B Victoir.
\newblock {\em Multidimensional stochastic processes as rough paths: theory and applications}, volume 120.
\newblock Cambridge University Press, 2010.

\bibitem{gaines1997variable}
Jessica~G Gaines and Terry~J Lyons.
\newblock Variable step size control in the numerical solution of stochastic differential equations.
\newblock {\em SIAM Journal on Applied Mathematics}, 57(5):1455--1484, 1997.

\bibitem{goodfellow2014generative}
Ian Goodfellow, Jean Pouget-Abadie, Mehdi Mirza, Bing Xu, David Warde-Farley, Sherjil Ozair, Aaron Courville, and Yoshua Bengio.
\newblock Generative adversarial nets.
\newblock {\em Advances in neural information processing systems}, 27, 2014.

\bibitem{goursat1916course}
Edouard Goursat.
\newblock {\em A Course in Mathematical Analysis: pt. 2. Differential equations.[c1917}, volume~2.
\newblock Dover Publications, 1916.

\bibitem{graham2013sparse}
Benjamin Graham.
\newblock Sparse arrays of signatures for online character recognition.
\newblock {\em arXiv preprint arXiv:1308.0371}, 2013.

\bibitem{gretton2012kernel}
Arthur Gretton, Karsten~M Borgwardt, Malte~J Rasch, Bernhard Sch{\"o}lkopf, and Alexander Smola.
\newblock A kernel two-sample test.
\newblock {\em The Journal of Machine Learning Research}, 13(1):723--773, 2012.

\bibitem{gretton2007kernel}
Arthur Gretton, Kenji Fukumizu, Choon Teo, Le~Song, Bernhard Sch{\"o}lkopf, and Alex Smola.
\newblock A kernel statistical test of independence.
\newblock {\em Advances in neural information processing systems}, 20, 2007.

\bibitem{gubinelli2004controlling}
Massimiliano Gubinelli.
\newblock Controlling rough paths.
\newblock {\em Journal of Functional Analysis}, 216(1):86--140, 2004.

\bibitem{gyurko2008numerical}
Lajos~Gergely Gyurk{\'o}.
\newblock Numerical methods for approximating solutions to rough differential equations.
\newblock {\em .}, 2008.

\bibitem{gyurko2010rough}
Lajos~Gergely Gyurk{\'o} and Terry Lyons.
\newblock Rough paths based numerical algorithms in computational finance.
\newblock {\em Contemporary Mathematics}, 515:17--46, 2010.

\bibitem{hairer2014theory}
Martin Hairer.
\newblock A theory of regularity structures.
\newblock {\em Inventiones mathematicae}, 198(2):269--504, 2014.

\bibitem{hambly2010uniqueness}
Ben Hambly and Terry Lyons.
\newblock Uniqueness for the signature of a path of bounded variation and the reduced path group.
\newblock {\em Annals of Mathematics}, pages 109--167, 2010.

\bibitem{hoglund2023neural}
Melker Hoglund, Emilio Ferrucci, Camilo Hernandez, Aitor~Muguruza Gonzalez, Cristopher Salvi, Leandro Sanchez-Betancourt, and Yufei Zhang.
\newblock A neural rde approach for continuous-time non-markovian stochastic control problems.
\newblock {\em arXiv preprint arXiv:2306.14258}, 2023.

\bibitem{horvath2023optimal}
Blanka Horvath, Maud Lemercier, Chong Liu, Terry Lyons, and Cristopher Salvi.
\newblock Optimal stopping via distribution regression: a higher rank signature approach.
\newblock {\em arXiv preprint arXiv:2304.01479}, 2023.

\bibitem{issa2024non}
Zacharia Issa, Blanka Horvath, Maud Lemercier, and Cristopher Salvi.
\newblock Non-adversarial training of neural sdes with signature kernel scores.
\newblock {\em Advances in Neural Information Processing Systems}, 36, 2024.

\bibitem{janssen2012order}
Arend Janssen.
\newblock {\em Order book models, signatures and numerical approximations of rough differential equations}.
\newblock PhD thesis, Oxford University, UK, 2012.

\bibitem{kawski1997noncommutative}
Matthias Kawski and H{\'e}ctor~J Sussmann.
\newblock Noncommutative power series and formal lie-algebraic techniques in nonlinear control theory.
\newblock {\em Operators, Systems and Linear Algebra: Three Decades of Algebraic Systems Theory}, pages 111--128, 1997.

\bibitem{kidger2022neural}
Patrick Kidger.
\newblock On neural differential equations.
\newblock {\em arXiv preprint arXiv:2202.02435}, 2022.

\bibitem{bonnier2019deep}
Patrick Kidger, Patric Bonnier, Imanol Perez~Arribas, Cristopher Salvi, and Terry Lyons.
\newblock Deep signature transforms.
\newblock In {\em Advances in Neural Information Processing Systems}, volume~32, 2019.

\bibitem{kidger2021neural}
Patrick Kidger, James Foster, Xuechen Li, Harald Oberhauser, and Terry Lyons.
\newblock Neural sdes as infinite-dimensional gans.
\newblock {\em arXiv preprint arXiv:2102.03657}, 2021.

\bibitem{signatory}
Patrick Kidger and Terry Lyons.
\newblock {Signatory: differentiable computations of the signature and logsignature transforms, on both CPU and GPU}.
\newblock {\em arXiv:2001.00706}, 2020.

\bibitem{kidger2020universal}
Patrick Kidger and Terry Lyons.
\newblock Universal approximation with deep narrow networks.
\newblock In {\em Conference on learning theory}, pages 2306--2327. PMLR, 2020.

\bibitem{kidger2020neural}
Patrick Kidger, James Morrill, James Foster, and Terry Lyons.
\newblock Neural controlled differential equations for irregular time series.
\newblock {\em arXiv preprint arXiv:2005.08926}, 2020.

\bibitem{kiraly2019kernels}
Franz~J Kir{\'a}ly and Harald Oberhauser.
\newblock Kernels for sequentially ordered data.
\newblock {\em Journal of Machine Learning Research}, 2019.

\bibitem{lees1960goursat}
Milton Lees.
\newblock The goursat problem.
\newblock {\em Journal of the Society for Industrial and Applied Mathematics}, 8(3):518--530, 1960.

\bibitem{lemercier2021siggpde}
Maud Lemercier, Cristopher Salvi, Thomas Cass, Edwin~V Bonilla, Theodoros Damoulas, and Terry Lyons.
\newblock Siggpde: Scaling sparse gaussian processes on sequential data.
\newblock In {\em International Conference on Machine Learning}. PMLR, 2021.

\bibitem{lemercier2021distribution}
Maud Lemercier, Cristopher Salvi, Theodoros Damoulas, Edwin Bonilla, and Terry Lyons.
\newblock Distribution regression for sequential data.
\newblock In {\em International Conference on Artificial Intelligence and Statistics}, pages 3754--3762. PMLR, 2021.

\bibitem{leslie2001spectrum}
Christina Leslie, Eleazar Eskin, and William~Stafford Noble.
\newblock The spectrum kernel: A string kernel for svm protein classification.
\newblock In {\em Biocomputing 2002}, pages 564--575. World Scientific, 2001.

\bibitem{lodhi2002text}
Huma Lodhi, Craig Saunders, John Shawe-Taylor, Nello Cristianini, and Chris Watkins.
\newblock Text classification using string kernels.
\newblock {\em Journal of machine learning research}, 2(Feb):419--444, 2002.

\bibitem{lyons1998differential}
Terry~J Lyons.
\newblock Differential equations driven by rough signals.
\newblock {\em Revista Matem{\'a}tica Iberoamericana}, 14(2):215--310, 1998.

\bibitem{lyons2007differential}
Terry~J Lyons, Michael Caruana, and Thierry L{\'e}vy.
\newblock {\em Differential equations driven by rough paths}.
\newblock Springer, 2007.

\bibitem{lyons2018inverting}
Terry~J Lyons and Weijun Xu.
\newblock Inverting the signature of a path.
\newblock {\em Journal of the European Mathematical Society}, 20(7):1655--1687, 2018.

\bibitem{magnus1954exponential}
Wilhelm Magnus.
\newblock On the exponential solution of differential equations for a linear operator.
\newblock {\em Communications on pure and applied mathematics}, 7(4):649--673, 1954.

\bibitem{manten2024signature}
Georg Manten, Cecilia Casolo, Emilio Ferrucci, S{\o}ren~Wengel Mogensen, Cristopher Salvi, and Niki Kilbertus.
\newblock Signature kernel conditional independence tests in causal discovery for stochastic processes.
\newblock {\em arXiv preprint arXiv:2402.18477}, 2024.

\bibitem{morrill2021neural}
James Morrill, Cristopher Salvi, Patrick Kidger, and James Foster.
\newblock Neural rough differential equations for long time series.
\newblock In {\em International Conference on Machine Learning}, pages 7829--7838. PMLR, 2021.

\bibitem{muandet2012learning}
Krikamol Muandet, Kenji Fukumizu, Francesco Dinuzzo, and Bernhard Sch{\"o}lkopf.
\newblock Learning from distributions via support measure machines.
\newblock In {\em Advances in neural information processing systems}, pages 10--18, 2012.

\bibitem{noble2006support}
William~S Noble.
\newblock What is a support vector machine?
\newblock {\em Nature biotechnology}, 24(12):1565--1567, 2006.

\bibitem{pannier2024path}
Alexandre Pannier and Cristopher Salvi.
\newblock A path-dependent pde solver based on signature kernels.
\newblock {\em arXiv preprint arXiv:2403.11738}, 2024.

\bibitem{pinkus1999approximation}
Allan Pinkus.
\newblock Approximation theory of the mlp model in neural networks.
\newblock {\em Acta numerica}, 8:143--195, 1999.

\bibitem{redmann2022runge}
Martin Redmann and Sebastian Riedel.
\newblock Runge-kutta methods for rough differential equations.
\newblock {\em Journal of Stochastic Analysis}, 3(4):6, 2022.

\bibitem{reichl1999modern}
Linda~E Reichl.
\newblock A modern course in statistical physics, 1999.

\bibitem{reizenstein2017calculation}
Jeremy Reizenstein.
\newblock Calculation of iterated-integral signatures and log signatures.
\newblock {\em arXiv preprint arXiv:1712.02757}, 2017.

\bibitem{reizenstein2018iisignature}
Jeremy Reizenstein and Benjamin Graham.
\newblock The iisignature library: efficient calculation of iterated-integral signatures and log signatures.
\newblock {\em arXiv preprint arXiv:1802.08252}, 2018.

\bibitem{reutenauer2003free}
Christophe Reutenauer.
\newblock Free {L}ie algebras.
\newblock In {\em Handbook of algebra}, volume~3, pages 887--903. Elsevier, 2003.

\bibitem{salvi2021signature}
Cristopher Salvi, Thomas Cass, James Foster, Terry Lyons, and Weixin Yang.
\newblock The signature kernel is the solution of a goursat pde.
\newblock {\em SIAM Journal on Mathematics of Data Science}, 3(3):873--899, 2021.

\bibitem{salvi2023structure}
Cristopher Salvi, Joscha Diehl, Terry Lyons, Rosa Preiss, and Jeremy Reizenstein.
\newblock A structure theorem for streamed information.
\newblock {\em Journal of Algebra}, 634:911--938, 2023.

\bibitem{salvi2022neural}
Cristopher Salvi, Maud Lemercier, and Andris Gerasimovics.
\newblock Neural stochastic pdes: Resolution-invariant learning of continuous spatiotemporal dynamics.
\newblock {\em Advances in Neural Information Processing Systems}, 35:1333--1344, 2022.

\bibitem{salvi2021higher}
Cristopher Salvi, Maud Lemercier, Chong Liu, Blanka Hovarth, Theodoros Damoulas, and Terry Lyons.
\newblock Higher order kernel mean embeddings to capture filtrations of stochastic processes.
\newblock {\em arXiv preprint arXiv:2109.03582}, 2021.

\bibitem{scholkopf2001generalized}
Bernhard Sch{\"o}lkopf, Ralf Herbrich, and Alex~J Smola.
\newblock A generalized representer theorem.
\newblock In {\em Computational Learning Theory: 14th Annual Conference on Computational Learning Theory, COLT 2001 and 5th European Conference on Computational Learning Theory, EuroCOLT 2001 Amsterdam, The Netherlands, July 16--19, 2001 Proceedings 14}, pages 416--426. Springer, 2001.

\bibitem{scholkopf2004kernel}
Bernhard Sch{\"o}lkopf, Koji Tsuda, Jean-Philippe Vert, et~al.
\newblock {\em Kernel methods in computational biology}.
\newblock MIT press, 2004.

\bibitem{simon2020metrizing}
Carl-Johann Simon-Gabriel, Alessandro Barp, Bernhard Sch{\"o}lkopf, and Lester Mackey.
\newblock Metrizing weak convergence with maximum mean discrepancies.
\newblock {\em arXiv preprint arXiv:2006.09268}, 2020.

\bibitem{simon2018kernel}
Carl-Johann Simon-Gabriel and Bernhard Sch{\"o}lkopf.
\newblock Kernel distribution embeddings: Universal kernels, characteristic kernels and kernel metrics on distributions.
\newblock {\em The Journal of Machine Learning Research}, 19(1):1708--1736, 2018.

\bibitem{smola2007hilbert}
Alex Smola, Arthur Gretton, Le~Song, and Bernhard Sch{\"o}lkopf.
\newblock A hilbert space embedding for distributions.
\newblock In {\em International Conference on Algorithmic Learning Theory}, pages 13--31. Springer, 2007.

\bibitem{sriperumbudur2016optimal}
Bharath Sriperumbudur.
\newblock On the optimal estimation of probability measures in weak and strong topologies.
\newblock {\em Bernoulli}, 22(3):1839--1893, 2016.

\bibitem{sriperumbudur2010hilbert}
Bharath~K Sriperumbudur, Arthur Gretton, Kenji Fukumizu, Bernhard Sch{\"o}lkopf, and Gert~RG Lanckriet.
\newblock Hilbert space embeddings and metrics on probability measures.
\newblock {\em The Journal of Machine Learning Research}, 11:1517--1561, 2010.

\bibitem{suli2003introduction}
Endre S{\"u}li and David~F Mayers.
\newblock {\em An introduction to numerical analysis}.
\newblock Cambridge university press, 2003.

\bibitem{szabo2016learning}
Zolt{\'a}n Szab{\'o}, Bharath~K Sriperumbudur, Barnab{\'a}s P{\'o}czos, and Arthur Gretton.
\newblock Learning theory for distribution regression.
\newblock {\em The Journal of Machine Learning Research}, 17(1):5272--5311, 2016.

\bibitem{toth2020bayesian}
Csaba Toth and Harald Oberhauser.
\newblock Bayesian learning from sequential data using gaussian processes with signature covariances.
\newblock In {\em International Conference on Machine Learning}, pages 9548--9560. PMLR, 2020.

\bibitem{walkden2014ergodic}
Charles Walkden.
\newblock Ergodic theory.
\newblock {\em Lecture Notes University of Manchester}, 2014.

\bibitem{wazwaz1993finitediff}
A.~M. Wazwaz.
\newblock On the numerical solution for the goursat problem.
\newblock {\em Applied Mathematics and Computation}, 59:89--95, 1993.

\bibitem{zhang2012kernel}
Kun Zhang, Jonas Peters, Dominik Janzing, and Bernhard Sch{\"o}lkopf.
\newblock Kernel-based conditional independence test and application in causal discovery.
\newblock {\em arXiv preprint arXiv:1202.3775}, 2012.

\end{thebibliography}
